\documentclass[numbers,webpdf,imaiai]{template}%

\theoremstyle{thmstyletwo}%
\newtheorem{theorem}{Theorem}[section]% meant for sectionwise numbers
\newtheorem{proposition}[theorem]{Proposition}%
\newtheorem{lemma}[theorem]{Lemma}
\newtheorem{corollary}[theorem]{Corollary}
\newtheorem{example}[theorem]{Example}
\newtheorem{remark}[theorem]{Remark}
\newtheorem{definition}[theorem]{Definition}
\numberwithin{equation}{section}

\colorlet{pink}{red!40!white} % Laura's
\colorlet{swooosh}{green!20!orange} % rab's
\usetikzlibrary{patterns.meta} % add to preamble if not already loaded
\tikzset{
  wuaxis/.style={->, black!70, line width=0.4pt},
  wuone/.style={circle, fill=MidnightBlue, inner sep=2.2pt},
  wuzero/.style={rectangle, fill=OrangeRed, inner sep=2.1pt},
  wulab/.style={font=\scriptsize},
  wutiny/.style={font=\scriptsize, black!55},
  wuboundary/.style={ForestGreen, dashed, line width=0.8pt},
}
\providecommand{\wuatom}[3]{%
  \begin{scope}[shift={#1}]
    \fill[OrangeRed] (0,0) circle (#3);
    \fill[MidnightBlue] (0,0) -- (90:#3)
      arc[start angle=90, end angle={90-360*#2}, radius=#3] -- cycle;
    \draw[black!55, line width=0.2pt] (0,0) circle (#3);
  \end{scope}}
  
\newlength{\DepthReference}
\newlength{\HeightReference}
\newlength{\Width}%
\newcommand{\inlinecolorbox}[2][red]%
{%
    \settowidth{\Width}{#2}%
    \colorbox{#1}%
    {%      
        \raisebox{-\DepthReference}%
        {%
                \parbox[b][\HeightReference+\DepthReference][c]{\Width}{\centering#2}%
        }%
    }%
}

\usepackage{xcolor}
\usepackage[most]{tcolorbox}   % loads 'skins' and 'breakable'

\definecolor{imaiainavy}{HTML}{002F65}
\colorlet{runexrule}{imaiainavy}
\colorlet{runexfill}{imaiainavy!6!white}

\tcbset{
  runexbox/.style={
    enhanced, breakable,                 % may split across pages
    colback=runexfill, colframe=runexrule,
    boxrule=0pt, leftrule=2pt, sharp corners,
    left=6pt, right=6pt, top=3pt, bottom=3pt,
    before skip=\medskipamount, after skip=\medskipamount,
  }
}

\makeatletter
\newtheorem{runexinner}{Running Example}[section]
\newenvironment{runex}[1][]
  {\if\relax\detokenize{#1}\relax
     \def\runex@begin{\begin{runexinner}}%
   \else
     \def\runex@begin{\begin{runexinner}[#1]}%
   \fi
   \runex@begin\upshape}
  {\end{runexinner}}
\makeatother
\tcolorboxenvironment{runex}{runexbox}

\begin{document}

\copyrightyear{2026}
\vol{00}
\pubyear{2026}
\access{}
\appnotes{Paper}
\copyrightstatement{}
\firstpage{1}

\title[Comparing Corrupted Constrained Learning Problems]{Comparing Corrupted Constrained Learning Problems}

\author{Laura Iacovissi* \ORCID{0000-0002-5012-551X}, Rabanus Derr \ORCID{0000-0003-1703-7283}
}
\author{Robert C. Williamson \ORCID{0000-0002-8862-1412}
\address{\orgdiv{Tübingen AI Center}, \orgname{University of Tübingen}, \orgaddress{\country{Germany}}}}
\authormark{Iacovissi, Derr, Williamson}

\corresp[*]{Corresponding author: \href{email:email-id.com}{laura.iacovissi@uni-tuebingen.de}}

\received{29}{9}{2026}
\revised{Date}{0}{Year}
\accepted{Date}{0}{Year}

%\editor{Associate Editor: Name}

\abstract{A key result in statistics is the data processing inequality, originally proved by \citet{blackwell1951comparison} and later refined by \citet{degroot1962uncertainty} in terms of statistical uncertainty. 
The latter statement claims that the Bayes risk achieved on raw data always undercuts the Bayes risk on a processed version of the same data. 
This seemingly contradicts empirical findings in machine learning: pre-training, representation learning, feature learning, data augmentation are techniques used to improve performance of machine learning models.
We reconcile both worlds by simply accepting that machine learning problems are \emph{constrained} learning problems: the model class used does not include all measurable functions.
We present counterexamples showing that the classical data processing inequality fails to hold in such a setting.
Hence, we formulate a generalized data processing inequality, requiring the constrained Bayes risk of a joint distribution (with respect to a loss function and a constrained model class) to lower bound the constrained Bayes risk on the stochastically modified data distribution, regardless of the choice of distribution. 
We show this inequality to be equivalent to a set containment condition on a specific function set induced by the loss and model class, called the \emph{superprediction set}. 
Finally, we exploit our characterization, derive sufficient conditions for this containment and quantify the inequality-gap.}

\keywords{constrained risk minimization, Markov kernel, (strong) data processing inequality, Blackwell order, generalized data corruptions.}

\maketitle

\section{Introduction}
David Blackwell (\citeyear{blackwell1951comparison}) introduced a novel method for comparing \emph{statistical experiments}, modeled as Markov kernels from some set of target variables $\cY$ to a set of observables $\cX$.
He proposed the comparison be made in terms of the expected losses of the \emph{statistical learning problem} associated with the statistical experiment, \ie, the decision-theoretic (dis)utility of (mis)identifying the targets from the observables.%
\footnote{In fact, his method of comparison was based on then unpublished work by \citet*{Bohnenblust1949reconnaissance}, which was later released and proved to be equivalent.}
The result was later expanded by himself as well as other authors \cite[\eg,][]{blackwell1953equivalent, degroot1962uncertainty,le1964sufficiency}, giving rise to the field of ``comparison of statistical experiments'' \cite{torgersen1991comparison,shiryaev2000statistical}.

The most notorious outcome of this line of work is the Blackwell-Sherman-Stein (BSS) theorem \cite[for full statement see, \eg,][]{le1964sufficiency}, which gives an answer to the question of when an experiment is more informative than another, \ie, more ``useful'' according to the decision-theoretic metric. 
Roughly stated, the theorem proves that the experiment $E$ is more informative than the experiment $E'$ \wrt \emph{any} decision problem if and only if $E'$ can be simulated %\Laura{recovered? obtained? replicated?} 
by performing $E$ and then applying some stochastic noise to its outcome.%
\footnote{This theorem can be seen as a justification of sufficient statistics in terms of ``usefulness'', as it holds as an equality of the decision-theoretic metric of comparison for $E$ and $E'$ when the simulation of $E'$ via $E$ happens using deterministic noise.
%Hence the theorem can be restated as, ``The experiment based on the sufficient statistic $T(X)$, defined \st $P_\theta(X \mid T(X))$ does not depend on $\theta$, is exactly as informative as the experiment based on $X$ for every decision problem''. \Laura{too much notation for this stage I think. cna be moved to later}
} 
Such an alteration process of an experiment is called \emph{randomization}.

One half implication of the BSS theorem is known as the \emph{data processing inequality} (DPI) for \emph{Bayes risk}, where the Bayes risk of a learning problem is the minimally achieved expected loss when identifying the target variables from the observable attributes. 
The DPI result states that the Bayes risk of a statistical experiment obtained as a randomization of another can not be lower than the originally achieved Bayes risk \cite{degroot1962uncertainty}. 
The same result has been also proved in the context of information theory \cite{kullback1951information} and the theory of divergences \cite{ali1966general}, and has found applications for instance in economics \cite[\eg, see][]{khan2024comparison} and machine learning.% \cite[\eg, see][]{tishby2015deep,xu,xu2020theory, williamson2024information}.

\subsection{Data Processing Inequality in Machine Learning} 
As the popularity of machine learning increased different flavors of the DPI regained new attention, for instance in the context of feature learning \cite{van2015theory}, or when understanding the relationship between predictive capabilities and information, both formally \cite{reid2011information,williamson2024information} and empirically \cite{zhao2022fundamental}, as well as to explain observed behaviors of learning schemes, \eg, using the information bottleneck principle \cite{tishby1999information,tishby2015deep,conklin2026learning}.
However, the modern modes, or cultures, of statistics \cite[cf.][]{breiman2001statistical} are high-dimensional, non-parametric, and less focused on the data-generating process \cite[cf. page 7,][]{shalev2014understanding}. 
They have features distinct from \citet{blackwell1951comparison}'s setting, which instead aligns with the classical statistical approach to learning and modeling that focuses on the data-generating distribution \cite{breiman2001statistical}. 
In other words, limiting ourselves to comparing experiments before and after randomization (\ie, what the classical DPI result does) is not necessarily a good abstraction when aiming to study machine learning problems.

This discrepancy may in principle lead to incorrect conclusions.
For instance, consider the case of pre-trained large models, now a standard tool in machine learning: they provide representations, \ie, non-bijective transformations of data, which are learned from large, general purpose datasets and then plugged into simple task-adapted predictive models, sometimes called ``heads'' or ``probes'' depending on context. Whether deterministic or randomized, the representation mapping is (mathematically) equivalent to a randomization of an experiment, \ie, a function from a distributions on $\cX$ to distributions on $\cX$. % adapted to specific tasks through fine-tuning. 
If we assume that it is legitimate to apply the data processing inequality for Bayes risk to this setting, the theorem prescribes that the use of pre-trained representations should not improve predictive performance over a non pre-trained representation when the ad-hoc ``head'' is perfectly optimized for a fixed loss function and data distribution. %\rab{I changed the above two sentences quite substantially here. @Laura please check. Part of your former text is commented out.} \Laura{I find ``adapted to the task'' a bit vague, but it can also work at this introductory stage.} \rab{changed it.}
%with respect to using high-quality data tailored to our intended task.
Yet, empirical evidence seemingly contradicts this. How can this observation be reconciled with the existing literature?

Different scholars have already noticed this discrepancy and started approaching it in different ways, either seeing it as a feature for the machine learning setting \cite[\eg,][]{xu2020theory,ethayarajh2022understanding} or a bug \cite[\eg,][]{finzi2026entropy,turgeman2026does}. 
We aim to face this problem in a more fundamental way, because the modern ``machine learning'' mode of statistics has shifted the importance from the underlying data-generating distribution to model classes of predictors, from experiments to joint data distributions, and from general decision problems to specific loss functions.
The BSS theorem and the DPI result are only partially equipped to shed light on the new modes, and therefore they require re-investigation with a focus on the new elements of importance.\footnote{The original line of work followed by \citet{degroot1962uncertainty} was heavily related to a statistical notion of information. However, in the following section we will focus only on a notion of generalized entropy, \ie, Bayes risk \cite{grunwald2004game}, instead of defining some generalized notion of information. %and try to find properties that can help proving data processing inequality results. We do not focus on any related notion of information at this stage, a leave it for future work.
However, there is a close link between our work and an existing type of constrained statistical information (Appendix~\ref{app:predictive information}). 
}

\subsection{A Generalized DPI Tailored to Machine Learning Practices} \label{sec:dpi tailored to ml}
To address the limitations of comparing statistical learning problems in machine learning using the DPI, we aim to generalize the inequality studied in \citet{degroot1962uncertainty}'s Theorem 6.21, \ie,\footnote{For a mathematically precise and exhaustive introduction of the symbols and moving parts, the reader is referred to Section~\ref{sec: stat probs with kernes}.}
\begin{align} \label{eq:DPI-informal}
    \br_{\color{OrangeRed} \ell} ({\color{MidnightBlue} \pi_{\cY} \tm E}) \le \br_{\color{OrangeRed} \ell} ({\color{MidnightBlue} \pi_{\cY} \tm E'}) \quad {\color{orange} \forall \ \text{marginal probability} \ \pi_{\cY}}, {\color{ForestGreen} \forall \ \text{loss} \ \ell}.
\end{align}
In words, the Bayes risk $\br$ with respect to all loss functions $\ell$ and under all fixed prior distribution $\pi_{\cY}$ over some finite set of target variables $Y$ is larger for the randomized experiment $E'$ derived from experiment $E$, than for the ``pure'' experiment $E$. Experiments are nothing else than the classical way of referring to certain conditional distributions.

In this paper, we instead contribute the definition and analysis of a more general formulation of the DPI (GDPI), %here only introduced informally,
\begin{align} \label{eq:GDPI-informal}
    \br_{\color{OrangeRed} \ell \circ \cH} ({\color{MidnightBlue} \phi}) \le \br_{\color{OrangeRed} \ell \circ \cH} ({\color{MidnightBlue} \phi'}) \quad {\color{orange} \forall \ \text{probability} \ \phi}, {\color{ForestGreen} \text{a fixed loss} \ \ell, \text{and model class} \ { \cH}},
\end{align}
which is motivated by the observations made for the pre-trained representation examples as well as the identified features of the ``machine learning'' culture of statistics described above. 
While at this stage keeping the definition of the generalized DPI informal, we want to bring to the reader's attention the key changes we introduce in our proposed formulation.
\begin{enumerate}[label=(\alph*)]
    \item From prior {\color{MidnightBlue}$\pi_{\cY}$} and experiment {\color{MidnightBlue}$E$} to joint distribution {\color{MidnightBlue}$\phi = \pi_{\cY} \tm E$}.
    \item From loss function {\color{OrangeRed} $\ell$} to loss function and constrained model class {\color{OrangeRed} $\ell \circ \cH$.}
    \item From ${\color{orange} \forall \ \text{marginal probability} \ \pi_{\cY}}$ to ${\color{orange} \forall \ \text{probability} \ \phi}$.
    \item From $ {\color{ForestGreen} \forall \ \text{loss} \ \ell}$ to ${\color{ForestGreen} \text{a fixed loss} \ \ell \ \text{and model class} \ { \cH}}$.
\end{enumerate}

The transition from unconstrained (no $\cH$) to constrained statistical problems (with $\cH$) is a central aspect in our shift from the classical statistical approach towards a more modern approach to learning.\footnote{Notice that \citet{blackwell1951comparison}'s notion of informativity did account for an action space, which effectively is a restriction of the prediction values the model class can attain--a subcase of our proposal.} 
Even in the case of large models that may look like no constraint is imposed, the \emph{effective model class}, \ie, models which are actually considered during training, is limited by the learning algorithm influencing the convergence dynamic, the loss landscape, the initialization conditions, how data are fed into the model at training time, etc.

This transition introduces substantial complications, discussed in detail in the next sections. In particular, the result of \citet{degroot1962uncertainty} does not extend to constrained Bayes risk, as we show in Section~\ref{sec:counterexample} with a simple counterexample. Moreover, the class of admissible corruptions, \ie, modifications of problems, is much broader than experiment randomization only, and different corruptions have different consequences depending on their type \cite{iacovissi2026corruptions}.

Consequently, in Section~\ref{sec: comparison attainable prediction losses} we characterize a generalized DPI from Eq.~\eqref{eq:GDPI-informal}, for a fixed loss and model class while allowing $\phi$ to vary. This perspective is not only novel,%
\footnote{Some of these alternative inequalities have been considered in the literature. For example, \citet[Section 5]{goel1979comparison} studied the data processing inequality result assuming fixed marginal probabilities. \citet{feldman1972some} instead allowed comparisons of experiments for a fixed model class and loss, while still quantifying over all priors. To our knowledge, however, there is no prior work on data processing inequalities aimed at comparing model class performance itself, that is, the composite set $\ell \circ \cH$.}
but also deliberately removes the least accessible component of a learning problem---the data-generating distribution---and centers the analysis on the loss and model class, which are specified by the practitioner and encode their domain knowledge and needs.

\subsection{Contributions}
\begin{itemize}

    \item[] (A) We showcase the failure of the classical data processing inequality in the constrained case (Section~\ref{sec:counterexample}).
    
    \item[] (B) Based on the technical, but central Theorem~\ref{theo:support-ineq-characterization-superpredictionset}, we prove our main result: the characterization of a generalized data processing inequality under corruption via the containment of sets generated by the choice of loss and model class (Corollary~\ref{corollary:kernel-gdpi-characterization}).
    
    \item[] (C) Leveraging Corollary~\ref{corollary:kernel-gdpi-characterization}, we can show that certain loss-model class combinations necessarily suffer performance degradation under structured corruptions. In particular, we show when loss and model classes are not immune to label (Proposition~\ref{proposition:Bayes Risk Inequality with Simple Label Corruption} and Proposition~\ref{prop:sufficient conditions dpi label corruption}) or attribute noise (Proposition~\ref{proposition:Bayes Risk Inequality with Simple Attribute Corruption} and Proposition~\ref{prop:sufficient conditions dpi attribute corruption}).
    
    \item[] (D)  A second consequence of Theorem~\ref{theo:support-ineq-characterization-superpredictionset} is Corollary~\ref{corollary:no-kernel-gdpi-characterization}, which characterizes the Bayes risk ordering of two learning problems. That is, under which conditions does a certain loss-model class combination outperforms another loss-model class combination for all distributions.

    \item[] (E) We complement our main result with a quantitative version of the generalized data processing inequality (strong GDPI) in the Appendix~\ref{app: antipolar inequalities}, using an antipolar duality.
\end{itemize}

As our proofs require a set of (unusual) mathematical tools, we provide a running example to illustrate the central statements. We recommend the reader to get guided by the running example, which appears in highlighted boxes.

\begin{runex}[Introduction]\label{rex:intro}
Throughout the example, the unconstrained learning problem, \ie, cross-entropy loss and unrestricted model class, is contrasted with the constrained one, \ie, cross-entropy loss and logistic regression. Proofs and are given in Appendix~\ref{app:warm-up}.
\end{runex}

% As a tool, we leverage the concept of the superprediction set, which is introduced and generalized to constrained statistical problems in Section~\ref{sec: supervised learning with superpred}.
% Such a characterization using superprediction sets allows us to prove a set of generalized DPI results, \ie, sufficient conditions for the inequality in Eq.~\eqref{eq:GDPI-informal} to hold (Section~\ref{sec:sufficient conditions}).%

\section{Statistical Learning Problems with Markov Kernels} \label{sec: stat probs with kernes}
We now introduce all the notation and formal setup to state and prove our contributions. At the end of the section, we initiate formally the running example, practically illustrating the setting described before.

A \emph{topological vector space} is a vector space $\cZ$ over $\reals$, equipped with a topology $\mathfrak{T}$ that makes the vector addition and scalar multiplication maps continuous. 
In the remainder, we simply refer to it as a vector space and specify the topology when not apparent.

A subset $\cX \subseteq \cZ$ is called \emph{convex} if for every pair of points $y, w \in \cX$ and every $\lambda \in [0,1]$, the point $\lambda y + (1-\lambda) w$ also belongs to $\cX$. 
The \emph{convex hull} of $\cZ$, denoted $\co(\cZ)$, is the smallest convex set containing $\cZ$. %Equivalently, it consists of all finite convex combinations of points in $\cZ$:
%$\co(\cZ) = \left\{ \sum_{i=1}^n \lambda_i z_i : n \in \mathbb{N},\; z_i \in \cZ,\; \lambda_i \ge 0,\; \sum_{i=1}^n \lambda_i = 1 \right\}.$
Finally, a point $z \in \cZ$ is called an \emph{extreme point} of $\cZ$ if it cannot be expressed as a nontrivial convex combination of other points in $\cZ$. Formally, $z$ is extreme if for any $y, w \in \cZ$ and any $\lambda \in (0,1)$, $z = \lambda y + (1-\lambda) w  \implies y = w = z$.
The set of all extreme points of $\cZ$ is denoted $\ext(\cZ) \coloneqq \{ z \in \cZ : z \text{ is an extreme point of } \cZ \}$.

 An ordered pair $(\cZ, \Sigma)$, consisting of a set $\cZ$ and a $\sigma$-algebra $\Sigma$ on it, is called a \emph{measurable space}. When $\cZ$ is a separable, completely metrizable topological space, it is said to be a \emph{Polish space}; if $\cZ$ is a Polish space and $\Sigma(\cZ)$ the associated Borel $\sigma$-algebra generated by its topology, the measurable space $(\cZ, \Sigma(\cZ))$ is said to be a \emph{standard Borel measurable space}.
In the following, we will only use the notation $\Sigma(\cZ)$ for the Borel $\sigma$-algebra, as we will mostly use them in our setting.

The set of functions $f \colon \cZ \to \reals$ that are \emph{bounded}, \ie, $\exists C\in \reals_{\ge0} \, \mid \, |f(z)| \le C, \forall \, z \in \cZ$, and \emph{$\Sigma(\cZ)$-measurable} is denoted as $\Bb(\cZ)$. %This set together with the supremum norm, 
%$$\| f\|_\infty = \sup_{z \in \cZ} |f(z)|$$ 
%for $f \in \Bb(\cZ)$ forms a Banach space \cite[page 804]{schechter1997handbook}.

The \emph{set of all continuous, linear functionals} on $\Bb(\cZ)$ coincides with $\ba(\cZ)$, the space of all \textit{finitely additive}, signed measures $\mu\colon (\cZ, \Sigma(\cZ)) \to \reals$ with bounded total variation $\|\mu\|_\infty$ \cite[Theorem 14.4]{aliprantis2006infinite}. %That is defined as $\| \mu\|_\infty \coloneqq \sup_{\cA \in \Sigma(\cZ)} |\mu(\cA)|.$
%Notice that the space $\ba(\cZ)$ normed with $\| \mu\|_\infty$ forms a Banach space \cite[29.6.c \& 29.29.f]{schechter1997handbook}.
Its subset $\ca(\cZ) \subset \ba(\cZ)$ instead indicates the set of all \textit{countably additive}, signed measures with bounded total variation.

We denote by $\finProb(\cZ)$ the set of finitely additive probability measures on $(\cZ,\Sigma)$, while its subset of all countably additive probability measures (or simply ``probability distributions'') is denoted $\Prob(\cZ)$.
 
Let $(\cZ,\Sigma)$ and $(\cZ',\Sigma')$ be measurable spaces. A function $f\colon \cZ \to \cZ'$ is said to be measurable if for every $\cA \in \Sigma'$ the pre-image of $\cA$ under $f$ is in $\Sigma$; that is, for all $\cA \in \Sigma'$, $f^{-1}(\cA) \coloneqq \{ z \in \cZ \mid f(x) \in \cA \} \in \Sigma.$

%\rab{To be added: weak topology definition!e.g., for $\sbaBb$}
We define $\sBbba$ as the topology on $\Bb(\cZ)$ which makes every functional $f \mapsto \langle f, \mu\rangle$ (for any $\mu \in \ba(\cZ)$) continuous. Analogously, $\sbaBb$ is the topology on $\ba(\cZ)$ which makes every functional $\mu \mapsto \langle f, \mu\rangle$ (for any $f \in\Bb(\cZ) $) continuous. We refer to these topologies as \emph{weak} for the respective spaces. % as they are contained within the norm-topologies of the respective spaces \cite[28.13.b \& 28.22.a]{schechter1997handbook}.
For further details see Appendix~\ref{app:duality function and measures}.

\subsection{Markov Kernels}
Markov kernels provide a formal and general framework for statistical learning problems.
They extend the notion of stochastic maps beyond finite or discrete settings, enabling a consistent and rigorous treatment of conditional probability in continuous or uncountable spaces. 
In our standard Borel framework, they coincide with regular conditional distributions, providing the standard measure-theoretic framework for conditional probability \cite{cinlar2011ProbabilityAS}.
We formally introduce them in the following.

\begin{definition}[Kernels and Markov Kernels] 
    \label{def:markov-kernel}
    Let $(\cZ_1,\Sigma_1)$ and $(\cZ_2,\Sigma_2)$ be standard Borel. Let $\ka \colon  \cZ_1 \tm \Sigma_2 \to [0, +\infty)$. Then, $\ka$ is called a \emph{kernel} from $(\cZ_1,\Sigma_1)$ to $(\cZ_2,\Sigma_2)$ if
    \begin{enumerate}
        \item the mapping $z_1 \mapsto \ka(z_1,B)$ is $\Sigma_1$-measurable for every set $B \in \Sigma_2$;
        \item the mapping $B \mapsto \ka(z_1,B)$ is a countably additive measure on $(\cZ_2,\Sigma_2)$ for every $z_1 \in \cZ_1$.
    \end{enumerate}
    When the second mapping induces a countably additive probability distribution, then $\ka$ is a \emph{Markov kernel}. We refer to the set of such Markov kernels as $\cM(\cZ_1, \cZ_2)$.
\end{definition}

\begin{example}
    Consider a set $\cZ=\{0,1\}$, a kernel $\ka \in \cM(\cZ,\cZ)$, and random variables $Z, \tilde{Z}$ on the probability space $(\cZ, \Sigma(\cZ), \phi)$. We can conveniently represent the kernel with a matrix whose entries are the values of the associated conditional density $p$:
    \begin{align}
        \ka =
        \begin{bmatrix}
            p( \tilde{Z}=1 \mid Z=1) & p( \tilde{Z}=1 \mid Z=0) \\ 
            p( \tilde{Z}=0 \mid Z=1) & p( \tilde{Z}=0 \mid Z=0)          
         \end{bmatrix} %=
        % \begin{bmatrix}
        %      0.7 & 0.5 \\ 
        %      0.3 & 0.5          
        % \end{bmatrix}  
        .
    \end{align}
\end{example}

Associated with a Markov kernel are the functionals called \emph{Markov transition}, \emph{Markov operator} and \emph{Markov operator adjoint}.
We define them and provide basic properties in the following. 

\begin{proposition}[Markov Transition, Theorem 19.13, \citet{aliprantis2006infinite}]
\label{prop:markov transition}
    Let $\ka$ be a Markov kernel from $(\cZ_1,\Sigma_1)$ to $(\cZ_2,\Sigma_2)$. The function 
    \begin{align}
        \mt{\ka} \colon \cZ_1 \to \Prob(\cZ_2), 
        \qquad 
        z_1 \mapsto \mt{\ka}(z_1) \coloneqq \ka(z_1, \cdot),
    \end{align}
    is know as \emph{Markov transition}, and it is measurable with respect to $\Sigma(\Prob(\cZ_2))$, the Borel-$\sigma$-algebra induced by the $\sbaBb$-topology on $\Prob(\cZ_2)$.
\end{proposition}
% \begin{remark}
%     An alternative and compact notation for a Markov kernel $\ka \in \cM(\cZ_1 , \cZ_2)$, reminding us of its role in defining a Markov transition functional, is 
% \begin{align}
%     \ka \colon \cZ_1 \karrow \cZ_2.
% \end{align}
% \end{remark}

Once we establish the concept of Markov transition, we can define the related operator and adjoint operator. 
These functionals are sometimes also referred to as Markov kernel \emph{actions},  as they allow the object to be combined with functions and probabilities.

\begin{proposition}[Markov Kernel Actions]
\label{prop:markov operator}
    Let $\ka \in \cM( \cZ_1 ,\cZ_2)$ be a Markov kernel.
    \begin{enumerate}
        \item The \emph{Markov operator} $\mto{\ka} \colon \Bb(\cZ_2) \to \Bb(\cZ_1)$ is the mapping
    \begin{align}
        f \mapsto \mto{\ka}f \quad \text{ \st } \quad z_1 \in \cZ_1 \mapsto (\mto{\ka}f)(z_1) \coloneqq \int_{\cZ_2}  f(z_2) \, \ka(z_1, dz_2) ;
    \end{align}

    \item The \emph{Markov operator adjoint}, denoted $\amto{\ka} \colon \ba(\cZ_1) \to \ba(\cZ_2)$,
    \begin{align}
        \mu \mapsto \amto{\ka}\mu \quad \text{ \st } \quad B \in \Sigma_2 \mapsto (\amto{\ka}\mu)(B) \coloneqq \int_{\cZ_1} \ka(z_1, B) d\mu .%= \langle \ka(\cdot, B), \mu \rangle.
    \end{align}
    In particular, if $\mu \in \ca(\cZ_1)$, then $\amto{\ka}\mu \in \ca(\cZ_2)$. If $\mu \in \Prob(\cZ_1)$, then $\amto{\ka}\mu \in \Prob(\cZ_2)$.
    \end{enumerate}  
\end{proposition}
\begin{proof}
    The first statement is proved in Theorem 19.7 in \cite{aliprantis2006infinite}. The second statement is proved in Theorem 19.9, point 2 in \cite{aliprantis2006infinite}.
\end{proof}

%Lastly, we give here some definitions, mostly of notational relevance.

\begin{definition}[Markov Kernel Action on a Set of Functions]
\label{def:kernel on set of functions}
    Let $\mto{\ka}$ be a Markov operator associated to the Markov kernel $\ka \in \cM( \cZ_1 , \cZ_2)$. We define the action of the kernel on a set as, 
    \begin{align}
        \mto{\ka} \cF \coloneq \{\mto{\ka}f \, \mid \,  f \in \cF  \}.
    \end{align}
\end{definition}

\begin{definition}[Deterministic, Identity and Degenerate Markov Kernels]
    A \emph{deterministic} Markov kernel $\ka_f \in \cM(\cZ_1, \cZ_2)$ %\colon \cZ_1 \tm \Sigma_2 \to [0, 1]$
    induced by some measurable function $f \colon \cZ_1 \to \cZ_2$, is defined as
    \begin{align}\label{eq:deterministic markov kernel}
        (z_1, \cA) \mapsto \ka_f(z_1, \cA) \coloneqq \inset_{\cA}(f(z_1)) \,, 
    \end{align}
    where $\inset_{\cA}(x)$ denotes the \emph{indicator function} for $\cA$ at $x$. 
    The \emph{identity} Markov kernel for a set $\cZ$ is the deterministic Markov kernel $\id \in \cM(\cZ, \cZ)$ induced by the identity function on $\cZ$.
    A \emph{degenerate} or \emph{non-informative} kernel $\ka_\nu \in \cM(\cZ_1, \cZ_2)$ induced by the probability distribution $\nu \in \Prob(\cZ_2)$ is the one with associated operator
    \begin{align} \label{eq:trivial-kernel}
        \mt{\ka}_{\nu} \colon \cZ_1 \to \Prob(\cZ_2) \,, \quad \mt{\ka}_{\nu}(z_1) = \nu \,, \quad \forall z_1\in \cZ_1.
    \end{align}
\end{definition}
As a consequence, every measurable function $f \colon \cZ_1 \to \cZ_2$ and probability $\nu \in \Prob(\cZ_2)$  can be represented by a kernel.

% \begin{proposition}\label{prop: deterministic kernel composition is function chaining}
%     Let $\ka_f \in \cM(\cZ_1, \cZ_2)$ for some $f \colon \cZ_1 \to \cZ_2$, and $\ka \in \cM( \cZ_2 , \cZ_3)$. The Markov transition associated to the kernel $\ka' \coloneqq \mto{\ka}_f \circ \ka$, is
%     \begin{align}
%         \mt{\ka'} = \mt{\ka} \circ f \colon \cZ_1 \to \Prob(\cZ_3) \,.
%     \end{align}
% \end{proposition}

\begin{proposition}\label{prop:kernel extreme points}
    The set of extreme points of all Markov kernels is the set of deterministic kernels induced by all measurable functions,
    %The extreme points of the set of Markov kernels are all deterministic kernels, 
    \ie,
    \begin{align}
        \ext \cM (\cZ_1, \cZ_2) = \{ \ka_f \in \cM (\cZ_1, \cZ_2) \,\mid\, f \colon \cZ_1 \to \cZ_2 \ \text{measurable} \, \}.
    \end{align}
\end{proposition}
\begin{proof}
    The statement is a direct consequence of \cite[Lemma 3.3]{gonzalezhernandex2005extreme} in their unconstrained setting. 
\end{proof}

Following \cite{iacovissi2026corruptions}, we will use Markov kernels for modeling changes in statistical learning problems, and talk of \emph{corruption kernels} or more generally \emph{corruptions}. %when they are used of the defining components of statistical learning problems, which are formally defined in the following. 

\subsection{Decision-theoretic setup} 

An experiment $E \in \cM (\cY , \cX)$ is a Markov kernel from the finite label space $(\cY, \Sigma(\cY))$ to the standard Borel attribute space $(\cX, \Sigma(\cX))$, with $\cX \subseteq \reals^d, \, d\ge 1$.
Given $\pi_{\cY} \in \Prob (\cY)$ and $E \in (\cY, \cX)$, their \emph{product composition} $\pi_{\cY} \tm E$ identifies a joint probability distribution $\phi \in \Prob(\XY)$. We call $\pi_{\cY} \tm E$ its \emph{Bayes decomposition}.

Let $\cH \subseteq \cM (\cX,\cY)$ be a model class and $\ell \colon \Prob(\cY)\tm \cY \to \reals_{\ge 0}$ a bounded, measurable loss function.%
\footnote{We consider the model class to be a set of Markov kernels instead of a set of functions. This does not make our work different from existing literature, as the two definitions are in fact equivalent (because Markov kernels and their transitions are in bijection) when we allow models to predict probabilities on the label set instead of ``deterministic'' labels.}
Notice that, as $\mt{h}$ is a measurable function for $h \in \cH$ (Proposition~\ref{prop:markov transition}), the function $(x,y) \mapsto \ell(\mt{h}(x),y)$ is also measurable and bounded. Then, for all $h \in \cH$, $\ell \circ h\colon (x,y) \mapsto \ell(\mt{h}(x),y) \in \Bb(\XY).$ The function $\ell \circ h$, the \emph{attainable prediction losses}, will reappear frequently. It maps to the loss value given the observable attribute $x$ and the label $y$ with respect to a model $h$.
The set of functions obtained through a composition of a loss and a model is written as $\ell \circ \cH \coloneqq \{ \ell \circ \mt{h} \colon h \in \cH\}$.

\begin{definition}[Risk and Bayes Risk] \label{def:br-risk}
    Consider a loss $\ell \colon \Prob(\cY)\tm \cY \to \reals_{\ge 0}$, a model class $\cH \subseteq \cM (\cX,\cY)$ and a joint probability distribution $\phi \coloneqq \pi_{\cY} \tm E \in \Prob (\XY)$, with $\pi_{\cY} \in \Prob(\cY)$ and $E \in \cM(\cY, \cX)$. Then, the \emph{Bayes risk} ($\br$) induced by the attainable prediction losses $\ell\circ\cH$ is, 
    \begin{align}
        \br_{\ell \circ \cH}(\pi_{\cY} \tm E ) &\coloneqq \inf_{h \in \cH } \risk_{\pi_{\cY} \tm E}(\ell \circ h) \,, \\
        \risk_{ \pi_{\cY} \tm E }(\ell \circ h ) &\coloneqq \, \bb E_{ Y \sim \pi_{\cY}} \bb E_{{X} \sim \mt{E}( Y)} \ell(h( X),  Y) \;. 
    \end{align}
    Here, $\risk$ is known as the \emph{risk} associated to $\ell \circ h$ with respect to the joint probability distribution $\pi_{\cY} \tm E$.\footnote{For more background on Bayes risk and conditional Bayes risk see Appendix~\ref{app:remarks on constrained conditional risk}.} %; the notation $h_{ X}$ and $\mt{E}( Y)$ denotes evaluation of a kernel (such as $h$ or $E$) at a random variable (such as $ X$ or $ Y$) and will be used consistently throughout.
\end{definition}

\begin{definition}[Statistical Learning Problem]
\label{def:sup-learn-prob}
    Let $(\XY, \Sigma(\XY))$ be a standard Borel space, with a finite set of labels $\cY$ and a set of attributes $\cX \subseteq \reals^d$. A \emph{statistical learning problem} consists of:
    \begin{enumerate}[label=(\alph*)]
        \item a loss $\ell \colon \Prob(\cY)\tm \cY \to \reals_{\ge 0}$, measurable and bounded,
        \item a model class $\cH \subseteq \cM(\cX, \cY)$,
        \item a data-generating distribution $\phi \in \Prob (\XY)$, which can be obtained by combining a marginal $\pi_{\cY} \in \Prob(\cY)$ and an experiment $E \in \cM(\cY, \cX)$ as $\phi \coloneqq \pi_{\cY} \tm E$,
    \end{enumerate}  
    paired with the risk minimization learning objective, \ie, to find the optimal model $h \in \cH$ that achieves the associated \emph{Bayes risk}.
\end{definition} 
In the following, we will refer to the set of \emph{measurable and bounded losses} of the form $\ell \colon \Prob(\cY)\tm \cY \to \reals_{\ge 0}$ as the set $\cL(\cY)$.
\begin{lemma}[Proof in Appendix~\ref{app:Remarks on Finitely Additive Probability Measures and Risk}]\label{lemma:kernel-is-adjoint-to-risk}
Let $(\XY, \Sigma(\cX) \tm \Sigma(\cY)$ be standard Borel. Consider a  loss $\ell \in \cL(\cY)$, a model $h \in \cM (\cX,\cY)$, and a probability distribution $\phi\in \Prob(\XY)$. Let $\ka \in \cM (\XY, \XY)$. Then, 
%\begin{align}
   $ \risk_\phi [ \mto\ka (\ell \circ h )] =  \risk_{\amto\ka \phi}  [\ell \circ h ].$
%\end{align}
\end{lemma}

\subsection{Constrained Superprediction Set} \label{sec: supervised learning with superpred}

\begin{figure}
    \centering
    \includegraphics[width=0.9\linewidth]{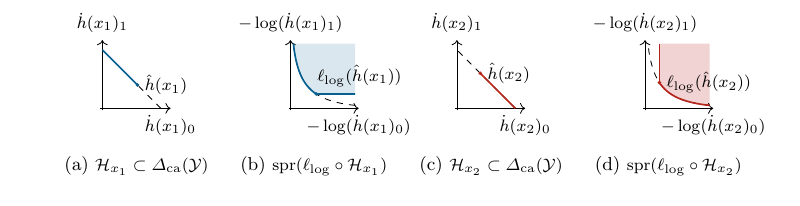}
    \caption{\textbf{Caption: Visualizing the constrained superprediction set.} Given different points in $x_1, x_2\in \cX$, one can plot the set projected constrained superprediction set $\spr (\ell_{\log} \circ \cH_x)$ for $\ell_{\log}(\phi, y) \coloneqq -\log(\phi_y), \phi \in \Prob(\cY)$ and a model class that does not cover the whole simplex. We assume $\hat{h}(x_i) = (\hat{h}(x_i)_0, \hat{h}(x_i)_1) \coloneqq (\eta_i, 1-\eta_i)$ for $\eta_1 = 0.6$ and $\eta_2 = 0.4$.
    %Notice that the $\ell_{\log}(p)$ vector has monotonic entries (decreasing and increasing resp.), preserving the ordering of the points on the simplex when pushed through the loss function.
    \textbf{Alt text:}
    Four panels labeled (a) to (d), each of them containing functions plotted with respect to a Cartesian plane's axes. The first and second panel plot a subset of the two-dimensional simplex; the second and forth panel represent a function and the area above it, corresponding to a projection of a superprediction set.
    }
    \label{fig:spr-partial-simplex}
\end{figure}

A central tool for the characterization of our generalized data processing inequality is the \emph{superprediction set}. We borrow this concept from the context of algorithmic learning theory \cite{vovk1990aggregating, vovk1995game, Kalnishkan2002absence, cabrerapacheco2024axiomatic}%
\footnote{To the best of our knowledge, the term ``superprediction set'' has first been mentioned only in \cite{kalnishkan2002mixability}.} 
and the geometry of loss functions \cite{williamson2016composite,williamson2023geometry,cabrerapacheco2023geometry}. 
Originally, the superprediction set consisted of all values bigger than (``super'') the loss values of any possible class probability prediction (see Appendix~\ref{app:unconstrained to constrained superprediction set}). 
In particular, the optimality properties of a (proper) loss function are exhaustively described by its corresponding superprediction set.
The key insight we leverage is that the conditional Bayes risk can be written as a support function of the superprediction set \cite[Theorem 15]{williamson2023geometry}. As former definitions, \eg, \citet{cranko2021analytic}, neglected constraints to hypotheses we introduce them here and thereby generalize the theory of superprediction sets.

\begin{definition}[Constrained Superprediction Set with Respect to Model Class]
\label{def:Constrained Superprediction Set with Respect to model class}
    Let $\cH \subseteq \cM(\cX, \cY)$ be a model class and $\ell\in \cL(\cY)$ be a loss function. We define the \emph{constrained superprediction set of $\ell \circ \cH$} as
    \begin{align}
        \spr(\ell \circ \cH) \coloneqq \bigcup_{h \in \cH} \big\{ f \in \Bb(\XY) \,\mid\, (\ell \circ h)(x,y) \le f(x,y) \enspace \forall (x,y) \in \XY \big\}\,.
    \end{align}
\end{definition}
For an illustration of the definition consider Figure~\ref{fig:spr-partial-simplex}. We shortly discuss a naive, alternative construction and the relationship of our definition with the classical definition of superprediction sets in Appendix~\ref{app:unconstrained to constrained superprediction set}.
We can also define the superprediction set in a more general fashion, not necessarily assuming a loss and model class but functions in $\Bb(\XY)_{\ge 0}$. This definition will be useful later, when studying how the superprediction set is modified by the action of a Markov kernel.
\begin{definition}[Superprediction Set Operator]
    Let $\cF \subseteq \Bb(\cZ)_{\ge 0}$ non-empty. The \emph{superprediction set of $\cF$} is defined as
    \begin{align}
        \spr(\cF) \coloneqq \bigcup_{a \in \cF} \big\{ f \in \Bb(\cZ) \mid a(z) \le f(z) \enspace \forall z \in \cZ \big\}\,.
    \end{align}
\end{definition}

%One characterizing property of the superprediction set is that it extends a set of positive functions only in the positive orthant; this result extends Eq.~\eqref{eq:unconstrained spr recession cone is reals n-dim} to the constrained learning case. 

% \begin{lemma}
%     \label{lemma:recession cone spr}
%     Let $\cF \subseteq \Bb(\cZ)_{\ge 0}$. The associated superprediction set can be written as
%     $$\spr(\cF) = \cF \oplus \Bb(\cZ)_{\ge 0}\,.$$
% \end{lemma}
% \begin{proof}
%     We first show the set inclusion from left to right. Let $f \in \spr(\cF)$, \ie, there exists $a \in \cF$ such that $a \le f$. Define $g \coloneqq f - a \in \Bb(\cZ)_{\ge 0}$. Hence, clearly $f = a + g \in \cF \oplus \Bb(\cZ)_{\ge 0}$. For the reverse direction, note that any $f \in \cF \oplus \Bb(\cZ)_{\ge 0}$ can be written as $f = a + g$ for some $a \in \cF$ and $g \in \Bb(\cZ)_{\ge 0}$, hence $f \ge a$.
% \end{proof}

%\subsection{Superprediction Set and Bayes Risk}
The key feature of a superprediction set is its relationship to the Bayes risk, which is based on \emph{support functions}, a standard tool in convex analysis. However, different to most works we consider \emph{concave} instead of \emph{convex} support functions, \ie, instead of a supremum we consider an infimum.
\begin{definition}[Support Functions]
\label{def:(concave) support function}
    Let $\cZ$ be a Polish space.
    Let $\cA \subseteq \Bb(\cZ)$ and $\cB \subseteq \ba(\cZ)$ be non-empty, the functions
    % \begin{align}
    %     \sigma_\cA(\mu) \coloneqq \sup_{f \in \cA} \langle f, \mu \rangle, \qquad \sigma_\cB(f) \coloneqq \sup_{\mu \in \cB} \langle f, \mu \rangle, 
    % \end{align}
    % are respectively called \emph{convex support function} of $\cA$ and $\cB$. The functions
    \begin{align}
        \mu \in \ba(\cZ) \mapsto \rho_\cA(\mu) \coloneqq \inf_{f \in \cA} \langle f, \mu \rangle, \qquad f \in \Bb(\cZ) \mapsto \rho_\cB(f) \coloneqq \inf_{\mu \in \cB} \langle f, \mu \rangle, 
    \end{align}
    are called \emph{concave support functions}.
\end{definition} 

When $\cZ$ is a Polish space and $\cF \subseteq \Bb(\cZ)$ a non-empty set of functions, we notice that the support function is invariant to the convex and topological closure (denoted $\overline{\co}$) of the set (Lemma~E.1), hence in particular, 
\begin{align}
    \rho_{\cF}(\mu) = \rho_{\overline{\co}(\cF) }(\mu)\,, \enspace \mu \in \ba(\cZ) \,.
\end{align}

\begin{proposition}[Support Function of Superprediction Set is Bayes Risk]
    \label{prop:support function of superprediction set is bayes risk}
    Let $\ell \in \cL(\cY)$ be a loss function and $\cH \subseteq \cM (\cX, \cY)$ a model class. Then,
    \begin{align}
        \rho_{\spr (\ell \circ \cH) }(\phi) = \br_{\ell \circ \cH}(\phi)\,, \enspace \phi \in \finProb(\XY) \,.
    \end{align}
\end{proposition}

Proposition~\ref{prop:support function of superprediction set is bayes risk}, proved in Appendix~\ref{app:properties supp func and superpred}, generalizes the correspondence of the unconstrained superprediction set with the conditional Bayes risk stated in \cite[Remark 18]{williamson2024information} and \cite[Theorem 15]{williamson2023geometry}.%Given Remark~\ref{remark: br and cbr}, we can also recover the connection with conditional Bayes risk for $\phi \in \Prob(\XY)$.
\footnote{Note that we need to assume that $\phi \in \finProb(\XY)$, as it is not sufficient to assume $\phi \in \Prob(\XY)$ to ensure the existence and (almost sure) uniqueness of the associated conditional expectation---in other words, the validity of the assumptions of the disintegration theorem.}

\begin{remark}
    The reader might have noticed the use of $\Delta(\XY)$ in Proposition~\ref{prop:support function of superprediction set is bayes risk}, which is the set of all \emph{finitely} additive probability measures. Still, the Bayes risk is well-defined (Appendix~\ref{app:remarks on constrained conditional risk}). Note that this is an artifact of our choice to use the dual pairing between bounded measurable functions and finitely additive measures (Section~\ref{sec: stat probs with kernes}). In principle, other pairings could as well drive our central technical contribution and work-horse for our main result, Theorem~\ref{theo:support-ineq-characterization-superpredictionset}. However, those pairings would come with other restrictions such as continuity of the functions and/or the measures. 
    Hence, our setup is rather general and encompasses as well arbitrary countably additive probability measures. To reduce further confusion, we assume, as standard, that a model in the model class or a corruption defined through Markov kernels push to countably additive probability measures (Definition~\ref{def:markov-kernel}).
\end{remark}

% \begin{remark} \label{remark: br upper semicontinuity}
%     For the topologically inclined readers, we note that Proposition~\ref{prop:support function of superprediction set is bayes risk} suggests that the Bayes risk functional is \emph{not} continuous with respect to the $\sbaBb$-topology. That seems counterintuitive, as all linear functionals of the type $\phi \mapsto \langle f, \phi \rangle$ are $\sbaBb$-continuous by definition of $\sbaBb$. The infimum over those functions is, however, is not necessarily continuous, but only upper semicontinuous \cite[Lemma 2.41]{aliprantis2006infinite}. Hence, it follows that even though $\overline\Prob = \finProb$ (Lemma~\ref{lemma:closure of countably additive probabilities are finitely additive probabilities}), the Bayes risk of a finitely additive probability distribution might not be approximated arbitrarily well by the Bayes risk of a countably additive probability distribution.
% \end{remark}

\begin{figure}[t]
\centering
\scalebox{0.9}{
\begin{tikzpicture}[x=1.05cm, y=1.05cm, >=stealth]
  % axes
  \draw[wuaxis] (-2.4,0) -- (2.5,0) node[right, black, font=\small] {$x_1$};
  \draw[wuaxis] (0,-2.4) -- (0,2.5) node[above, black, font=\small] {$x_2$};
  \foreach \t in {-1,1}{
     \draw[black!45] (\t,0.06) -- (\t,-0.06);
     \draw[black!45] (0.06,\t) -- (-0.06,\t);}
  \node[wutiny, below right=1pt and 0pt] at (1,0) {$1$};
  \node[wutiny, below left=1pt and 0pt]  at (-1,0) {$-1$};
  % quadrant names
  \foreach \p/\q in {{(1.85,2.15)}/Q_1, {(-1.85,2.15)}/Q_2,
                     {(-1.85,-2.15)}/Q_3, {(1.85,-2.15)}/Q_4}
      \node[wulab, black!40] at \p {$\q$};
  % one logistic decision boundary  (alpha = (1,1,-1/2):  x_2 = 2x_1 + 2)
  \draw[wuboundary] (-2.15,-2.3) -- (0.2,2.4);
  \node[wulab, ForestGreen, left] at (-1.75,-1.5) {$\{w_\alpha = 0\}$};
  % atoms
  \node[wuone] (a1) at (1,1) {};
  \node[wuzero] (a2) at (-1,1) {};
  \node[wuone] (a3) at (-1,-1) {};
  \node[wuzero] (a4) at (1,-1) {};
  \node[wulab, above right=-1pt and 1pt] at (a1) {$x^{(1)}$};
  \node[wulab, above right=-1pt and 1pt] at (a2) {$x^{(2)}$};
  \node[wulab, below right=-1pt and 1pt] at (a3) {$x^{(3)}$};
  \node[wulab, above right=1pt and 3pt] at (a4) {$x^{(4)}$};
  % the atom this particular boundary must sacrifice
  \draw[orange, line width=0.7pt] (a4) circle (0.2);
  \node[wulab, orange, below=4pt] at (a4) {misclassified};
  % legend
  \begin{scope}[shift={(-2.4,-3.1)}]
    \node[wuone] at (6,3.5) {};
    \node[wulab, right=3pt] at (6,3.5) {$y = 1$};
    \node[wuzero] at (6,3) {};
    \node[wulab, right=3pt] at (6,3) {$y = 0$};
    %\node[wulab] at (6,2) {each atom has mass $\tfrac14$};
  \end{scope}
\end{tikzpicture}
}
\caption{
\textbf{Caption: The base distribution $\phi_{(1,0,1,0)}$.} Neighboring quadrants carry different labels
and the base distribution is invariant under $x \mapsto -x$. A model $h_\alpha \in \cH$
predicts $\sigma(w_\alpha(x))$ with $w_\alpha(x) = \alpha_1x_1 + \alpha_2x_2 +
\alpha_0$, so its decision boundary is the line $\{w_\alpha = 0\}$ (dashed);
the one drawn, $\alpha = (1,1,-0.5)$, classifies three atoms correctly and has to sacrifice the fourth (circled). Each atom has mass $\tfrac14$. No model does better (Lemma~\ref{lemma:wu-param}).
\textbf{Alt text}: Four quadrants of a Cartesian plane, labeled Q underscore i, i from 1 to 4. Each of the quadrants contains a point, labeled x underscore i. Two points are drawn as blue dots, two as red squares: the different placeholder indicated the two different types of labels that can assigned to a point. A green dotted line, representing a classifier called w underscore alpha, divides the Cartesian plane in two parts. Because of how the division happens, one red-square point appears to be incorrectly classified , which is indicated with an orange label (which reads ``misclassified'') next to said point.
}
\label{fig:warm-up-example distribution}
\end{figure}
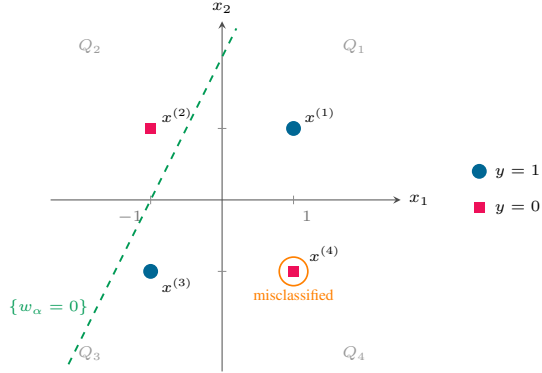

\begin{runex}[The Learning Problem]
\label{sec:wu-setup}
Let $\cX = \reals^2$ and $\cY = \{ 0,1\}$. Let $\phi_{\eta} \in \Deltaca(\cX\times\cY)$ be uniform on the four atoms
\begin{align}
    \label{eq:wu-atoms}
    x^{(1)} := (1,1), \quad x^{(2)} := (-1,1), \quad x^{(3)} := (-1,-1), \quad x^{(4)} := (1,-1) ,
\end{align}
one in each quadrant, \ie, $\phi_{\eta}(X = x^{(i)}) = \frac{1}{4}$ for all $i \in  \{ 1,2,3,4\}$, and conditional label probabilities
$\eta = (\eta_1,\dots,\eta_4) \in [0,1]^4$, $\eta_i \coloneqq \phi_\eta(Y = 1 \mid X = x^{(i)})$.
The running example will mainly concern $\phi_{(1,0,1,0)}$, see Figure~\ref{fig:warm-up-example distribution}.
% We also use the linearly separable distribution $\Psi := P_{(1,0,0,1)}$, for which $y=1$ if and only if $x_1 > 0$.
We consider prediction functions (mapping input $x$ to a probability distribution),
\begin{align}
    h_\alpha (x) = \bigl( 1 - \sigma(w_\alpha(x)),\, \sigma(w_\alpha(x)) \bigr)
    \in \Deltaca(\cY),
    \qquad
    w_\alpha(x) := \alpha_1 x_1 + \alpha_2 x_2 + \alpha_0 ,
\end{align}
for $\alpha_0,\alpha_1,\alpha_2 \in \reals$, with $\sigma(w) := (1+e^{-w})^{-1}$, and the model class is
\begin{align}
    \label{eq:wu-Hlog}
    \cH := \{ h_\alpha \in \Mkern(\cX,\cY) \mid \alpha = (\alpha_0,\alpha_1,\alpha_2) \in \reals^3 \}
    \subsetneq \Mkern(\cX,\cY) =: \cH_{\mathrm{all}} ,
\end{align}
where $\cH_{\mathrm{all}}$ denotes the full (unconstrained) model class. In other words, $\cH$ is the model class of logistic regression. Note that $\cH$ is closed under affine transformations (\cf Proposition~\ref{prop:sufficient conditions dpi attribute corruption}).

Combined with the cross-entropy loss $\ell_{\log} \colon \Deltaca(\cY )\times \cY \rightarrow \reals_{\ge 0}$,
\footnote{To be precise, in the rest of the paper's running example we will need to truncate the cross-entropy loss $\ell_{\log}$ to a finite interval to make Corollary~\ref{corollary:kernel-gdpi-characterization} applicable. See Remark~\ref{rm:wu-boundedness} in Appendix~\ref{app:wu-dpi-breaks}.}
$(\pi,y) \mapsto -\log(\pi_y)$, we obtain the following Bayes risks for $\eta \in [0,1]^4$,
\begin{align}
    \label{eq:wu-BRall}
    \BR_{\ell_{\log}\circ\cH_{\mathrm{all}}}(\phi_\eta) &= \tfrac14 \textstyle\sum_{i=1}^4 \Ent(\eta_i),\\
    \label{eq:wu-BRlog}
    \BR_{\ell_{\log}\circ\cH}(\phi_\eta) &= \tfrac14 \inf_{w \in \reals^4 \colon w_1 + w_3 = w_2 + w_4} \textstyle\sum_{i=1}^4 G_{\eta_i}(w_i).
\end{align}
We write $\Ent(\eta) := -\eta\log\eta - (1-\eta)\log(1-\eta)$ for the binary
entropy, with $\Ent(0) = \Ent(1) = 0$, and
\begin{align}
    G_\eta(w) := \eta \log(1+e^{-w}) + (1-\eta)\log(1+e^{w}),
    \qquad \eta \in [0,1],\ w \in \reals ,
\end{align}
for the expected cross-entropy of the prediction $\sigma(w)$ under label
probability $\eta$. For the proofs of \eqref{eq:wu-BRall} and \eqref{eq:wu-BRlog} see Appendix~\ref{app:wu-setup}.
%Each $G_\eta$ is strictly convex, $\min_{w \in \reals} G_\eta(w) = \Ent(\eta)$, attained at $w = \log\frac{\eta}{1-\eta}$ for $\eta \in (0,1)$, and $G_\eta + G_{1-\eta} = G_1 + G_0$.
\end{runex}

\section{A Counterexample to the DPI Holding in the Constrained Case} \label{sec:counterexample}
With Markov kernels, support functions and superprediction sets at our hands, let us get back to the data processing inequality.
A straightforward way to generalize the DPI \eqref{eq:DPI-informal} to the constrained setting is to consider the inequality
\begin{align}\label{eq: constrained experiment dpi}
\br_{{\color{OrangeRed}\ell \circ \cH}} [\pi_{\cY} \tm E] \le \br_{{\color{OrangeRed}\ell \circ \cH}} [\pi_{\cY} \tm   \tilde{E}] \,,
\end{align}
where $ \tilde{E}$ is a randomization of $E$. The randomization process is carried out via some Markov kernel $\ka \in \cM (\cX, \cX)$ as 
\begin{align}
     \tilde{E}(\cA, y) \coloneqq \int_{ \tilde x \in \cA}  \tilde E(d \tilde x, y) \coloneqq \int_{ \tilde x \in \cA} \int_{x \in \cX} \ka(dx,  \tilde x) E(d x, y),  \quad \cA \subseteq \cX.
\end{align}
A natural question is whether the classical DPI result extends to this setting, which formally is equivalent to asking \emph{does Eq.~\eqref{eq: constrained experiment dpi} hold?}

When the model is well-specified, \ie, the optimal model in the unconstrained setting is contained in $\cH$, for both data distributions, the answer is yes, for all $\pi_{\cY}$ and $\ell$, as both constrained and unconstrained Bayes' risks assume the same value.
The situation changes when the model class is constrained and not necessarily well-specified: in this case the swapping of integration and infimum \cite[see][Theorem 16.40]{rockafellar1998variational} which drives the proof of \cite[Theorem 4.3]{degroot1962uncertainty} breaks, and Eq.~\eqref{eq:DPI-informal} cannot be shown this way.
Upon further investigation, we discover that in this setting we are less fortunate than in \citet{degroot1962uncertainty}'s framework: the inequality in Eq.~\eqref{eq: constrained experiment dpi} can fail to hold, and a simple counterexample can be constructed.

\begin{example}[Counterexample for Constrained DPI] \label{counterexample for costrained dpi}
    Let $\cX \coloneqq \{ 0,1\}$ and $\cY \coloneqq \{ 0,1\}$. We consider the experiment $E$ and the joint data distribution $\phi$ in their matrix representation, with columns being the $\cY$-values and rows being the $\cX$-values:
    \begin{align}
        E \coloneqq
        \begin{pmatrix}
            1 & 0\\
            0 & 1
        \end{pmatrix} \,, \quad
        \phi \coloneqq \pi_{\cY} \tm E =
        \begin{pmatrix}
            1-p & 0\\
            0 & p
        \end{pmatrix} \,,
    \end{align}
    where $\pi_{\cY} = [1-p , p]$, $p\in (0,1)$ is some general label prior with full support.
    Let $h\in\cH$ be a model, written as $\mt{h}(x) = (\mt{h}(x)_0, \mt{h}(x)_1) \in \Prob(\cY)$, and $\ell \in \cL(\cY)$ be a loss function such that $\ell(y, \mt{h}(x)_y) = 0 $ if and only if $\mt{h}(x)_y(x) = 1$, \ie, the model assigns probability one to the correct label. This implies that for $\mt{h}(x)_y(x) < 1$ the loss function will be strictly greater than zero. 
    Let $\cH \coloneqq \{ \mt{h} \colon \cX \to \Delta (\cY) \,|\, \mt{h}(x)_y = 1-\inset_{\{y\}}(x) \}$, where $\inset_{\cA}(x)$ is the indicator function of the set $\cA$, be a model class with only one element in it. Written in matrix form, that element is equal to: 
    \begin{align}
        h \coloneqq
        \begin{pmatrix}
            0 & 1\\
            1 & 0
        \end{pmatrix}\,.
    \end{align}
    Evaluating the Bayes risk gives
   % \begin{align}
       $ \min_{h \in \cH}\bb{E}_\phi[(\ell \circ h) (X,Y)] = \ell(\mt{h}(0)_0,0)* (1-p) + \ell(\mt{h}(1)_1,1) * p > 0 \,.$
    %\end{align}
    If we consider the Markov kernel which puts full corruption on $\cX$, we obtain the corrupted distributions,
    \begin{align}
         \tilde{E} \coloneqq
        \begin{pmatrix}
            0.5 & 0.5\\
            0.5 & 0.5
        \end{pmatrix} \,, \quad
         \tilde{\phi} \coloneqq
        \begin{pmatrix}
            \f{(1-p)}2 & \f{p}2\\
            \f{(1-p)}2 & \f{p}2
        \end{pmatrix} \,.
    \end{align}
    This gives the Bayes risk,
    %\begin{align}
        $\min_{h \in \cH}\bb{E}_{ \tilde\phi}[\ell(\mt{h}(X),Y)] = 0.5 * \left[ \, \ell(\mt{h}(0)_0,0)* (1-p) + \ell( \mt{h}(1)_1,1) *p \, \right]$
    %\end{align}
    which is half of the Bayes risk of the original task. Hence, the inequality stated in Eq.~\eqref{eq: constrained experiment dpi} does not hold in this setting.
\end{example}
The example used an extremely restricted model class, was constructed ad-hoc such that it mispredicts before corruption. 
%The randomization over the space $\cX$ then helped to increase the likelihood to sample an $x$  labeled correctly ``by accident''.
%The example shows that the data processing inequality stated in Eq.~\eqref{eq: constrained experiment dpi} does not generally hold. 
Although this construction is not realistic, as it used a singleton model class, is showcased a type of reasoning that can be extended to larger $\cH$'s.

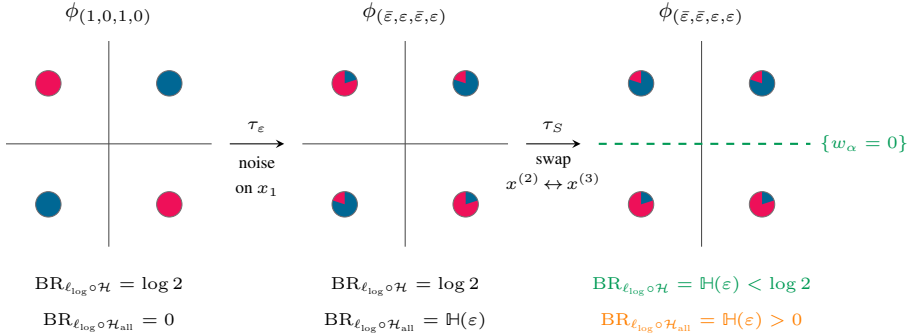
\begin{figure}[t]
\centering
\begin{tikzpicture}[x=0.8cm, y=0.8cm, >=stealth]
 % panels 1 and 2 get both axes, panel 3 only the vertical one
 \foreach \sh in {0, 4.9}{
   \begin{scope}[shift={(\sh,0)}]
     \draw[black!70] (-1.7,0) -- (1.7,0);
     \draw[black!70] (0,-1.7) -- (0,1.7);
   \end{scope}}
 \draw[black!70] (9.8,-1.7) -- (9.8,1.7);
 % panel 1: clean
 \begin{scope}
   \wuatom{(1,1)}{1}{0.21}   \wuatom{(-1,1)}{0}{0.21}
   \wuatom{(-1,-1)}{1}{0.21} \wuatom{(1,-1)}{0}{0.21}
   \node[font=\footnotesize] at (0,2.15) {$\phi_{(1,0,1,0)}$};
 \end{scope}
 % panel 2: after the noise
 \begin{scope}[shift={(4.9,0)}]
   \wuatom{(1,1)}{0.8}{0.21}   \wuatom{(-1,1)}{0.2}{0.21}
   \wuatom{(-1,-1)}{0.8}{0.21} \wuatom{(1,-1)}{0.2}{0.21}
   \node[font=\footnotesize] at (0,2.15)
     {$\phi_{(\bar\varepsilon,\varepsilon,\bar\varepsilon,\varepsilon)}$};
 \end{scope}
 % panel 3: after noise and swap
 \begin{scope}[shift={(9.8,0)}]
   \wuatom{(1,1)}{0.8}{0.21}   \wuatom{(-1,1)}{0.8}{0.21}
   \wuatom{(-1,-1)}{0.2}{0.21} \wuatom{(1,-1)}{0.2}{0.21}
   \draw[wuboundary] (-1.7,0) -- (1.8,0);
   \node[wulab, ForestGreen, right] at (1.8,0) {$\{w_\alpha = 0\}$};
   \node[font=\footnotesize] at (0,2.15)
     {$\phi_{(\bar\varepsilon,\bar\varepsilon,\varepsilon,\varepsilon)}$};
 \end{scope}
 % arrows between the panels
 \draw[->, black] (2.0,0) -- (2.9,0)
   node[midway, above=1pt, wulab] {$\ta_\varepsilon$}
   node[midway, below=1pt, wulab, align=center] {noise\\on $x_1$};
 \draw[->, black] (6.9,0) -- (7.8,0)
   node[midway, above=1pt, wulab] {$\ta_S$}
   node[midway, below=1pt, wulab, align=center] {swap\\$x^{(2)}\!\leftrightarrow\!x^{(3)}$};
 % per-panel Bayes risks
 \foreach \sh/\lg/\fu in {0/{\log 2}/{0}, 4.9/{\log 2}/{\Ent(\varepsilon)}}{
   \node[font=\scriptsize] at (\sh,-2.35) {$\BR_{\ell_{\log}\circ\cH} = \lg$};
   \node[font=\scriptsize] at (\sh,-2.95) {$\BR_{\ell_{\log}\circ\cH_{\mathrm{all}}} = \fu$};}
 \node[font=\scriptsize, ForestGreen] at (9.8,-2.35)
   {$\BR_{\ell_{\log}\circ\cH} = \Ent(\varepsilon) < \log 2$};
 \node[font=\scriptsize, orange] at (9.8,-2.95)
   {$\BR_{\ell_{\log}\circ\cH_{\mathrm{all}}} = \Ent(\varepsilon) > 0$};
 % the two conclusions
 % \node[font=\scriptsize, ForestGreen] at (4.9,-3.75)
 %   {constrained risk \emph{falls}: $\Ent(\varepsilon) < \log 2$ \ --- \ the GDPI fails};
 % \node[font=\scriptsize, orange] at (4.9,-4.3)
 %   {unconstrained risk \emph{rises}: $\Ent(\varepsilon) > 0$ \ --- \ the classical DPI holds};
\end{tikzpicture}
\caption{%
\textbf{Caption: Corruption breaks the data processing inequality.} Every disk is an atom of mass $\tfrac14$; its blue
fraction is $\eta_i = \phi(Y = 1 \mid X = x^{(i)})$ and its red fraction
$1 - \eta_i$, drawn for $\varepsilon = 0.2$ and $\bar\varepsilon :=
1-\varepsilon$. Left: the clean problem $\phi_{(1,0,1,0)}$ of
Figure~\ref{fig:warm-up-example distribution}. Middle: the attribute corruption
$\ta_\varepsilon = (1-\varepsilon)\ta_{\mathrm{id}} + \varepsilon\ta_\rho$
resamples the first coordinate, which blurs the labels but leaves the
constrained Bayes risk at the chance level $\log 2$. Right: the bijection $S$
then swaps the two left atoms; now the label depends on $x_2$ alone, and the
model with $\alpha = (0,0,\log\frac{\bar\varepsilon}{\varepsilon})$, whose
boundary is the dashed line, is optimal even among all measurable models. The
composite corruption from $\phi_{(1,0,1,0)}$ over $\phi_{(\bar\varepsilon,\varepsilon,\bar\varepsilon,\varepsilon)}$ to $\phi_{(\bar\varepsilon,\varepsilon,\bar\varepsilon,\varepsilon)}$ thus lowers the Bayes risk of logistic regression while
raising the Bayes risk of the full class $\cH_{\mathrm{all}}$: it destroys
information in the classical sense and still helps the constrained learner,
contradicting the classical data processing inequality. $\Ent$ denotes the binary entropy.
\textbf{Alt text}: Three Cartesian planes with their four quadrants; each quadrant contains a point. Each point is depicted using a pie chart, representing the probability of having a label 1 assigned given the value of the attribute x. The first Cartesian plane represents the clean data. The second and third planes represent the effect of two distinct attribute corruptions, changing pie charts. The third plane shows that the classifier w underscore alpha correctly classifies the points after applying two corruptions.
}
\label{fig:wu-dpi-breaks}
\end{figure}

\begin{runex}[Counterexample for DPI via Logistic Regression]
\label{sec:wu-dpi-breaks}
Let us take a closer look on the Bayes risks of the full model class and logistic regression. We consider the base distribution $\phi_{(1,0,1,0)} = \pi_{\cY} \tm {E}$ and the corrupted distribution $\phi_{(\bar\varepsilon, \bar\varepsilon, \varepsilon, \varepsilon)} = \pi_{\cY} \tm \tilde{E}$. The randomization which maps $\phi_{(1,0,1,0)}$ to $\phi_{(\bar\varepsilon, \bar\varepsilon, \varepsilon, \varepsilon)}$ is a \emph{Markov kernel} which applies noise on the first $x$-coordinate and then swaps the second $x$-coordinate. For details and a proper mathematical treatment see Appendix~\ref{app:wu-dpi-breaks}. As expectable by the standard DPI, for every $\varepsilon \in (0,\tfrac12)$,
\begin{align}
    \BR_{\ell_{\log}\circ\cH_{\mathrm{all}}}(\phi_{(\bar\varepsilon, \bar\varepsilon, \varepsilon, \varepsilon)}) = \Ent(\varepsilon) > 0
    = \BR_{\ell_{\log}\circ\cH_{\mathrm{all}}}(\phi_{(1,0,1,0)}) ,
\end{align}
\ie, the unconstrained problem becomes strictly harder. Surprisingly, however,
\begin{align}
    \BR_{\ell_{\log}\circ\cH}(\phi_{(\bar\varepsilon, \bar\varepsilon, \varepsilon, \varepsilon)}) = \Ent(\varepsilon) < \log 2
    = \BR_{\ell_{\log}\circ\cH}(\phi_{(1,0,1,0)}) ,
\end{align}
\ie, the constrained problem becomes \emph{strictly easier} under the same corruption, and the constrained DPI does not hold. In other words, the randomization from $\phi_{(1,0,1,0)}$ to $\phi_{(\bar\varepsilon, \bar\varepsilon, \varepsilon, \varepsilon)}$  \emph{destroys} ``constrained information'' in the sense of DeGroot (Appendix~\ref{app:predictive information}). %: for the same corruption, the classical DPI holds strictly on the unconstrained problem, while the very same corruption strictly \emph{improves} the constrained problem. 
Intuitively, the prediction problem is made ``linearly separable'' by the noise. Without noise (right hand side, lower equation) the deterministic labels are assigned through a XOR function in the signs of the $x$-coordinates; the optimal logistic regressor is constant, and hence incurs a relatively large loss. 
However, the noise-and-swap corruption (see Figure~\ref{fig:wu-dpi-breaks}) which leads to $\phi_{(\bar\varepsilon, \bar\varepsilon, \varepsilon, \varepsilon)}$ makes the label distribution ``linearly separable''. That explains why the logistic regression class achieves fully optimal performance (left hand side, lower equation). %Despite the noise on the labels, the predictive performance of the optimal logistic regressor is equal to the best possible predictor and hence undercuts the Bayes risk on the original distribution. \Laura{not sure what this last sentence adds}
The formal argument is given in Appendix~\ref{app:wu-dpi-breaks}.
\end{runex}

\section{Generalized Data Processing Inequality} \label{sec: comparison attainable prediction losses}

\begin{figure}[t]
    \centering
    \includegraphics[width=0.73\linewidth]{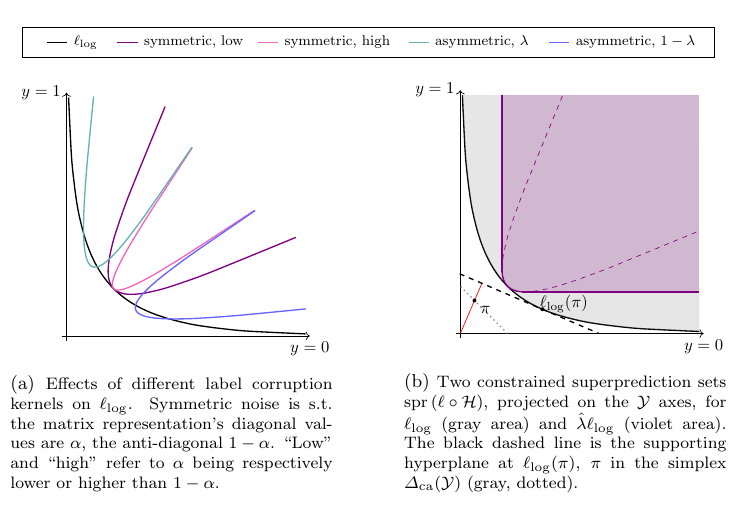}
    \caption{%
    \textbf{Caption: Visualizing the effect of corruption on the superprediction set for an unconstrained problem} (Corollary~\ref{corollary:kernel-gdpi-characterization}).
    We consider a problem composed by: a model class $\cH$, whose transitions $\mt{h}$ correspond all measurable functions mapping to the relative interior of $\Prob(\cY)$; the logarithmic loss $\ell_{\log}$; $|\cX| = 1\,; \cY=\{0,1\}$. 
    We apply some label corruption, \ie, $\ka = \idx \otm \la$, only acting in the loss as depicted in panel (a). 
    In the \emph{asymmetric} cases depicted, we can graphically see that the  set containment condition, and hence generalized data processing inequality, \emph{does not hold}.
    Panel (b) represents the clean superprediction set and the corrupted one, associated to the symmetric, low corruption. %The latter set it such that some parts of the loss function lie inside the set (drawn as a violet dashed curve). 
    In this case we instead see that the set containment \emph{holds}, since $\cspr(\mto{\la}\ell_{\log} \circ \cH) \subset \cspr(\ell_{\log} \circ \cH)$. %In the next section (Proposition~\ref{prop:sufficient conditions dpi label corruption}), we will formally prove that this holds for a generalization of symmetric losses.
    To better understand visually why set containment implies the GDPI, notice that the Bayes risk associated to the clean problem and $\pi = (0.3, 0.7)$ is the length of the {\color{OrangeRed} red} line, connecting the origin to the hyperplane (dashed line) and perpendicular to it; a similar observation holds for the corrupted superprediction set. 
    \textbf{Alt text}: Two panels, labeled (a) and (b). Panel (a) depicts different functions with respect the Cartesian plane's axes. One of this is the clean loss function ell; the other four are corrupted loss functions, using four distinct corruption kernels. Panel (b) depicts the superprediction sets associated to the clear loss ell and on e corrupted loss, as well as one supporting hyperplane of the superprediction set for the clean loss. This supporting hyperplane is identified by its distance vector from the origin (in red), passing through the simplex at the probability pi.
    } 
    \label{fig:characterization-viz}
\end{figure}

We have discussed above why the classical data processing inequality result does not always hold in the constrained case, and now turn our attention to how to generalize it to our new setting. 
In the introduction, we already stated our goal of finding a different data processing inequality of the form of Eq.~\eqref{eq:GDPI-informal}, which by design can be studied with tools different from \citet{blackwell1951comparison}'s and \citet{degroot1962uncertainty}'s. 
In addition we also want to be able to encompass all cases of problem modifications that can occur at the probability distribution level. 
We therefore adopt the corruption taxonomy introduced by \citet{iacovissi2026corruptions}, which has been proved to satisfy this requirement if we strictly consider ``one-step corruptions'', \ie, corruptions that do not happen in multiple time steps. 
By their analysis, we can always select a joint kernel $\ka \in \cM(\XY, \XY)$ and be assured that its one-step formula is either irreducible to a combination of further simpler corruptions, or equal to the combination of an attribute (\ie, only corrupting $\cX$, which includes randomization kernels used for classical DPI) or label corruption (\ie, only corrupting $\cY$). We will formally define these two notions of corruptions later in the paper, and for now only make use of joint corruptions $\ka \in \cM(\XY, \XY)$. 

\begin{definition}[GDPI for Bayes Risk]\label{def:gdpi for br}
Given a loss function $\ell\in\cL(\cY)$ and a model class $\cH \subseteq \cM (\cX, \cY)$, we write the generalized data processing inequality for Bayes risk as
    \begin{align}
        \br_{\ell \circ \cH} ( \phi ) \le \br_{\ell \circ \cH} (\amto\ka \phi) \,, \quad \phi \in \finProb(\XY)\,, \quad \ka \in \cM(\XY, \XY) .
    \end{align}
\end{definition}

In words, we are asking the joint corruption $\ka$ to be bad for a fixed whole statistical learning problem with $\phi=\pi_\cY \tm E$ instead of solely for the experiment $E$, as $\ka$ can also modify the marginal $\pi_\cY$.
Note that this version of the data processing inequality is clearly more general than the one used in the statement of Eq.~\eqref{eq:DPI-informal}, as choosing $\cH = \cM (\cX, \cY)$ and $\ka \in \cM (\cX, \cX)$ we obtain the classical data processing inequality formula.%
\footnote{Here we only claim for the formula, not for the data processing inequality \emph{result}. However, we will prove later that also the DPI theorem is subsumed by our framework.}

Given the nature of the more general inequality, we turn our attention to studying the validity of Definition~\ref{def:gdpi for br} for a fixed set of attainable prediction losses while varying the probability distribution $\phi$.\footnote{Following Blackwell's classical path is not feasible: a notion of sufficiency for joint probability distributions is not meaningful, as two distributions are always related by at least a degenerate kernel constantly equal to one of the two. } As we have seen in Proposition~\ref{prop:support function of superprediction set is bayes risk}, one can relate the Bayes risk associated to a probability distribution $\phi$, a loss $\ell$, and a model class $\cH$ to the support function of the set $\sprH$ evaluated on $\phi$. Therefore, we can characterize the GDPI for Bayes risk by means of geometrical properties. This is similar in spirit but not equivalent to the work from  \citet{blackwell1951comparison} and \citet{degroot1962uncertainty}.

\begin{theorem}[Equivalence between Set and Support Functions  of the Superprediction Orderings]
    \label{theo:support-ineq-characterization-superpredictionset}
    Consider $\cF, \cG \subseteq \Bb(\XY)_{\ge 0}$ non-empty. Then, 
    $$\cspr \, \cF \supseteq \cspr \, \cG \quad \Leftrightarrow \quad \rho_{\cF}(\phi) \le \rho_{\cG}(\phi) \quad \forall \phi \in \finProb(\XY) \,. $$
\end{theorem}
\begin{proof}
    That set containment implies the inequality follows directly from that $\rho$ is the concave support function defined in Definition~\ref{def:(concave) support function} as an infimum, and that is is immune to convexification and closure of the set. 
    The reverse direction is proved in Appendix~\ref{app:characterization proof}.
    %The reverse direction is a consequence of Lemma~\ref{lemma:co-spr-definition-with-fin-probabilities}.
\end{proof}

%Theorem~\ref{theo:support-ineq-characterization-superpredictionset} above is a characterization of some generalized version of the data processing inequality for Bayes risk. %Key to this interpretation is Lemma~\ref{lemma:co-spr-definition-with-fin-probabilities}, as it restricts the set of measures on which we have to verify the inequality to the positive and normalized ones. 
%Considering specific function sets $\cF$ and $\cG$ allow for using this result for learning problems as we see with the next result. 
%In addition, we can make the set containment requirement weaker or find equivalent conditions by simply using properties of convex hulls. A complete hierarchy of the set containment conditions which hold is illustrated in Figure~\ref{fig:implications for set containment condition}.

%\subsection{Orderings for Different Model Classes and Loss Functions}

%One flavor of corollary of Theorem~\ref{theo:support-ineq-characterization-superpredictionset} provides a statement on the ordering of Bayes risk functions for different losses and model classes.

%It immediately follows our main result. A visualization of this property is given in Figure~\ref{fig:characterization-viz}.
\begin{corollary}[GDPI under Corruption]
    \label{corollary:kernel-gdpi-characterization}
    Let $\ell\in\cL(\cY)$ be a loss function and $\cH \subseteq \cM(\cX,\cY)$ a model class. 
    Consider a Markov kernel $\ka \in \cM(\XY, \XY)$. 
    Then,
    \begin{align}
        \cspr(\ell \circ \cH) \supseteq \cspr(\mto{\ka}(\ell \circ \cH))
        \quad \Leftrightarrow \quad 
        \br_{\ell \circ \cH}(\phi) \le \br_{\ell \circ \cH}( \amto{\ka}\phi)\,, \quad \forall \phi \in \finProb(\XY) \,.
    \end{align}
\end{corollary}
\begin{proof}
    Define $\cF \coloneqq \ell \circ \cH$ and $\cG \coloneqq \mto{\ka}(\ell \circ \cH)$.
    Apply Theorem~\ref{theo:support-ineq-characterization-superpredictionset},  Proposition~\ref{prop:support function of superprediction set is bayes risk} and Lemma~\ref{lemma:kernel-is-adjoint-to-risk} for the statement to follow.
\end{proof}

\begin{figure}
    \centering
    \includegraphics[width=0.8\linewidth]{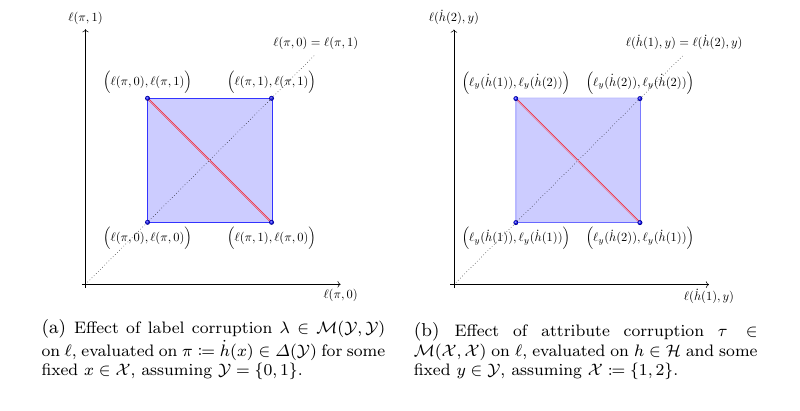}
    \caption{%
    \textbf{Caption: Visualizing the effect of corruption on one point of the superprediction set.} Given a point $\ell \circ h \in \spr(\ell \circ \cH)$, we consider the different forms the point $\mto\ka(\ell\circ h)$ can take for different corruption types and properties. {\color{OrangeRed} Red line}: bistochastic corruption; {\color{blue!75!black} Blue marks}: points generated by a degenerate binary kernel that, in its matrix form, has ones only in one row and zeros otherwise; \inlinecolorbox[blue!20]{Blue area}: corruption outcome for all possible kernels.
    \textbf{Alt text:}
    Two panels, labeled (a) and (b). Panel (a) depicts a blue square embedded in a Cartesian plane with its axes. The vertices of the square are: one two-dimensional point representing a loss function computed on a probability pi, once against label 0 and once against label 1; three other two-dimensional points that are corrupted versions of the same evaluated loss. The diagonal of the square is red, and corresponds to points obtained using a bistochastic corruption kernel. Panel (b) represents the same elements as panel (a), but the vertices of the square are: one two-dimensional point representing a loss function computed on a model h, once evaluated at  evaluated at x equal 1, once at x equal 2; three other two-dimensional points that are corrupted versions of the same evaluated loss and model. In this case, the label y is left as a free parameter.
    }
    \label{fig:kernel-on-spr}
\end{figure}

\begin{definition}
    Let $\ell\in\cL(\cY)$ be a loss function and $\cH \subseteq \cM(\cX,\cY)$ a model class, we define the set of Markov kernels for which the GDPI for Bayes risk holds as
\begin{align}
    \BRIkerns(\ell \circ \cH) \coloneqq \{\, \ka \in \cM(\XY, \XY) \,\mid\, \br_{\ell \circ \cH}(\phi) \le \br_{\ell \circ \cH}( \amto{\ka}\phi) \quad \forall \phi \in \finProb(\XY) \,\}.
\end{align}
\end{definition}

In words, the set $\BRIkerns(\ell \circ \cH)$ consists of all joint corruptions which deteriorate the performance of $\cH$ with respect to $\ell$ on all base distributions.
Note that the joint identity kernel $\id \in \cM (\XY, \XY)$ always belongs to this set, hence it is ensured to be non-empty.
Using Corollary~\ref{corollary:kernel-gdpi-characterization}, we can alternatively write,
\begin{align}
    \BRIkerns(\ell \circ \cH) = \{ \,\ka \in \cM(\XY, \XY) \,\mid\, \cspr(\mto{\ka}(\ell \circ \cH)) \subseteq \cspr(\ell \circ \cH) \,\}.
\end{align}
%We now prove some useful property of this set.
\begin{proposition}[$\BRIkerns$ is a Convex Set, proof in Appendix~\ref{app:convexity of BRIkerns}]
\label{prop:BRIkerns is convex}
    Let $\ell\in\cL(\cY)$ be a loss function and $\cH \subseteq \cM (\cX, \cY)$ a model class. The associated set $\BRIkerns(\ell \circ \cH)$ is convex.
\end{proposition}

\subsection{Some Corruption Kernels of Interest}
To give the reader some intuition for possible corruptions, we list different types of corruption kernels which have a relevant structure for our data processing results. They are subcases of the taxonomy presented in \cite{iacovissi2026corruptions}.

\begin{definition}[Label Corruption]
\label{def:structured label corruption}
    Let $\cX$ be an input space and $\cY$ be an output space.
    \begin{enumerate}
        \item A Markov kernel $\la \in \cM (\XY, \cY)$ % \colon \XY \tm \Sigma(\cY) \rightarrow [0,1]$
        is called \emph{label corruption}.
        \item A Markov kernel $\la \in \cM ( \cY, \cY)$ %\colon \cY \tm \Sigma(\cY) \rightarrow [0,1]$
        is called \emph{simple label corruption}.
        \item A Markov kernel $\la \in \cM (\XY, \cY)$ %\colon \XY \tm \Sigma(\cY) \rightarrow [0,1]$
        is called \emph{$\cF$-label corruption}, if for every $x \in \cX$ there exists a finite tuple $(\alpha^x_i)_{i \in I} \in [0,1]^{|I|}$ such that $\sum_{i \in I} \alpha^x_i = 1$ and a finite set of measurable mappings $\cF \subseteq \{ \cY \to \cY \}$ for every $i \in I$, such that, for $f^x_i \in \cF$
        \begin{align}
            \la_x \in \cM (\cY, \cY): (y,\cA) \mapsto \sum_{i \in I} \alpha^x_i \ka_{f^x_i}(y,\cA) , \quad \forall y \in \cY, \cA \in \Sigma(\cY),
        \end{align}
        where the deterministic kernels $\ka_{f^x_i}$ are defined as in Eq.~\eqref{eq:deterministic markov kernel}.
        \item A Markov kernel $\la \in \cM (\XY, \cY)$ %$ \colon \cY \tm \Sigma(\cY) \rightarrow [0,1]$
        is called \emph{bistochastic label corruption}, if it is a $\cF$-label corruption with $\cF$ being the set of all measurable, bijective mappings.
    \end{enumerate}
\end{definition}
%Analogously, we can also define the attribute corruption counterpart to the previous definition.

\begin{definition}[Attribute Corruption]
\label{def:structured attribute corruption}
    Let $\cX$ be an input space and $\cY$ be an output space.
    \begin{enumerate}
        \item A Markov kernel $\ta \in \cM (\cX \times \cY, \cX)$ is called \emph{attribute corruption}.
        \item  A Markov kernel $\ta \in \cM (\cX, \cX)$ is called \emph{simple attribute corruption}.
        \item  A Markov kernel $\ta \in \cM (\cX \times \cY, \cY)$ is called \emph{$\cG$-attribute corruption}, if for every $y \in \cY$ there exists a finite tuple $(\alpha^y_i)_{i \in I} \in [0,1]^{|I|}$ such that $\sum_{i \in I} \alpha^y_i = 1$ and a finite set of measurable mappings $\cG \subseteq \{ \cX \to \cX\}$ for every $i \in I$, such that, for $g^y_i \in \cG$,
    \begin{align}
        \ta_y \in \cM (\cX, \cX): (x,\cA) \mapsto \sum_{i \in I} \alpha^y_i \ka_{g^y_i}(y,\cA), \quad \forall x \in \cX, \cA \in \Sigma(\cX),
    \end{align}
    where the deterministic kernels $\ka_{g^y_i}$ are defined as in Eq.~\eqref{eq:deterministic markov kernel}.
    \item A Markov kernel $\ta \in \cM (\cX \times \cY, \cY)$ %\colon \XY \tm \Sigma(\cY) \rightarrow [0,1]$
    is called \emph{bistochastic attribute corruption}, if it is a $\cG$-label corruption with $\cG$ being the set of all measurable, bijective mappings.
    \end{enumerate}
\end{definition}

\begin{remark}
    On finite spaces, bistochastic corruptions are represented by doubly stochastic matrices, as we know that the Birkhoff-von Neumann theorem holds \cite{birkhoff1946tres}. The same statement does not hold for general continuous spaces, which is a relevant observation to quantum mechanics. We do not investigate here possible generalizations of the theorem, as it is an open problem outside the scope of this paper, but we point the interested reader to some recent developments in this direction, \eg, \cite{safarov2005birkhoff,paunescu2017generalization}.
\end{remark}

The illustrations in Figure~\ref{fig:kernel-on-spr} show how finite bistochastic kernels acts on a point of the superprediction set, and how it differs from a general corruption kernel. It suggests that the GDPI may hold under assumptions on the superprediction set determined by the corruption kernel structure. 

\subsection{Convex Combination of Kernels and Superprediction Set}\label{Kernel Action on Superprediction Set}

Corollary~\ref{corollary:kernel-gdpi-characterization} tells us that, in order to show a GDPI for Bayes risk for a general $\phi \in \finProb(\XY)$, it is equivalent to understand the set relationship between the superprediction set of $\ell \circ \cH$ and the superprediction set of $\mto{\ka}(\ell \circ \cH)$. Understanding what is the effect of the kernel $\ka$ on $\ell \circ \cH$, however, is non-trivial. 

A first observation comes from the following: because of its normalization property ($\ka(\cA, z) =1 \ \forall z\in \cZ$, for some $\cA \in \cZ'$), and because a kernel $\ka$ acts on a function as,
\begin{align}
    \mto{\ka} (\ell(\mt{h}(x),y)) &\coloneqq \int_{\XY}  \ka(x,y, d( \tilde{x},  \tilde{y}))\, \ell( \mt{h}( \tilde{x}),  \tilde{y}) , \qquad \forall (x,y) \in \XY\,,
\end{align}
we know that the kernel convexly combines the values of a mapping $(x,y)\mapsto \ell(\mt{h}(x),y)$ for different $(x,y)$-pairs.
For illustration, fix $x \in \cX$ and suppose that $\cY = \{ 0,1\}$ is binary. Choose a loss $\ell\in\cL(\cY)$, a model $h \in \cH$ with $\mt{h}(x) = \pi \in \Prob(\cY)$ and a label corruption $\la \in \cM (\XY, \cY)$. 
We are interested in the point $q_1 = \big(\ell(\pi,0), \ell(\pi,1)\big)$ and its element-wise convex combinations $(\la \ell)_x(\pi, \tilde{y}) \in \co(\{ \ell(\pi,y) \}_{y\in \cY})$ for a fixed $x$. The effect of the kernel is the point $q_{\la} = \big(\la \ell(\pi,0), \la \ell(\pi,1)\big)$. Note that because $\la$ is label corruption it does not have an effect on the $x \in \cX$, but it only depends on it. An illustration of all the possible outcomes for all possible $\la$ kernels is given by Figure~\ref{fig:kernel-on-spr}(a), \ie, the blue area. We will use the intuition gained through this example to build conditions for the data processing inequality to hold.
%We now provide an argument that we can largely neglect the blue marked area in Figure~\ref{fig:kernel-on-spr} while only carrying about the extreme points marked as blue dots.
Additional insights on possible strategies to prove a data processing result can come from Proposition~\ref{prop:BRIkerns is convex}. 
The proposition shows also that the subset of kernels fulfilling the data processing inequality for a given choice of loss and model class is convex. 
In general, this gives us no knowledge of the extreme points of such a set knowing the extreme points of the whole set of Markov kernels. 
Nevertheless, if we assume we know which extreme points generate such a subset of kernels, we may get interesting results. To this end, we state a consequence of Proposition~\ref{prop:BRIkerns is convex} and Corollary~\ref{corollary:kernel-gdpi-characterization}.

\begin{proposition}[Generalized Data Processing Inequality with Convex Combination of Kernels]
\label{proposition:Sufficient Conditions for DPI with Two Attribute Corruption Kernels}
    Let $\ka_1, \ka_2 \in \cM(\XY, \XY)$. Let $\ell\in\cL(\cY)$ be a loss function and $\cH \subseteq \cM (\cX, \cY)$ a model class such that $\mto{\ka}_i (\ell \circ h) \in \cspr(\ell \circ \cH)$ for all $h \in \cH$ and $\ka_i, i \in \{ 1, 2\}$. Then, for all $\ka \coloneqq \alpha \ka_1 + (1-\alpha)\ka_2$ with $\alpha \in [0,1]$,
        $\br_{\ell \circ \cH}(\phi) \le \br_{\ell \circ \cH}(\amto{\ka} \phi) \,, \forall \phi \in \finProb(\XY) \,.$
\end{proposition}
% \begin{proof}
%     The statement follows immediately from Proposition~\ref{prop:BRIkerns is convex} and Corollary~\ref{corollary:kernel-gdpi-characterization}.
% \end{proof}
% We will now take a closer look at the label and attribute corruption, for which this result will have different consequences.

\section{Sufficient Conditions for the GDPI} \label{sec:sufficient conditions}

In the following, we leverage our characterization result (Corollary~\ref{corollary:kernel-gdpi-characterization}) to derive the rest of our core contributions: sufficient conditions for the GDPI for Bayes risk in the case of dependent corruptions, as well as a new characterization for certain types of simple corruptions. 

The conditions in the results apply to specific families of learning problems and corruptions, where learning problems are specified through assumptions on the loss function and the model class, but not on the underlying probability distribution. Such distributional generality is required to obtain a data processing result in the sense of Definition~\ref{def:gdpi for br}.
Our objective is to identify learning problems for which certain classes of Markov kernels are included in $\BRIkerns(\ell \circ \cH)$. Since $\BRIkerns(\ell \circ \cH)$ is convex, it suffices to study its extreme points, which considerably simplifies the analysis. 

In the binary setting depicted in Figure~\ref{fig:kernel-on-spr}, all relevant kernel actions can be expressed as convex combinations of deterministic kernels induced by constant functions (degenerate binary kernels) or certain bijective functions (symmetric bistochastic kernels). 
Applying Proposition~\ref{proposition:Sufficient Conditions for DPI with Two Attribute Corruption Kernels}, it is therefore sufficient to impose conditions on $\ell \circ \cH$ that guarantee a decrease in Bayes risk under both degenerate and bistochastic kernels: under such conditions, the generalized data processing inequality holds for all kernels. For larger domains, the class of relevant kernels is substantially richer. 
Nevertheless, we keep in mind these two fundamental cases as examples, while we establish sufficient conditions for the generalized DPI with families of kernels that are generated by certain deterministic kernels. 
As the extreme points of the set of Markov kernels are deterministic kernels (Proposition~\ref{prop:kernel extreme points}), these sets will share a portion of their boundary with the set of all Markov kernels with the same signature. 
We imagine that showing results for differently constructed sets of kernels may require more sophisticated arguments; investigating the GDPI for such cases is left as future work.

In the next section, we begin by showing that when the loss function and the model class satisfy a suitable closure property with respect to label transformations, then optimal performance as measured by the Bayes risk can only deteriorate under label corruptions induced by such transformations, and vice versa.
By strengthening this invariance assumption, we extend the result to label corruptions that depend on the attribute $x \in \cX$ as a sufficient condition. An important special case is that of bistochastic label corruption, for which we derive sufficient conditions on $\ell$ and $\cH$ ensuring that the GDPI for Bayes risk holds.
We then establish analogous results for attribute corruption and illustrate the corresponding statements using commonly encountered and practically relevant transformations. Finally, we present examples demonstrating that all conditions introduced in this section are sufficient but not necessary. %for the GDPI for Bayes risk to hold.

\subsection{Label Corruption}
%Label corruption is a Markov kernel which only distorts the data distribution on the labels.
We start from the simplest form of label corruption in our considered taxonomy \cite{iacovissi2026corruptions}: simple label corruption generated by deterministic kernels, \ie, the labels are ``smoothly'' and ``randomly'' permuted.
\begin{proposition}%[Sufficient Conditions for GDPI with Simple Label Corruption]
\label{proposition:Bayes Risk Inequality with Simple Label Corruption}
    Let $\cF$ be a class of measurable functions $f \colon \cY \rightarrow \cY$. Let us define the following set of simple label corruptions induced by $\cF$,
    \begin{align}
        \Lambda_\cF \coloneqq \co \big\{ \la_f \in \cM(\cY , \cY) \mid f \in \cF \big\},
    \end{align}
    where $\la_f$ is a deterministic kernel.
    Let $\ell\in\cL(\cY)$ be a bounded loss function and $\cH \subseteq \cM(\cX,\cY)$ a model class. The set containment condition 
    \begin{align}
    \label{eq:sufficient condition for label corruption - convex closure}
        \bigcup_{f \in \cF} \big\{ \mto{\la}_f\ell \circ h \mid  h\in \cH \big\} \subseteq \cspr(\ell \circ \cH),  
    \end{align}
    %the superprediction set containment holds. Hence, Corollary~\ref{corollary:kernel-gdpi-characterization} and 
    %Proposition~\ref{proposition:Sufficient Conditions for DPI with Two Attribute Corruption Kernels} imply that 
    holds if and only if the generalized data processing inequality $\br_{\ell \circ \cH}(\phi) \le \br_{\ell \circ \cH}(\amto{\la} \phi)$ holds for all $ \phi \in \finProb(\XY), \la \in \Lambda_\cF$.
\end{proposition}
\begin{proof}
    By definition of the superprediction set (Definition~\ref{def:Constrained Superprediction Set with Respect to model class}), $\cspr\, \cG \subseteq \cspr \,\cG'$ holds if and only if $\cG \subseteq \cspr \,\cG'$. Let $\cG = \co \left( \bigcup_{f \in \cF} \big\{ \mto{\la}_f\ell \circ h \mid  h\in \cH \big\} \right)$ and  $\cG' = \ell \circ \cH$. We know that $\cG = \bigcup_{\la \in \Lambda_\cF} \big\{ \mto{\la} \ell \circ h \mid  h\in \cH \big\} $, with only the loss being corrupted ($\mto{\la}(\ell \circ h) = (\mto{\la}\ell \circ h)$), as
    \begin{align}
        \left(\la \ell \circ h\right)(x,y) \coloneqq \left[\left(\sum_{i=1}^n \alpha_i \la_{f_i}\ell \right) \circ h \right](x,y) = \left[\sum_{i=1}^n \alpha_i (\la_{f_i}\ell \circ h)\right] (x,y).
    \end{align}  
    Now noticing that $\rho_{\cA \cup \cB}(\phi) = \min \{ \rho_{\cA}(\phi), \rho_{\cB}(\phi)\} \le  \rho_{*}(\phi)$ for $(*) \in \{\cA, \cB\}$ because of properties of the infimum, we can apply the support function characterization result from Theorem~\ref{theo:support-ineq-characterization-superpredictionset} and the equality between Bayes risk and support function from Proposition~\ref{prop:support function of superprediction set is bayes risk} to obtain the thesis.
\end{proof}

We now show that we can go beyond simple label corruption when strengthening the assumptions on the $\ell$ and $\cH$ combination. In this case the corruption $\lambda$ depends on $x$, \ie, conditioned on $x$, the label $y$ is randomly flipped.

\begin{proposition}[Generalized Data Processing Inequality with $\cF$-Label Corruption]
\label{prop:sufficient conditions dpi label corruption}
    Let $\cY$ be finite and $\cF$ a class of functions $\cY \rightarrow \cY$. Let $\ell\in\cL(\cY)$ be a loss function and $\cH \subseteq \cM(\cX,\cY)$ a model class. Consider the following conditions:
    \begin{enumerate}
        \item The loss function $\ell$ is \emph{$\cF$-symmetric}, \ie, $\ell(\pi,y) = \ell([\pi_{f(y)}]_y, f(y))$, for all $f \in \cF$, where $[\pi_{f(y)}]_y$ denotes the corresponding coordinate-wise rearrangement of the probability vector $\pi \in \finProb(\cY)$ induced by $f$;
        \item The model class is \emph{$\cF$-label-invariant}, \ie, for every $h \in \cH$, and every $f \in \cF$, there exists $h_f \in \cH$ such that $\mt{h}_f(x) = [\mt{h}(x)_{f(y)}]_y$, for all $x \in \cX$;
        \item The kernel is a $\cF$-label corruption $\la \in \cM(\XY, \cY)$.
    \end{enumerate}
    If all of these conditions are fulfilled, then $\br_{\ell \circ \cH}(\phi) \le \br_{\ell \circ \cH}(\amto{\la}\phi)$ for all $\phi \in \finProb(\XY)$.
\end{proposition}
\begin{proof}
    We show that the application of the kernel $\la$ on the attainable prediction losses $\ell \circ \cH$ leads to the subsethood condition $\cspr(\mto\la (\ell \circ \cH)) \subseteq \cspr(\ell \circ \cH) $. Then, we apply Corollary~\ref{corollary:kernel-gdpi-characterization}. We can write a generic element of $\mto{\la}\ell \circ \cH$ as
    \begin{align}
        \mto{\la}(\ell \circ h)(x,y) = (\mto{\la}\ell \circ h)(x,y) &= \int_{\cY} \ell(\mt{h}(x),  \tilde y) \, \la(x,y,d \tilde y)
        = \sum_{i \in I } \alpha^x_i \int_{\cY} \ell(\mt{h}(x),  \tilde y) \, \ka_{f^x_i}(y,d \tilde y)\\
        &= \sum_{i \in I } \alpha^x_i \ell(\mt{h}(x), f^x_i(y)) = \sum_{i \in I } \alpha^x_i \ell([\mt{h}(x)_{f^x_i(y)}]_y, y).
    \end{align}
    By $\cF$-label invariance, there exists $\mt{h}_{f^x_i}(x) = [\mt{h}(x)_{f^x_i(y)}]_y$ \st $h_{f^x_i} \in \cH$. Hence, we know that $(x,y) \mapsto \ell([\mt{h}(x)_{f^x_i(y)}]_y, y)$ is in $\ell \circ \cH$ by definition. It follows that
    %\begin{align}
        $\sum_{i \in I } \alpha^x_i \ell([\mt{h}(x)_{f^x_i(y)}]_y, y) \in \co(\ell \circ \cH)$, a subset of $\cspr(\ell \circ \cH).$
    %\end{align}
\end{proof}
% \begin{remark}
% \label{remark:comparing proof techniques - label corruption}
%     In Proposition~\ref{proposition:Bayes Risk Inequality with Simple Label Corruption} we make use of the left implication chain in Figure~\ref{fig:implications for set containment condition} while in the proof of Proposition~\ref{prop:sufficient conditions dpi label corruption} we make use of the right implication chain in Figure~\ref{fig:implications for set containment condition}. In particular, the assumptions $\cF$-symmetry and $\cF$-label-invariance, essentially guarantee that the strong sufficient condition $\mto\la (\ell \circ \cH) \subseteq \ell \circ \cH$ holds for $\la$ being a $\cF$-label corruption. This assumption is not enough to prove the necessity of the condition, differently from the approach used in Proposition~\ref{proposition:Bayes Risk Inequality with Simple Label Corruption}.
% \end{remark}

\begin{corollary}[of Proposition~\ref{prop:sufficient conditions dpi label corruption} for Bistochastic Label Corruption]
\label{corollary:sufficient conditions dpi label corruption bistochastic}
    Let $\cY$ be finite, $\ell\in\cL(\cY)$ be a loss function and $\cH \subseteq \cM(\cX,\cY)$ a model class. Consider the following conditions:
    \begin{enumerate}
        \item The loss function $\ell$ is \emph{symmetric}, \ie, $\ell(\pi,y) = \ell([\pi_{\sigma(y)}]_y, \sigma(y))$, for all bijective $\sigma \colon \cY \to \cY$, where $[\pi_{\sigma(y)}]_y$ denotes the corresponding coordinate-wise rearrangement of the probability vector $\pi \in \finProb(\cY)$ induced by $\sigma$;
        \item The model class is \emph{$\cY$-permutation invariant}, \ie, for every $h \in \cH$, and every bijective $\sigma \colon \cY \to \cY$, there exists $h_\sigma \in \cH$ such that $\mt{h}_\sigma(x) = [\mt{h}(x)_{\sigma(y)}]_y$, for all $x \in \cX$;
        \item The kernel is a bistochastic label corruption $\la \in \cM(\XY, \cY)$.
    \end{enumerate}
    If all of them are fulfilled, then $\br_{\ell \circ \cH}(\phi) \le \br_{\ell \circ \cH}(\amto{\la}\phi)$ holds for all $\phi \in \finProb(\XY)$.
\end{corollary}
% \begin{proof}
%     We apply Proposition~\ref{prop:sufficient conditions dpi label corruption} with its $\cF$ the set of all bijective functions from $\cY$ to $\cY$.
% \end{proof}

\begin{remark}
    The conditions in Corollary~\ref{corollary:sufficient conditions dpi label corruption bistochastic} agree with previous findings regarding robust losses. In particular, \citet[Theorem 1]{ghosh2017robust} prove that symmetric losses are robust to low-rate symmetric label corruption, as well other non-symmetric and attribute-dependent forms of labels corruption (Theorems 2 and 3). Both cases are subsumed by our result, which includes all noise rates, but in the form of an inequality. 
\end{remark}
\begin{example}
    Let $\cX = \reals^d$ for some $d \in \bb{N}$, $\cY = \{ 0,1\}$. The loss function $\ell(\pi,y) = (y-\pi)^2$ is symmetric. Any model class $\cH \subseteq \cM(\cX,\cY)$ such that for all $x \in \cX$ and $ h \in \cH$, there is $h' \in \cH$ such that $1-\mt{h}(x) = \mt{h}'(x)$, is $\cY$-permutation invariant. Hence, by Corollary~\ref{corollary:sufficient conditions dpi label corruption bistochastic}, label corruption does not degrade the performance of this statistical learning problem. 
\end{example}

\subsection{Attribute Corruption}

Analogously to Proposition~\ref{proposition:Bayes Risk Inequality with Simple Label Corruption}, we can provide a sufficient condition for the GDPI for Bayes risk under simple attribute corruption.

\begin{proposition}%[Sufficient Conditions for GDPI with Simple Attribute Corruption]
\label{proposition:Bayes Risk Inequality with Simple Attribute Corruption}
    Let $\cF$ be a class of measurable functions $f\colon \cX \rightarrow \cX$. Let us define the following set of simple attribute corruptions induced by $\cF$,
    \begin{align}
        T_\cF \coloneqq \co \left\{ \ta_f \in \cM(\cX , \cX) \mid f \in \cF \right\}.
    \end{align}
    where $\ta_f$ is a deterministic kernel.
    Let $\ell \in \cL(\cY)$ and $\cH \subseteq \cM(\cX,\cY)$. The set containment condition 
    \begin{align}
    \label{eq:sufficient condition for attribute corruption - convex closure}
        \bigcup_{f \in \cF} \big\{ \mto{\ta}_f(\ell \circ h) \mid  h\in \cH \big\} \subseteq \cspr(\ell \circ \cH),
    \end{align}
    holds if and only if the generalized data processing inequality
    $\br_{\ell \circ \cH}(\phi) \le \br_{\ell \circ \cH}(\amto{\ta} \phi)$ holds for all $\phi \in \finProb(\XY), \ta \in T_\cF$.
\end{proposition}
\begin{proof}
    Let $\cG = \co \left( \bigcup_{f \in \cF} \big\{ (x,y) \to \mto{\ta}_f(\ell \circ h)(x,y) = \ell ((\mt{h} \circ f) (x), y) \mid  h\in \cH \big\} \right)$. We know that $\cG = \bigcup_{\ta \in T_\cF } \big\{ \mto{\ta}( \ell \circ h )\mid  h\in \cH \big\} $, because
    \begin{align}
        \left[\ta (\ell \circ h) \right](x,y) \coloneqq \left[\left(\sum_{i=1}^n \alpha_i \ta_{f_i}\right)\left(\ell  \circ h \right)\right](x,y) = \left[\sum_{i=1}^n \alpha_i \ta_{f_i} (\ell \circ h)\right] (x,y),
    \end{align} 
    as we can exchange the finite sum with the action of a kernel on some functions, as it is defined as an integral of integrable functions. The rest of the proof is analogous to that of Proposition~\ref{proposition:Bayes Risk Inequality with Simple Label Corruption}. 
\end{proof}

\begin{remark}[Recovery of Classical Data Processing Inequality for Bayes Risk]
\label{remark:dpi finite space}
    The standard data processing result for unconstrained settings can be easily proved using a technique similar to the proposition above. 
    Let $\ell\in\cL(\cY)$ be a loss function,  $\cH = \cM(\cX, \cY)$ and $\ta \in \cM(\cX, \cX)$. Then, for all $h \in \cH$,
    \begin{align} \label{eq:kernel application in the remark for recovery of DPI}
        \mto{\ta}(\ell \circ h)(x,y) %&= \int \ell(\mt{h}( \tilde x), y)\ta(x, d \tilde x) \,, \quad \forall (x,y) \in \XY\\
        &\coloneqq \int_{ \tilde x \in \cX} \ell(\mt{h}( \tilde x), y) \, \ta(x, d \tilde x) \,, \quad \forall (x,y) \in \XY\,.
    \end{align}
    We know that Proposition~\ref{prop:kernel extreme points} holds. It follows that we can write $\ta$ as the convex combination of some deterministic Markov kernels. Hence, by definition of extreme points of a set there exists a finite set of measurable functions $\cF \subseteq \{ f\colon \cX \to \cX \}$ and $\alpha_f \in [0,1]$ with $\sum_{f \in \cG} \alpha_f=1$, such that $\ta = \sum_{f \in \cG} \alpha_f \ta_f$ where $\ta_f$ denotes the deterministic kernel corresponding to the function $f \in \cF$. Using the notation $\ta_f( \tilde x, x)$ for the matrix entry corresponding to row $ \tilde x$, column $x$, we can write
    \begin{align}
        \mto{\ta}(\ell \circ h)(x,y) &= \int_{ \tilde x \in \cX} \ell(\mt{h}( \tilde x), y) \sum_{f \in \cG} \alpha_f \ta_f(x, d  \tilde x) \,, \\
        &= \sum_{f \in \cG} \alpha_f \int_{ \tilde x \in \cX} \ta_f(x, d  \tilde x) \, \ell(\mt{h}( \tilde x), y) 
        %&= \sum_{g \in \cG} \alpha_g \ell(\mt{h}(g(x)), y) \,, \\
        = \sum_{f \in \cG} \alpha_f \, \ell(\mt{h}_f(x), y) \,, 
    \end{align}
    for all $(x,y) \in \XY$, and where $\mt{h}_f \coloneqq \mt{h} \circ f$ is a transition of some kernel $h_f$ in $\cM(\cX, \cY)$. By definition of all sets involved, we have that $\sum_{f \in \cG} \alpha_f \ell ( \mt{h}_f (x), y) \in \co (\ell \circ \cH) \subseteq \cspr(\ell \circ \cH)\,.$
   The GDPI for Bayes risk follows by Corollary~\ref{corollary:kernel-gdpi-characterization}. Since the statements are true for a general loss $\ell$ and joint probability $\phi$, we can pick $\phi = \pi_{\cY} \tm E$ and obtain that the inequality holds for all $\ell, \pi_{\cY}$. We therefore recover the classical DPI, as stated in Eq.~\eqref{eq:DPI-informal}. 
\end{remark}

As done for the label corruption case, we can go beyond the simple $\cF$-attribute corruption.

\begin{proposition}[Generalized Data Processing Inequality with $\cF$-Attribute Corruption]
\label{prop:sufficient conditions dpi attribute corruption}
    Let $\cF$ be a class of functions $\cX \rightarrow \cX$. Let $\ell \in \cL(\cY)$ be a loss function and $\cH \subseteq \cM(\cX, \cY)$ a model class. Consider the following conditions:
    \begin{enumerate}
        \item The model class is \emph{$\cF$-attribute-invariant}, \ie, for every $h \in \cH$, and every $f \in \cF$, there exists $\mt{h}_f \in \cH$ such that $\mt{h}_f(x) = \mt{h}(f(x))$, for all $x \in \cX$;
        \item The kernel is a $\cF$-attribute corruption $\ta \in \cM(\XY, \cX)$.
    \end{enumerate}
    If both these conditions are fulfilled, then $\br_{\ell \circ \cH}(\phi) \le \br_{\ell \circ \cH}(\amto{\ta}\phi)$ holds for all $\phi \in \finProb(\XY)$.
\end{proposition}
\begin{proof}
    We show that the application of the kernel $\ta$ on the set $\ell \circ \cH$ leads to the subsethood condition $\cspr(\mto\ta (\ell \circ \cH)) \subseteq \cspr(\ell \circ \cH) $. Then, we apply Corollary~\ref{corollary:kernel-gdpi-characterization}. 
    Observe that we can write
    \begin{align}
        \mto{\ta}(\ell \circ h)(x,y) &= \int_{\cX} \ell(\mt{h}( \tilde x), y) \ta(x,y,d( \tilde x))  \\
        &= \sum_{i \in I } \alpha^y_i \int_{\cX} \ell(\mt{h}( \tilde x), y) \, \ka_{f^y_i}(x,d( \tilde x)) = \sum_{i \in I } \alpha^y_i \ell(\mt{h}(f^y_i(x)), y)\,.
    \end{align}
    By $\cG$-attribute invariance, there exists $\mt{h}_{f^y_i}(x) = \mt{h}(f^y_i(x)) \in \cH$. It follows that
    %\begin{align}
        $\sum_{i \in I } \alpha^y_i \ell(\mt{h}(f^y_i(x)), y) \in \co(\ell \circ \cH)$, a subset of $\cspr(\ell \circ \cH).$
    %\end{align}
\end{proof}
%\begin{remark}
%\label{remark:comparing proof techniques - attribute corruption}
    %An analogous statement to Remark~\ref{remark:comparing proof techniques - label corruption} holds.
%\end{remark}

\begin{corollary}[of Proposition~\ref{prop:sufficient conditions dpi attribute corruption} for Bistochastic Attribute Corruption]
\label{corollary:sufficient conditions dpi attribute corruption bistochastic}
    Let $\ell\in\cL(\cY)$ be a loss function and $\cH \subseteq \cM(\cX, \cY)$ a model class. Consider the following conditions:
    \begin{enumerate}
        \item The model class is \emph{$\cX$-permutation invariant}, \ie, given $h \in \cH$, for every $\sigma \colon \cX \rightarrow \cX$ bijective and measurable, there exists $h_\sigma$ such that $\mt{h}(\sigma(x)) = h_\sigma(x)$;
        \item The kernel is a bistochastic attribute corruption $\ta \in \cM(\XY, \cX)$. 
    \end{enumerate} 
    If both these conditions are fulfilled, then $\br_{\ell \circ \cH}(\phi) \le \br_{\ell \circ \cH}(\amto{\ta}\phi)$ holds for all $\phi \in \finProb(\XY)$.
\end{corollary}

\begin{runex}[Affine Corruptions and Logistic Regression]
\label{sec:wu-gdpi}
Proposition~\ref{prop:sufficient conditions dpi attribute corruption} can be nicely illustrated in the case of logistic regression. The data processing inequality holds for corruptions built from \emph{affine} maps of attributes.

Let $\ell \in L(\cY)$ be \emph{any} loss and $\cF \subseteq \Aff(\reals^2)$.
Then every $\cF$-attribute corruption $\ta$ in the sense of
Definition~\ref{def:structured attribute corruption}, in particular every $\ta \in T_\cF$, satisfies
\begin{align}
    \hat\ka_\ta(\ell\circ\cH) \subseteq \operatorname{co}(\ell\circ\cH)
    \subseteq \csprH,
    \qquad\text{hence}\qquad \ka_\ta \in \cK(\ell\circ\cH) ,
\end{align}
\ie, $\BR_{\ell\circ\cH}(\phi) \le \BR_{\ell\circ\cH}(\check\ka_\ta\phi)$
for all $\phi \in \Delta(\cX\times\cY)$. %\footnote{To be precise, we need to truncate the cross-entropy loss $\ell_{\log}$ to a finite interval to make Corollary~\ref{corollary:kernel-gdpi-characterization} applicable. See Remark~\ref{rm:wu-boundedness} in Appendix~\ref{app:wu-dpi-breaks}.} 
For detailed computations see Appendix~\ref{app:wu-gdpi}.
\end{runex}

\subsection{Sufficient but not Necessary Conditions}

In order to obtain some BSS-like theorems for our GDPI, we would need the results to hold as characterizations. 
Indeed, the set containment condition in Eq.~\eqref{eq:sufficient condition for label corruption - convex closure} (respectively Eq.~\eqref{eq:sufficient condition for attribute corruption - convex closure}) is necessary and sufficient for all simple label corruption in $\la \in \Lambda_\cF$ (respectively, all simple attribute corruptions in $\ta \in T_\cF$) given some function set $\cF$. 
Corollary~\ref{corollary:kernel-gdpi-characterization} used in the proof is fulfilled only when we impose the set containment condition for the closed and convexified superprediction set. 
Instead, using the simple superprediction set as done in Proposition~\ref{prop:sufficient conditions dpi label corruption},~\ref{prop:sufficient conditions dpi attribute corruption} leads to the results being only sufficient for the GDPI to hold. We can construct a simple example to better understand this phenomenon.
\begin{example}
    Let $\cX = \{ x\}$ and $\cY = \{ 0,1\}$. A predictor is a distribution on $\cY$ which is defined through the probability on $y = 1$, for which we simply write $p \in [0,1]$. Let $\cH = [0, \gamma]$ for some $\gamma \in (0,0.5)$. Let $\ell$ be a cost-weighted loss function, \ie, $\ell(p,y) = \inset_{p < \gamma , y = 1}(p,y) + \inset_{p > \gamma , y = 0}(p,y) + \gamma \inset_{p = \gamma}(p) $. If $\gamma = 0.5$, and neglecting the last tie-breaking summand, then $\ell$ is the $0$-$1$-loss. The restrictions put to $\ell$ and $\cH$ lead to,
    \begin{align}
        \ell \circ \cH = \left\{\, \big((x,0) \mapsto 0, (x,1) \mapsto 1\big), \big((x,0) \mapsto \gamma, (x,1) \mapsto \gamma\big)\,\right\}\,.
    \end{align}
    Let $\la$ be uniform bistochastic simple label corruption. Hence, it is the convex combination of bistochastic simple label noises which are defined through the two permutations on $\cY$ by Birkhoff-von Neumann theorem \cite{birkhoff1946tres}.
    It is now a matter of simple computations to show that the GDPI for Bayes risk is preserved for the uniform bistochastic simple label corruption, but not for both simple permutation label noises. 
\end{example}

As we mentioned earlier, Proposition~\ref{prop:sufficient conditions dpi label corruption} (respectively Proposition~\ref{prop:sufficient conditions dpi attribute corruption}) are both built upon verifying the condition $\ka(\ell \circ \cH) \subseteq \ell \circ \cH$. The sufficiency of this condition can be illustrated in multiple ways. 
The simple example in the following constructs a model class that violates $\cX$-permutation invariance without invalidating the GDPI for Bayes risk. %An analogous construction applies to bistochastic label corruption.

% \begin{example}
%     Let $\cX = \{ a,b\}$ and $\cY = \{ 0,1\}$. We consider binary probabilistic predictors which are defined through points in $\Prob(\cY)^{\cX}$. The model class $\cH_{(0,1]} \coloneqq \{ a \mapsto (p,1-p), b\mapsto (1-p, p) \colon p \in (0,1]\} \subseteq \finProb(\cY)^{\cX}$ is not $\cX$-permutation invariant. However, for any bistochastic attribute corruption $\ta \in \cM(\cX, \cX)$,
%     \begin{align}
%         \br_{\ell \circ \cH_{(0,1]}}(\phi) \le \br_{\ell \circ \cH_{(0,1]}}(\amto{\ta}\phi),
%     \end{align}
%     for all $\phi \in \finProb(\XY)$ and all $\ell$. The reasoning is simple,
%     \begin{align}
%         \br_{\ell \circ \cH_{(0,1]}}(\amto{\ta}\phi) &= \br_{\mto{\ta}( \ell \circ \cH_{(0,1]})}(\phi)
%         = \inf_{f \in \mto{\ta}( \ell \circ \cH_{(0,1]} )} \langle f,\phi \rangle\\
%         &= \inf_{f \in  \ell \circ \cH_{[0,1]} } \langle f,\phi \rangle 
%         \ge \inf_{f \in  \ell \circ \cH_{(0,1]} } \langle f,\phi \rangle = \br_{\ell \circ \cH_{(0,1]}}(\phi) \,.
%     \end{align}
% \end{example}
% Even without invoking closure conditions, which are irrelevant for infimum operators appearing in the Bayes risk, we can show that $\cX$-permutation invariance is not a necessary condition for the GDPI for Bayes risk under attribute corruption. %The counterexample exploits the coarse grained nature of the loss function $\ell$.
\begin{example}
    Let $\cX = \{ a,b\}$ and $\cY = \{ 0,1\}$. Let $\ell \colon \Prob(\cY) \times \cY \to \reals_{\ge 0}$ be the $0$-$1$-loss function.
    We consider binary probabilistic predictors which are uniquely identified through points in $[0,1]^{|\cX|}$ because of the finiteness of all sets involved. The model class 
    \begin{align}
        \cM(\cX, \cY) \supset \cH \coloneqq \{\,( &a \mapsto (0,1), b\mapsto (0.9,0.1)), 
        ( a \mapsto (1,0), b\mapsto (0.25, 0.75)), \\
        ( &a \mapsto (0.2,0.8), b\mapsto (0.15, 0.85)), 
        ( a \mapsto (0.6,0.4), b\mapsto (0.7, 0.3)) \,\} ,
    \end{align} 
    here expressed in its Markov transition form for convenience, is not $\cX$-permutation invariant. However, for any (including bistochastic) attribute corruption $\ta \in \cM(\cX, \cX)$,
    %\begin{align}
        $\br_{\ell \circ \cH}(\phi) \le \br_{\ell \circ \cH}(\amto{\ta}\phi),$
    %\end{align}
    for all $\phi \in \finProb(\XY)$, as $\mto{\ta}( \ell \circ \cH ) =  \ell \circ \cH $. This is true because the loss function only distinguishes between predictors using the threshold value $0.5$. In other words, the class $\cH$ is, through the lens of $\ell$, as powerful as the unconstrained model class $\cM(\cX, \cY)$.
\end{example}

\section{Bayes Risk Ordering for Comparing \texorpdfstring{$\ell \circ \cH$}{(l,H)}'s}
\label{sec:Bayes risk orderings of ell-H combinations}

In fact, we can provide another interpretation of Theorem~\ref{theo:support-ineq-characterization-superpredictionset} by stating a corollary focusing on the comparison of Bayes risks computed on differing loss and model class combinations.
%In fact, thanks to Lemma~\ref{lemma:kernel-is-adjoint-to-risk}, we know that this is equivalent to comparing $\ell\circ\cH$ and $\mto\ka(\ell\circ\cH)$ through the Bayes risk ordering.
% Our second main contribution re

%\Laura{I think this result detaches from the original motivation. I wouldn't call this the main result. The one above with the support functions is the core result (it is in fact a theo); this below goes in the direction of ``comparison of different loss/architecture choices''; we although stated in the intro that we care about comparison of the before and after Bayes optimal performance of one fixed choice of loss/architecture. Therefore this result below is confusing to the reader, esp if called the main one. We should specify that it is a consequence of our core result Theo 5.1, one that we do not explore much (we do not because the corrupted spr does not correspond necessarily to a loss and model class), but that we believe is worth mentioning as an additional application of our Theo 5.1. However, our main result to the end of a GDPI is in fact Corollary 5.3}
\begin{corollary}
    \label{corollary:no-kernel-gdpi-characterization}
    Let $\ell, \ell' \in \cL(\cY)$ be loss functions and $\cH, \cH' \subseteq \cM(\cX,\cY)$ model classes. Then,
    \begin{align}
        \csprH \supseteq \cspr(\ell' \circ \cH')
        \quad \Leftrightarrow \quad 
        \br_{\ell \circ \cH}(\phi) \le \br_{\ell' \circ \cH'}(\phi) \quad \forall \phi \in \finProb(\XY).
    \end{align}
\end{corollary}
\begin{proof}
    The statement follows directly from Proposition~\ref{prop:support function of superprediction set is bayes risk} and Theorem~\ref{theo:support-ineq-characterization-superpredictionset}.
\end{proof}

We can rephrase the above statement as:
\emph{An ordering for attainable prediction losses, $\ell \circ \cH \brprec \ell' \circ \cH'$, induced by the data processing inequality, is equivalent to the inverse inclusion order on the convex hull of the associated superprediction sets, $\csprH \supseteq \cspr( \ell' \circ \cH')$.} 

Observe that this property is similar to the equivalence of the Bayesian-better-than and informativeness orderings, as illustrated in \citet{khan2024comparison}. However, our ``informativeness'' is not Blackwell's analogous for the $\ell\circ\cH$ comparison, as it does not make use of the attainable average losses. For a fixed loss function $\ell$, Corollary~\ref{corollary:no-kernel-gdpi-characterization} implies that two model classes $\cH$ and $\cH'$ are equally powerful, \ie, achieve same Bayes risk on every distribution, if and only if the superprediction sets of $\ell \circ \cH$ and $\ell \circ \cH'$ are equal.

\begin{remark}[Regularization]
\label{remark:subsethood of H}
    The subsethood condition characterizing the GDPI for Bayes risk above can be achieved by $\cH' \subset \cH$ and $\ell = \ell '$. In other words, the Bayes risk of a smaller model class evaluated with the same loss cannot decrease, and therefore we observe a best-case performance degradation. This means that the benefits (for finite samples) of regularizing the model class cannot be seen at the Bayes risk level.  
\end{remark}

\begin{runex}[Logistic Regression with Cross-Entropy Loss and Zero-One-Loss]
\label{sec:wu-comparing two learning problems}
Corollary~\ref{corollary:no-kernel-gdpi-characterization} makes it possible to compare optimal performance across distributions for differing choices of model classes and losses. 
For instance, we can show that logistic regression with tie-broken zero-one loss undercuts the optimal performance of logistic regression with cross-entropy loss on \emph{every} distribution.
More formally, the tie-broken zero-one loss is defined as,
\begin{align}
    \ell_{01}(\pi, y) \coloneqq \llbracket \pi_y < \tfrac12\rrbracket + \tfrac12 \llbracket \pi_y = \tfrac12\rrbracket,
\end{align}
so that $\ell_{01}(\dot h_\alpha(x), y)$ is $0$, $\tfrac12$ or $1$ according to whether $h_\alpha$ classifies $(x,y)$ correctly, is undecided, or errs. In fact, $\cspr(\ell_{01}\circ\cH) \supseteq \cspr(\ell_{\log}\circ\cH)$ and
consequently, %\footnote{To be precise, we need to truncate the cross-entropy loss $\ell_{\log}$ to a finite interval to make Corollary~\ref{corollary:no-kernel-gdpi-characterization} applicable. See Remark~\ref{rm:wu-boundedness} in Appendix~\ref{app:wu-dpi-breaks}.},
\begin{align}
    \BR_{\ell_{01}\circ\cH}(\phi) \le \BR_{\ell_{\log}\circ\cH}(\phi)
    \qquad \text{for all } \phi \in \Delta(\cX\times\cY) .
\end{align}
The inequality is strict at $\phi_{(1,0,1,0)}$: $\tfrac14 < \log 2$. The proof is detailed in Appendix~\ref{app:wu-comparing two learning problems}.
\end{runex}

\section{Conclusions and Explanations}

In the introduction (Section~\ref{sec:dpi tailored to ml}) we asked how to align the existing theory on the effects of data processing and the new modes of statistics induced by machine learning practices. 
In particular, we asked how can the DPI result be compatible with the empirically observed high performance of simple task-adapted predictive models (``heads'' or ``probes'', depending on context) when they are given in input some general purpose representations, provided by pre-trained large models, instead of the raw data.
We demonstrated that this alleged contradiction can be readily fixed by considering a \emph{constrained} model class and defining a \emph{generalized data processing inequality}.
It turns out that this generalized data processing inequality can even be quantified: in Appendix~\ref{app: antipolar inequalities}, we give a multiplicative factor which exactly describes the gap of the inequality. Hence, we obtain \emph{strong} versions of the generalized data processing inequality (Proposition~\ref{cor: antipolar support function sufficient condition for sgdpi}). Because the derivations require the introduction of more technical tools such as antipolar dualities, we leave the exposition to the appendix.
Going back to the content of the main text, we highlight that our formal results come with a series of practically relevant conclusions and explanations.
%Noise can improve Performance

% In Section~\ref{sec:dpi tailored to ml} we asked how to align the existing theory on the effects of data processing and the new modes of statistics induced by machine learning practices. 
% In particular, we asked how can the DPI result be compatible with the empirically observed high performance of simple task-adapted predictive models (``heads'' or ``probes'', depending on context) when they are given in input some general purpose representations, provided by pre-trained large models, instead of the raw data.
% The intuition provided by classical DPI is that learning from data which are processed via some $\ka\in\cM(\cX,\cX)$, and therefore in this case transformed into some general purpose representation, should at best yield the same performance of the predictive model alone, assuming they both achieved the Bayes risk optimal.
% This alleged contradiction flagged that the existing theoretical understanding of data processing and its consequence were not aligned with empirical findings. 
% We suspected this could have been due to limitations of the theory, and in particular of the data processing inequality result, in how it abstracted and compared learning problems. 
% We therefore introduced a generalized theory of data processing, that reconciles empirical findings and intuitive understanding with formal analysis.

\begin{table}[t]
    \centering
    \caption{Statements equivalent to the Generalized Data Processing Inequality.}
    \setcellgapes{.5pt}
\makegapedcells
\begin{tabular}{ccc}
\toprule
Object                    & Formula                                                                                    & Reference \\
\midrule
Superprediction Set                        & $\cspr (\ell'\circ\cH') \subseteq \cspr (\ell\circ\cH)$                                                                        & - \\
Support Function           & \thead{ $\rho_{\spr(\ell'\circ\cH')}(\mu) \ge \rho_{\spr(\ell\circ\cH)}(\mu)$ \\[3pt] $\forall \mu \in \ba(\XY)$}           & Theorem~\ref{theo:support-ineq-characterization-superpredictionset} \\
Bayes Risk (change of set)                 & \thead{ $\br_{\ell'\circ\cH'}(\phi) \ge \br_{\ell\circ\cH}(\phi)$ \\[3pt] $\forall \phi \in \finProb(\XY)$}     & Corollary~\ref{corollary:no-kernel-gdpi-characterization} \\
\thead{\normalsize Bayes Risk (probability corruption) \\[3pt]  when $\ell'\circ\cH'= \mto\ka(\ell\circ\cH)$}            & \thead{ $\br_{\ell\circ\cH} (\amto{\ka} \phi) \ge \br_{\ell\circ\cH}(\phi)$ \\[3pt] $\forall \phi \in \finProb(\XY)$}     & Corollary~\ref{corollary:kernel-gdpi-characterization} \\
%Antipolar Set              & $\cG^\diamond \supseteq \cF^\diamond$                                                              & \cref{lemma:antipolar-duality-support-sets} \\
%Antipolar Support Function & \thead{$\rho_{\cG}^\diamond(f) \le \rho_{\cF}^\diamond(f)$ \\[3pt] $\forall f \in \Bb(\XY)$}       & \cref{prop:antipolar-dpi-characterization} \\
%Bayes Risk                 & \thead{$\br_{\cG} ( \phi) = \br_{\cF} (\amto{\ka} \phi) \ge \alpha_{\cF}(\ka) \; \br_{\cF} (\phi)$ \\[3pt] $\ka \in \cM(\XY,\XY)\,, \enspace \forall \phi \in \finProb(\XY)$}                                                                                      & \cref{prop:sgdpi-characterization}\\
%Inflation Coefficient    & $\alpha_{\cF}(\ka) \ge 1$                                                                         & \cref{prop:sgdpi-characterization}\\
%Minimal Relative Risk & \thead{$\reg_{\cF}(f) \ge 1$ \\[3pt] $ \forall f \in \cG$ }    & \cref{prop:sgdpi-characterization}\\

\botrule
\end{tabular}
    \label{tab:dpi-statements}
\end{table}

\begin{unlist}
    % \item \textit{Comparing learning problems before and after corruption.} 
    % We provided a first answer to the question ``what makes the classical DPI not appropriate for understanding machine learning practices?'' by showing counterexamples to the classical DPI when a constrained model class is involved (Section~\ref{sec:counterexample}). 
    % Observing the role of the constrained model class informed our search for a more suitable notion of ``data processing'' in machine learning, which lead to our definition of the generalized data processing inequality (GDPI, Definition \ref{def:gdpi for br}). 
    % We then proceeded asking for which models classes, losses and corruptions the DPI \emph{does} hold (Section~\ref{sec:sufficient conditions}), getting to a first set of sufficient conditions the inequality. 
    % %By proving these sufficient conditions for the DPI to hold, we identified a subset of kernels in $\BRIkerns(\ell \circ \cH)$ when $\cH$ and $\ell$ respect certain properties. \rab{Last two sentences could be remvoed, I think. It's repetitive.}

    \item \textit{Using GDPI as a sanity check for the effective model class.}
    Beyond their theoretical value as extensions of the classical data processing theory, the sufficient conditions' results in Section~\ref{sec:sufficient conditions} also identify properties that can be used as a sanity check for the chosen loss and architecture.
    In the case that the performance achieved on the raw/clean data is worse that the one achieved on transformed/corrupted data, one can claim that the empirical GPDI does not hold. If furthermore, the sufficient conditions are satisfied, it seems plausible that the only possible explanation is that the performances are not Bayes optimal.
    However, one can argue that the performance is not Bayes optimal because we are considering \emph{the wrong model class}: there exists some effective model class $\cH_{\text{effective}}$ which contains only the models that one \emph{can} learn given practical constraints, \eg, loss landscape, model architecture, learning algorithm, what sample is used, etc.
    Under this new abstraction of the problem, we can more safely assume that the Bayes optimal performance is indeed achieved in both clean and corrupted settings, but with respect to the effective model class.
    Hence, if we pick a set of functions $\cF$, %and construct our $\ell\circ\cH$ to be symmetric \wrt $\cF$, we learn that the applied corruption is not generated by $\cF$ in the sense of Definitions~\ref{def:structured label corruption} and \ref{def:structured attribute corruption}. 
    and construct $\ka$ using $\cF$ (in the sense of Definitions \ref{def:structured label corruption} and \ref{def:structured attribute corruption}), we learn that $\ell\circ\cH_{\text{effective}}$ must not respect the symmetry conditions.
    In summary, this can be used as a sanity check for the model class: if $\cH$ has been specified to be symmetric \wrt $\cF$ but is not respecting the GDPI for the appropriate $\ka$, then the $\cH_{\text{effective}}$ must have lost the symmetry property. 
    Picking appropriate families of $\cF$ can ensure that the desired representative power of $\cH$ is preserved by the implementation leading to $\cH_{\text{effective}}$.
    
    \item \textit{An information-theoretic explanation for the success of representation learning.} Going back to the use of pre-trained large models as general purpose representations, these can be represented as some $\ka\in\cM(\cX,\cX)$ corruption. 
    If this corruption is deterministic, \ie, it can be also represented as a function $f_0$, and the GDPI does not hold for a some prediction model optimized once with and once without pre-training, then we can attribute the success of the representation $f_0$ in improving performance to the effective model class associated to the prediction model (probe or head) not being $f_0$-symmetric. 
    This fact sounds very intuitive: the success of a representation lies in how it modifies the data space, and on whether the prediction model used was already capable of encoding that data modification or not. We remark that it was not necessary to require $f_0$ to be the function allowing the ``true Bayes optimum'' to be achieved. To the best of our knowledge, existing information theory frameworks were not capable of incorporating this rather intuitive fact. Finding additional sufficient conditions for the GDPI, or even characterizations, will bring more clarity to the mechanisms behind the proved usefulness of data processing in machine learning.
    
    \item \textit{When are two hypothesis classes equivalently performing?}
    In Section~\ref{sec:Bayes risk orderings of ell-H combinations}, we re-use our central technical contribution (Theorem~\ref{theo:support-ineq-characterization-superpredictionset}) to systematically compare the optimal performance among all data distributions for different choices of loss and model classes. First, one can read this section as a machine-learning-complement to the classical literature on comparing experiments and joint distributions, \ie, comparing loss and model classes. Second, the results answer an extremely simple, but to our knowledge unanswered question: under which conditions do two hypothesis classes, given a fixed loss function, equally perform (in the optimum) on \emph{all} data distributions? The answer is: when the convex and closed superpredictions sets are equal.
    % After proving the validity results for the GDPI included in Section~\ref{sec:sufficient conditions}, we used the constrained superprediction set characterization of the GDPI in its general form (Theorem~\ref{theo:support-ineq-characterization-superpredictionset}) for another type of Bayes risk inequality in the constrained setting. 
    % This framework, described in Section~\ref{sec:Bayes risk orderings of ell-H combinations}, provides a systematic method to compare different choices of loss and model class through using their average optimal performance for data distributions and considering certain corruptions, \ie, Markov kernels acting on the distribution. This complements existing literature on comparing experiments or joint probabilities.
\end{unlist}

\section*{Acknowledgments}
% The authors thank the anonymous reviewers for their valuable suggestions. This work is supported in part by funds from the National Science Foundation (NSF: \# 1636933 and \# 1920920).
We are very grateful for the feedback given by Georg Schollmayer!
The corresponding author is Laura Iacovissi; Laura Iacovissi and Rabanus Derr share first authorship. 
The authors acknowledge the use of the following generative AI tools: 
the running example and the graphics for the running example were crafted using Claude Opus 5.5 (September 2026); minor text editing, \eg, spelling checks and the examples' and figures' names formatting, was aided by Claude Sonnet 5 (September 2026) and Overleaf's integrated AI assistant (September 2026).
% \section*{Funding}
This work was partially funded by the Deutsche Forschungsgemeinschaft (DFG, German Research Foundation) under Germany’s Excellence Strategy — EXC number 2064/1 — Project number 390727645, as well as by the German Federal Ministry of Education and Research (BMBF): T\"{u}bingen AI Center, FKZ: 01IS18039A.
The authors thank the International Max Planck Research School for Intelligent Systems (IMPRS-IS) for supporting Laura Iacovissi and Rabanus Derr. 

\bibliographystyle{abbrvnat}
\bibliography{biblio}

@article{ahlswehde1976spreading,
 author = {Rudolf Ahlswede and Péter Gács},
 journal = {The Annals of Probability},
 number = {6},
 pages = {925--939},
 publisher = {Institute of Mathematical Statistics},
 title = {Spreading of Sets in Product Spaces and Hypercontraction of the Markov Operator},
 volume = {4},
 year = {1976}
}

@article{ali1966general,
author = {Ali, S. M. and Silvey, S. D.},
title = {{A General Class of Coefficients of Divergence of One Distribution from Another}},
journal = {Journal of the Royal Statistical Society: Series B (Methodological)},
volume = {28},
number = {1},
pages = {131--142},
year = {1966}
}

@book{aliprantis2006infinite,
  title={Infinite Dimensional Analysis: a Hitchhiker’s Guide},
  author={Aliprantis, Charalambos D. and Border, Kim C.},
  publisher={Springer},
  edition={3rd},
  year={2006}
}

@article{breiman2001statistical,
author = {Leo Breiman},
title = {{Statistical Modeling: The Two Cultures (with comments and a rejoinder by the author)}},
volume = {16},
journal = {Statistical Science},
number = {3},
publisher = {Institute of Mathematical Statistics},
pages = {199 -- 231},
year = {2001}
}

@article{birkhoff1946tres,
  title={Tres observaciones sobre el algebra lineal},
  author={Birkhoff, Garrett},
  journal={Univ. Nac. Tucuman, Ser. A},
  volume={5},
  pages={147--154},
  year={1946}
}

@article{Bohnenblust1949reconnaissance,
  title={Reconnaissance in {Game Theory}},
  author={H. Frederic Bohnenblust and Lloyd S. Shapley and Seymour Sherman},
  year={1949},
  journal={Santa Monica, CA: RAND Corporation},
  volumne={208},
  pages={1--18}
}

@article{blackwell1951comparison,
  title={Comparison of experiments},
  author={Blackwell, David},
  journal={Second Berkeley Symposium on Mathematical Statistics and Probability},
  pages={93--102},
  year={1951},
  publisher={University of California Press}
}

@article{blackwell1953equivalent,
 author = {Blackwell, David},
 journal = {The Annals of Mathematical Statistics},
 number = {2},
 pages = {265--272},
 publisher = {Institute of Mathematical Statistics},
 title = {Equivalent Comparisons of Experiments},
 volume = {24},
 year = {1953}
}

@article{bui2023interchange,
  title={Interchange rules for integral functions},
  author={B{\`u}i, Minh N. and Combettes, Patrick L.},
  journal={arXiv preprint arXiv:2305.04872},
  year={2024}
}

@article{chancelier2022minimization,
title = {Minimization interchange theorem on posets},
journal = {Journal of Mathematical Analysis and Applications},
volume = {509},
number = {1},
pages = {125927},
year = {2022},
author = {Jean-Philippe Chancelier and Michel {De Lara} and Benoît Tran},
}

@article{cabrerapacheco2024axiomatic,
    title={Aggregating Algorithm and Axiomatic Loss Aggregation},
    author={Cabrera Pacheco,Armando J. and Derr, Rabanus and Williamson, Robert C.},
    journal={Transactions on Machine Learning Research},
    year={2025},
}

@article{cabrerapacheco2023geometry,
    title={{The Geometry of Mixability}},
    author={Cabrera Pacheco, Armando J. and Williamson, Robert C.},
    journal={Transactions on Machine Learning Research},
    year={2023},
}

@phdthesis{cranko2021analytic,
  title={{An analytic approach to the structure and composition of General Learning Problems}},
  author={Cranko, Zac},
  year={2021},
  school={The Australian National University}
}

@article{chancelier2024unified,
  title={A Unified View of Polarity for Functions},
  author={Chancelier, Jean-Philippe and De Lara, Michel},
  journal={Journal of Convex Analysis},
  volume={32},
  number={1},
  pages={227--262},
  year={2025}
}

@book{cinlar2011ProbabilityAS,
  title={Probability and Stochastics},
  author={Erhan \c{C}inlar},
  year={2011},
  publisher={Springer}
}

@inproceedings{conklin2026learning,
title={Learning is Forgetting; {LLM} Training As Lossy Compression},
author={Henry Conklin and Tom Hosking and Tan Yi-Chern and Jonathan D. Cohen and Sarah-Jane Leslie and Thomas L. Griffiths and Max Bartolo and Seraphina Goldfarb-Tarrant},
booktitle={The Fourteenth International Conference on Learning Representations},
year={2026},
url={https://openreview.net/forum?id=tvDlQj0GZB}
}

@article{degroot1962uncertainty,
  title={Uncertainty, information, and sequential experiments},
  author={DeGroot, Morris H.},
  journal={The Annals of Mathematical Statistics},
  volume={33},
  number={2},
  pages={404--419},
  year={1962},
  publisher={Institute of Mathematical Statistics}
}

@article{duanmu2018finitely,
  title={Finitely-additive, countably-additive and internal probability measures},
  author={Duanmu, Haosui and Weiss, William},
  journal={Comment. Math. Univ. Carolin},
  volume={59},
  number={4},
  pages={467--485},
  year={2018}
}

@inproceedings{ethayarajh2022understanding,
  author={Kawin Ethayarajh and Yejin Choi and Swabha Swayamdipta},
  title={{Understanding Dataset Difficulty with V--Usable Information}},
  year={2022},
  booktitle={International conference on machine learning (ICML)}
}

@inproceedings{fel2024understanding,
 author = {Fel, Thomas and B\'{e}thune, Louis and Lampinen, Andrew Kyle and Serre, Thomas and Hermann, Katherine},
 booktitle = {Advances in neural information processing systems (NeurIPS)},
 pages = {69888--69924},
 title = {Understanding Visual Feature Reliance through the Lens of Complexity},
 year = {2024}
}

@article{finzi2026entropy,
  title={From Entropy to Epiplexity: Rethinking Information for Computationally Bounded Intelligence},
  author={Finzi, Marc and Qiu, Shikai and Jiang, Yiding and Izmailov, Pavel and Kolter, J. Zico and Wilson, Andrew Gordon},
  journal={arXiv preprint arXiv:2601.03220},
  year={2026}
}

@article{feldman1972some,
    author = {D. Feldman},
    title = {{Some Properties of Bayesian Orderings of Experiments}},
    volume = {43},
    journal = {The Annals of Mathematical Statistics},
    number = {5},
    pages = {1428 -- 1440},
    year = {1972},
}

@article{giner2009set,
	author = {Giner, Emmanuel},
	journal = {Set-Valued and Variational Analysis},
	number = {4},
	pages = {321--357},
	title = {Necessary and Sufficient Conditions for the Interchange Between Infimum and the Symbol of Integration},
	volume = {17},
	year = {2009}}

@article{gonzalezhernandex2005extreme,
    author = {Gonz\'{a}lez Hern\'{a}ndez, Juan and Hern\'{a}ndez Lerma, On\'{e}simo},
    title = {Extreme Points of Sets of Randomized Strategies in Constrained Optimization and Control Problems},
    journal = {SIAM Journal on Optimization},
    volume = {15},
    number = {4},
    pages = {1085-1104},
    year = {2005},
}

@article{goel1979comparison,
    author = {Goel, Prem K. and DeGroot, Morris H.},
    title = {{Comparison of Experiments and Information Measures}},
    volume = {7},
    journal = {The Annals of Statistics},
    number = {5},
    publisher = {Institute of Mathematical Statistics},
    pages = {1066 -- 1077},
    year = {1979},
}

@inproceedings{ghosh2017robust,
  title={Robust loss functions under label noise for deep neural networks},
  author={Ghosh, Aritra and Kumar, Himanshu and Sastry, P. Shanti},
  booktitle={AAAI Conference on Artificial Intelligence},
  year={2017}
}

@inproceedings{garcia2012divergences,
  title={Divergences and risks for multiclass experiments},
  author={Garc\'{i}a Garc\'{i}a, Dario and Williamson, Robert C.},
  booktitle={Conference on Computational Learning Theory (COLT)},
  pages={28.1--28.20},
  year={2012},
  organization={JMLR Workshop and Conference Proceedings}
}

@article{grunwald2004game,
  title={Game theory, maximum entropy, minimum discrepancy and robust {B}ayesian decision theory},
  author={Gr{\"u}nwald, Peter D. and Dawid, A. Philip},
  journal={Annals of Statistics},
  volume={32},
  number={4},
  pages={1367--1433},
  year={2004},
  publisher={Institute of Mathematical Statistics},
}

@book{hiriart2004fundamentals,
  title={Fundamentals of convex analysis},
  author={Hiriart-Urruty, Jean-Baptiste and Lemar{\'e}chal, Claude},
  year={2004},
  publisher={Springer Science \& Business Media}
}

@article{iacovissi2026corruptions,
  title={{Corruptions of Supervised Learning Problems: Typology and Mitigations}},
  author={Iacovissi, Laura and Lu, Nan and Williamson, Robert C.},
  journal={Journal of machine learning research},
  volume={27},
  pages={1--73},
  year={2026}
}

@article{khan2024comparison,
    title = {On comparisons of information structures with infinite states},
    journal = {Journal of Economic Theory},
    volume = {218},
    year = {2024},
    author = {M. Ali Khan and Haomiao Yu and Zhixiang Zhang}
}

@article{kullback1951information,
  author    = {Kullback, Solomon and Leibler, Richard A.},
  title     = {On Information and Sufficiency},
  journal   = {The Annals of Mathematical Statistics},
  year      = {1951},
  volume    = {22},
  number    = {1},
  pages     = {79--86},
}

@inproceedings{lee2024contraction,
  author={Lee, Dongmin and Makur, Anuran},
  booktitle={Annual Allerton Conference on Communication, Control, and Computing},
  title={Contraction Coefficients of Product Symmetric Channels},
  year={2024},
  pages={1-8},
}

@inproceedings{lu2023measuring,
    title={{Measuring Pointwise V-Usable Information In-Context-ly}},
    author={Sheng Lu and Shan Chen and Yingya Li and Danielle Bitterman and Guergana K. Savova and Iryna Gurevych},
    booktitle={Conference on Empirical Methods in Natural Language Processing},
    year={2023},
}

@article{le1964sufficiency,
  title={Sufficiency and approximate sufficiency},
  author={Le Cam, Lucien},
  journal={The Annals of Mathematical Statistics},
  pages={1419--1455},
  volume  = {35},
  number  = {4},
  year={1964}
}

@article{makur2020comparison,
	author = {Makur, A. and Zheng, L.},
	journal = {Problems of Information Transmission},
	number = {2},
	pages = {103--156},
	title = {Comparison of Contraction Coefficients for f-Divergences},
	volume = {56},
	year = {2020}
}

@article{paunescu2017generalization,
  author  = {P{\u a}unescu, Liviu and R{\u a}dulescu, Florin},
  title   = {A generalisation to {B}irkhoff--von {N}eumann theorem},
  journal = {Advances in Mathematics},
  volume  = {308},
  year    = {2017},
  pages   = {836--858}
}

@article{penot2000harmonic,
  title={Harmonic sum and duality},
  author={Penot, Jean-Paul and Zalinescu, Constantin},
  journal={Journal of Convex Analysis},
  volume={7},
  number={1},
  pages={95--114},
  year={2000},
  publisher={Heldermann Verlag}
}

@book{polyanskiy2025information,
  title={Information theory: From coding to learning},
  author={Polyanskiy, Yury and Wu, Yihong},
  year={2025},
  publisher={Cambridge university press}
}

@book{rao1983theory,
  title={Theory of charges: A study of finitely additive measures},
  author={Rao, K. P. S. Bhaskara and Rao, M. Bhaskara},
  year={1983},
  publisher={Academic Press}
}

@book{rockafellar1998variational,
  title={Variational analysis},
  author={Rockafellar, R. Tyrrell and Wets, Roger J.-B.},
  year={1998},
  publisher={Springer}
}

@book{rockafellar1970convexanalysis,
  author = {Rockafellar, R. Tyrrell},
  publisher = {Princeton University Press},
  title = {Convex analysis},
  year = 1970
}

@article{reid2011information,
  title={Information, divergence and risk for binary experiments},
  author={Reid, Mark D. and Williamson, Robert C.},
  year={2011},
  publisher={MIT Press},
  journal={Journal of machine learning research},
  pages={731 -- 817},
  volume={12}
}

@article{safarov2005birkhoff,
  title={Birkhoff's theorem and multidimensional numerical range},
  author={Safarov, Yu.},
  journal={Journal of Functional Analysis},
  volume={222},
  number={1},
  pages={61--97},
  year={2005},
  publisher={Elsevier}
}

@book{schechter1997handbook,
  author    = {Schechter, Eugene},
  title     = {Handbook of Analysis and Its Foundations},
  publisher = {Academic Press},
  year      = {1997},
  address   = {San Diego},
}

@book{shiryaev2000statistical,
  title={Statistical Experiments And Decision, Asymptotic Theory},
  author={Shiryaev, Albert N. and Spokoiny, Vladimir G.},
  year={2000},
  publisher={World Scientific}
}

@book{shalev2014understanding,
  title={Understanding machine learning: From theory to algorithms},
  author={Shalev-Shwartz, Shai and Ben-David, Shai},
  year={2014},
  publisher={Cambridge university press}
}

@book{torgersen1991comparison,
  title={Comparison of statistical experiments},
  author={Torgersen, Erik},
  year={1991},
  publisher={Cambridge University Press}
}

@inproceedings{tishby1999information,
  title        = {The Information Bottleneck Method},
  author       = {Tishby, Naftali and Pereira, Fernando C. and Bialek, William},
  booktitle    = {Proceedings of the 37th Annual Allerton Conference on Communications, Control and Computing},
  year         = {1999},
  pages        = {368--377},
}

@inproceedings{tishby2015deep,
  title={Deep learning and the information bottleneck principle},
  author={Tishby, Naftali and Zaslavsky, Noga},
  booktitle={2015 IEEE Information Theory Workshop (ITW)},
  pages={1--5},
  year={2015},
  organization={IEEE}
}

@inproceedings{turgeman2026does,
    title={{Does the Data Processing Inequality Reflect Practice? On the Utility of Low-Level Tasks}},
    author={Roy Turgeman and Tom Tirer},
    booktitle={International Conference on Learning Representations (ICLR)},
    year={2026}
}

@article{van2015theory,
  title={A theory of feature learning},
  author={van Rooyen, Brendan and Williamson, Robert C..},
  journal={arXiv preprint arXiv:1504.00083},
  year={2015}
}

@inproceedings{vovk1990aggregating,
  author    = {Vladimir Vovk},
  title     = {Aggregating Strategies},
  booktitle = {Conference on Computational Learning Theory (COLT)},
  year      = {1990},
  pages     = {371--383},
}

@book{walley1991statistical,
  author    = {Walley, Peter},
  title     = {Statistical Reasoning with Imprecise Probabilities},
  publisher = {Chapman and Hall},
  year      = {1991},
  address   = {London}
}

@article{williamson2024information,
  author  = {Robert C. Williamson and Zac Cranko},
  title   = {Information Processing Equalities and the Information--Risk Bridge},
  journal = {Journal of machine learning research},
  year    = {2024},
  volume  = {25},
  number  = {103},
  pages   = {1--53}
}

@article{williamson2023geometry,
  author  = {Robert C. Williamson and Zac Cranko},
  title   = {The Geometry and Calculus of Losses},
  journal = {Journal of machine learning research},
  year    = {2023},
  volume  = {24},
  number  = {342},
  pages   = {1--72}
}

@article{williamson2016composite,
  title={Composite multiclass losses},
  author={Williamson, Robert C. and Vernet, Elodie and Reid, Mark D.},
  journal = {Journal of machine learning research},
  year    = {2016},
  volume  = {17},
  number  = {222},
  pages   = {1--52}
}

@article{xu2020theory,
  title={A theory of usable information under computational constraints},
  author={Xu, Yilun and Zhao, Shengjia and Song, Jiaming and Stewart, Russell and Ermon, Stefano},
  journal={arXiv preprint arXiv:2002.10689},
  year={2020}
}

@article{zhao2022fundamental,
  author  = {Han Zhao and Chen Dan and Bryon Aragam and Tommi S. Jaakkola and Geoffrey J. Gordon and Pradeep Ravikumar},
  title   = {{Fundamental Limits and Tradeoffs in Invariant Representation Learning}},
  journal = {Journal of Machine Learning Research},
  year    = {2022},
  volume  = {23},
  number  = {340},
  pages   = {1--49},
}

@book{zalinescu2002convex,
  author    = {Constantin Z\u{a}linescu},
  title     = {Convex Analysis in General Vector Spaces},
  year      = {2002},
  publisher = {World Scientific},
  address   = {Singapore},
}

@inproceedings{vovk1995game,
  title={A game of prediction with expert advice},
  author={Vovk, Vladimir G.},
  booktitle={Proceedings of the eighth annual conference on Computational learning theory},
  pages={51--60},
  year={1995}
}

@InProceedings{Kalnishkan2002absence,
author="Kalnishkan, Yuri
and Vyugin, Michael V.",
editor="Cesa-Bianchi, Nicol{\`o}
and Numao, Masayuki
and Reischuk, R{\"u}diger",
title="On the Absence of Predictive Complexity for Some Games",
booktitle="Algorithmic Learning Theory",
year="2002",
publisher="Springer Berlin Heidelberg",
pages="164--172",
}

@InProceedings{kalnishkan2002mixability,
author="Kalnishkan, Yuri
and Vyugin, Michael V.",
editor="Kivinen, Jyrki
and Sloan, Robert H.",
title="Mixability and the Existence of Weak Complexities",
booktitle="Computational Learning Theory",
year="2002",
publisher="Springer Berlin Heidelberg",
pages="105--120",
}

% \newpage
% \listoffigures
% \listoftables

%%%%%%%%%%%%%%
\newpage
\appendix

%%% ==================================================================
%%% Warm-Up section.
%%%
%%% (1) Macros and theorem environments below are guarded, so they are
%%%     silently ignored if the main preamble already provides them.
%%% (2) References into the main text are macros, collected in one block,
%%%     because this file does not know your label names. Replace e.g.
%%%     \renewcommand{\refCorFiveThree}{Corollary~\ref{cor:set-containment}}
%%%     once, and the whole section follows.
%%% ==================================================================
\providecommand{\reals}{\mathbb{R}}
\providecommand{\cX}{\mathcal{X}}
\providecommand{\cY}{\mathcal{Y}}
\providecommand{\cH}{\mathcal{H}}
\providecommand{\cF}{\mathcal{F}}
\providecommand{\cK}{\mathcal{K}}
\providecommand{\Mkern}{\mathcal{M}}
\providecommand{\spr}{\operatorname{spr}}
\providecommand{\cco}{\overline{\operatorname{co}}}
\providecommand{\BR}{\operatorname{BR}}
\providecommand{\Deltaca}{\Delta_{\mathrm{ca}}}
\providecommand{\Aff}{\operatorname{Aff}}
\providecommand{\Ent}{\mathbb{H}}
\providecommand{\texorpdfstring}[2]{#1}
%%% ==================================================================

\section{Running Example: Proofs and Explanations}
\label{app:warm-up} 
This section provides the proofs for the logistic regression example used throughout the paper.

\subsection{Proofs for Example~\ref{sec:wu-setup}}
\label{app:wu-setup}

\begin{proof}[Proof Eq.~\eqref{eq:wu-BRall} and Eq.~\eqref{eq:wu-BRlog}]
    For $\cH_{\mathrm{all}}$ the predictions at the four atoms are unconstrained, so
infimum and sum interchange and
$\inf_{\pi \in \Deltaca(\cY)} \mathbb{E}_{Y\sim\eta_i} \ell_{\log}(\pi, Y)
= \Ent(\eta_i)$, which is Eq.~\eqref{eq:wu-BRall} (\cf Remark~\ref{remark: br and cbr}).
The risk of $h_\alpha$ under $\phi_\eta$ is $\frac14\sum_i G_{\eta_i}(w_\alpha(x^{(i)}))$ by the definition of Bayes risk (Definition~\ref{def:br-risk}), and Lemma~\ref{lemma:wu-param} identifies the attainable vectors $(w_\alpha(x^{(i)}))_i$ with $V$, which is Eq.~\eqref{eq:wu-BRlog}.
\end{proof}

\begin{lemma}[What logistic regression can do on four atoms]
\label{lemma:wu-param}
The evaluation map
$\alpha \mapsto w(\alpha) := \bigl( w_\alpha(x^{(1)}),\dots, w_\alpha(x^{(4)}) \bigr)$
is a linear bijection from $\reals^3$ onto
\begin{align}
    \label{eq:wu-V}
    V := \{ w \in \reals^4 \mid w_1 + w_3 = w_2 + w_4 \} .
\end{align}
\end{lemma}
\begin{proof}
We have $w_1 = \alpha_0+\alpha_1+\alpha_2$, $w_2 = \alpha_0-\alpha_1+\alpha_2$,
$w_3 = \alpha_0-\alpha_1-\alpha_2$ and $w_4 = \alpha_0+\alpha_1-\alpha_2$, hence
$w_1+w_3 = 2\alpha_0 = w_2+w_4$ and the image is contained in $V$. The map is
injective because $\alpha_0 = (w_1+w_3)/2$, $\alpha_1 = (w_1-w_2)/2$ and
$\alpha_2 = (w_1-w_4)/2$; since $\dim V = 3$, it is onto $V$.
\end{proof}

\subsection{Proofs for Example~\ref{sec:wu-dpi-breaks}}
\label{app:wu-dpi-breaks}

Let us first introduce the required mathematical tools.
For a simple attribute corruption $\ta \in \Mkern(\cX,\cX)$ we write
$\ka_\ta \in \Mkern(\cX\times\cY, \cX\times\cY)$ for the associated joint
corruption, $\ka_\ta((x,y), \cdot) := \ta(x,\cdot)\otimes \idy(y,\cdot)$, so
that
\begin{align}
    \label{eq:wu-action}
    \hat\ka_\ta (\ell\circ h)(x,y) = \int_\cX \ell(\dot h(\tilde x), y)\,\ta(x,d\tilde x)
\end{align}
as in Eq.~\eqref{eq:kernel application in the remark for recovery of DPI}, and $\check\ka_\ta \phi$ is the corrupted data
distribution.

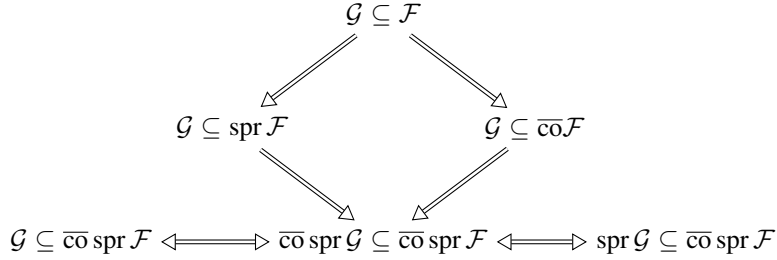
\begin{figure}
    \centering
    \begin{tikzpicture}[>=Stealth, node distance=2.5cm]
    
      % Nodes
      \node (A) at (0,1.5) {$\cG \subseteq \cF$};
      \node (B) at (-2,0) {$\cG \subseteq \spr \, \cF$};
      \node (C) at (2,0) {$\cG \subseteq \overline{\co} \cF$};
      \node (D) at (0,-1.5) {$\cspr \, \cG \subseteq \cspr \, \cF$};
      \node (E) at (-4,-1.5) {$\cG \subseteq \cspr \, \cF$};
      \node (F) at (4,-1.5) {$\spr \, \cG \subseteq \cspr \, \cF$};

      % Arrows (implications)
      \draw[-{Triangle[open,scale=0.7]}, double, double distance=1.0] (A) -- (B) node[midway, left] {};
      \draw[-{Triangle[open,scale=0.7]}, double, double distance=1.0] (A) -- (C) node[midway, right] {};
      \draw[-{Triangle[open,scale=0.7]}, double, double distance=1.0] (B) -- (D) node[midway, left] {};
      \draw[-{Triangle[open,scale=0.7]}, double, double distance=1.0] (C) -- (D) node[midway, right] {};
      \draw[{Triangle[open,scale=0.7]}-{Triangle[open,scale=0.7]}, double, double distance=1.0] (D) -- (E) node[midway, right] {};
      \draw[{Triangle[open,scale=0.7]}-{Triangle[open,scale=0.7]}, double, double distance=1.0] (D) -- (F) node[midway, right] {};
    
\end{tikzpicture}
    \caption{\textbf{Caption: Hierarchy of set containment conditions.} The set containment introduced by Theorem~\ref{theo:support-ineq-characterization-superpredictionset} (bottom) and its equivalents are the weakest conditions, while the inclusion of the simple sets is the strongest (top). The bottom level equivalences hold by definition of convex hull and superprediction set. 
    }
    \label{fig:implications for set containment condition}
\end{figure}

We consider four corruptions, all of them attribute corruptions. The corruption used in Section~\ref{sec:wu-dpi-breaks} is the composite corruption d).
\begin{definition}[Four corruptions]
\label{def:wu-corruptions}
Fix a noise level $\varepsilon \in (0, \tfrac12)$.
\begin{enumerate}
    \item[a)] \emph{(permutation)} $R(x_1,x_2) := (-x_2, x_1)$, the rotation by
    $\pi/2$, which permutes the atoms cyclically,
    $x^{(1)} \mapsto x^{(2)} \mapsto x^{(3)} \mapsto x^{(4)} \mapsto x^{(1)}$.
    Set $\ka_a := \ka_{\ta_R}$, with $\ta_R$ the deterministic kernel
    induced by $R$.
    \item[b)] \emph{(swap of the lower left with the upper left quadrant)}
    \begin{align*}
        S(x_1,x_2) := \begin{cases} (x_1,\ x_2), & x_1 \ge 0,\\ (x_1,-x_2), & x_1 < 0,\end{cases}
    \end{align*}
    a measurable bijection of $\cX$ which swaps
    $x^{(2)} \leftrightarrow x^{(3)}$ and fixes $x^{(1)}, x^{(4)}$. Set
    $\ka_b := \ka_{\ta_S}$.
    \item[c)] \emph{(noise on the first coordinate)} with
    $\rho(x_1,x_2) := (-x_1,x_2)$, let
    $\ta_\varepsilon := (1-\varepsilon)\,\ta_{\mathrm{id}} + \varepsilon\, \ta_\rho$,
    i.e.\ the first coordinate is resampled with probability $\varepsilon$. Set
    $\ka_c := \ka_{\ta_\varepsilon}$.
    \item[d)] \emph{(noise, then swap)} $\ta_d := \ta_S \circ \ta_\varepsilon$,
    so that $\check\ka_d = \check\ka_b \circ \check\ka_c$ and
    $\hat\ka_d = \hat\ka_c \circ \hat\ka_b$. Set
    $\ka_d := \ka_{\ta_d}$.
\end{enumerate}
\end{definition}

% TODO make in one line!
Applying Definition~\ref{def:wu-corruptions} to $\phi_{(1,0,1,0)}$ gives, with
$\bar\varepsilon := 1-\varepsilon$,
\begin{align}
    \label{eq:wu-corrupted}
    \check\ka_a \phi_{(1,0,1,0)} = \phi_{(0,1,0,1)},\quad
    \check\ka_b \phi_{(1,0,1,0)} = \phi_{(1,1,0,0)},\quad
    \check\ka_c \phi_{(1,0,1,0)} = \phi_{(\bar\varepsilon, \varepsilon, \bar\varepsilon, \varepsilon)},\quad
    \check\ka_d \phi_{(1,0,1,0)} = \phi_{(\bar\varepsilon, \bar\varepsilon, \varepsilon, \varepsilon)} .
\end{align}
For c): the atom $x^{(1)}$ keeps its own mass, with label $1$, with probability
$\bar\varepsilon$ and receives mass from $x^{(2)} = \rho(x^{(1)})$, with label
$0$, with probability $\varepsilon$, so $\eta_1 = \bar\varepsilon$; the other
three atoms are analogous. For d) one composes with the swap, which exchanges
the second and the third entry of $\eta$. Note that in $\check\ka_b \phi_{(1,0,1,0)}$ and
$\check\ka_d \phi_{(1,0,1,0)}$ the label depends on $x_2$ only, i.e.\ the corrupted problem
is (softly) linearly separable, whereas $\phi_{(1,0,1,0)}$ and $\check\ka_c \phi_{(1,0,1,0)}$ have the
XOR shape.

\begin{table}[t]
\centering
\caption{\textbf{Corruptions in the running example}. $\phi_{(1,0,1,0)}$ is the XOR distribution,
$\varepsilon \in (0,\tfrac12)$ the noise level and $\Ent$ the binary entropy;
``$=$'', ``$<$'', ``$>$'' compare the corrupted with the clean Bayes risk.
``GDPI'' refers to Definition~\ref{def:gdpi for br} for the respective class, i.e.\ to the
inequality holding for \emph{all} $\phi$.}
\label{tab:warmup-summary}
\begin{tabular}{llccc}
\hline
corruption & generating maps & $\BR_{\ell_{\log}\circ\cH}$ & $\BR_{\ell_{\log}\circ\cH_{\mathrm{all}}}$ & GDPI for $\cH$\\
\hline
--- (clean $\phi_{(1,0,1,0)}$) & --- & $\log 2$ & $0$ & ---\\
a) rotation & affine, bijective & $\log 2$\ \ ($=$) & $0$\ \ ($=$) & holds\\
b) quadrant swap & bijective, not affine & $0$\ \ ($<$) & $0$\ \ ($=$) & \textbf{fails}\\
c) noise on $x_1$ & affine & $\log 2$\ \ ($=$) & $\Ent(\varepsilon)$\ \ ($>$) & holds\\
d) noise, then swap & not affine & $\Ent(\varepsilon)$\ \ ($<$) & $\Ent(\varepsilon)$\ \ ($>$) & \textbf{fails}\\
\hline
\end{tabular}
\end{table}

\begin{lemma}
\label{lemma:bayes risks computed for corruptions a,b,c,d}
    With the corruptions of Definition~\ref{def:wu-corruptions} we obtain the following Bayes risks,
\begin{align*}
    \BR_{\ell_{\log}\circ\cH}(\phi_{(1,0,1,0)}) = \BR_{\ell_{\log}\circ\cH}(\check\ka_a \phi_{(1,0,1,0)})
    = \BR_{\ell_{\log}\circ\cH}(\check\ka_c \phi_{(1,0,1,0)}) &= \log 2,\\
    \BR_{\ell_{\log}\circ\cH}(\check\ka_b \phi_{(1,0,1,0)}) = 0, \qquad
    \BR_{\ell_{\log}\circ\cH}(\check\ka_d \phi_{(1,0,1,0)}) &= \Ent(\varepsilon),
\end{align*}
while, for the full class,
\begin{align*}
    \BR_{\ell_{\log}\circ\cH_{\mathrm{all}}}(\phi_{(1,0,1,0)})
    = \BR_{\ell_{\log}\circ\cH_{\mathrm{all}}}(\check\ka_a \phi_{(1,0,1,0)})
    = \BR_{\ell_{\log}\circ\cH_{\mathrm{all}}}(\check\ka_b \phi_{(1,0,1,0)}) &= 0, \\
    \BR_{\ell_{\log}\circ\cH_{\mathrm{all}}}(\check\ka_c \phi_{(1,0,1,0)})
    = \BR_{\ell_{\log}\circ\cH_{\mathrm{all}}}(\check\ka_d \phi_{(1,0,1,0)}) &= \Ent(\varepsilon) .
\end{align*}
\end{lemma}
\begin{proof}

\emph{$\phi_{(1,0,1,0)}$ and $\check\ka_c \phi_{(1,0,1,0)}$.} Let
$\eta = (\bar\varepsilon, \varepsilon, \bar\varepsilon, \varepsilon)$ with
$\varepsilon \in [0,\tfrac12)$; the case $\varepsilon = 0$ is $\phi_{(1,0,1,0)}$. Let
$w \in V$ and put $s := \tfrac12(w_1+w_3) = \tfrac12(w_2+w_4)$. By convexity of
$G_{\bar\varepsilon}$ and $G_{\varepsilon}$ (Jensen),
\begin{align*}
    \sum_{i=1}^4 G_{\eta_i}(w_i)
    &= \bigl[G_{\bar\varepsilon}(w_1) + G_{\bar\varepsilon}(w_3)\bigr]
     + \bigl[G_{\varepsilon}(w_2) + G_{\varepsilon}(w_4)\bigr]
     \ \ge\ 2 G_{\bar\varepsilon}(s) + 2 G_{\varepsilon}(s)\\
    &= 2\bigl[G_1(s) + G_0(s)\bigr]
     = 2\log\bigl((1+e^{-s})(1+e^{s})\bigr)
     = 2\log\bigl(2 + e^{s} + e^{-s}\bigr)\ \ge\ 4\log 2 ,
\end{align*}
with equality if and only if $w_1 = w_3$, $w_2 = w_4$ and $s = 0$, i.e.\ iff
$w = 0$, i.e.\ iff $\alpha = 0$. Hence the infimum in Eq.~\eqref{eq:wu-BRlog}
equals $4\log 2$, is attained, and only by the constant predictor $\tfrac12$;
so $\BR_{\ell_{\log}\circ\cH}(\phi_\eta) = \log 2$. In words: \emph{the noise of
c) leaves the constrained Bayes risk exactly at the chance level.}

\emph{$\check\ka_a \phi_{(1,0,1,0)}$.} $R$ is a linear bijection, so $\cH \circ R = \cH$ by
Lemma~\ref{lemma:wu-affine} and, by Lemma~\ref{lemma:kernel-is-adjoint-to-risk},
\begin{align*}
    \BR_{\ell\circ\cH}(\check\ka_a \phi_{(1,0,1,0)}) = \BR_{\hat\ka_a(\ell\circ\cH)}(\phi_{(1,0,1,0)})
    = \BR_{\ell\circ(\cH\circ R)}(\phi_{(1,0,1,0)}) = \BR_{\ell\circ\cH}(\phi_{(1,0,1,0)})
\end{align*}
for \emph{every} loss $\ell$; the same argument applies to
$\cH_{\mathrm{all}}$.

\emph{$\check\ka_d \phi_{(1,0,1,0)}$.} Here $\eta = (\bar\varepsilon, \bar\varepsilon,
\varepsilon,\varepsilon)$, so
$\sum_i G_{\eta_i}(w_i) \ge \sum_i \min_{w} G_{\eta_i}(w) = 4\,\Ent(\varepsilon)$,
and this bound is attained inside $V$: the choice $\alpha = (0,0,c)$ with
$c := \log\frac{\bar\varepsilon}{\varepsilon}$ gives $w = (c,c,-c,-c) \in V$
and predictions $\sigma(\pm c) \in \{\bar\varepsilon, \varepsilon\}$ matching
$\eta$ atom by atom. Hence
$\BR_{\ell_{\log}\circ\cH}(\check\ka_d \phi_{(1,0,1,0)}) = \Ent(\varepsilon) =
\BR_{\ell_{\log}\circ\cH_{\mathrm{all}}}(\check\ka_d \phi_{(1,0,1,0)})$: after the
corruption, logistic regression is well specified.

\emph{$\check\ka_b \phi_{(1,0,1,0)}$.} This is the case $\varepsilon = 0$ of the previous
computation: $\alpha = (0,0,c)$ gives risk $\log(1+e^{-c}) \to 0$ as
$c \to \infty$, and the risk is nonnegative, so the infimum is $0$.
\end{proof}

\begin{lemma}[Affine invariance of logistic regression]
\label{lemma:wu-affine}
Let $f \colon \reals^2 \to \reals^2$, $f(x) = Mx + c$ with
$M \in \reals^{2\times2}$, $c \in \reals^2$, be affine. Then for every
$h_\alpha \in \cH$ there is $h_{\alpha^f} \in \cH$ with
$\dot h_{\alpha^f} = \dot h_{\alpha} \circ f$; that is, $\cH \circ f \subseteq \cH$,
with equality if $f$ is bijective. Consequently $\cH$ is
$\cF$-attribute-invariant in the sense of Proposition~\ref{prop:sufficient conditions dpi attribute corruption}(1) for every
$\cF \subseteq \Aff(\reals^2)$.
\end{lemma}
\begin{proof}
Write $\alpha_{1:2} := (\alpha_1,\alpha_2)^\top$. Then
%\begin{align*}
   $ w_\alpha(f(x)) = \alpha_{1:2}^\top (Mx + c) + \alpha_0
    = (M^\top \alpha_{1:2})^\top x + (\alpha_{1:2}^\top c + \alpha_0)
    = w_{\alpha^f}(x)$
%\end{align*}
with $\alpha^f := (\alpha_{1:2}^\top c + \alpha_0,\, M^\top\alpha_{1:2})
\in \reals^3$, hence $\dot h_\alpha \circ f = \dot h_{\alpha^f} \in \cH$. If
$f$ is bijective then $M$ is invertible, $\alpha \mapsto \alpha^f$ is a linear
bijection of $\reals^3$, and $\cH\circ f = \cH$.
\end{proof}

\begin{remark}[Attainment]
\label{rem:wu-attainment}
In case b) the infimum is not attained: $\cH$ separates the corrupted problem
only in the limit $c \to \infty$ (see definition of $c$ in the proof of Lemma~\ref{lemma:wu-affine}). Since $\BR$ is defined as an infimum
(Definition~\ref{def:conditional bayes risk}), this is immaterial; under the truncation of
Remark~\ref{rm:wu-boundedness} (a) the value is
$\log(1+e^{-B}) \in (0,\log 2)$ and attained. In all other cases above the
infimum is attained.
\end{remark}

\begin{remark}[Boundedness]
\label{rm:wu-boundedness}
The work-horse Theorem~\ref{theo:support-ineq-characterization-superpredictionset} for the Corollaries~\ref{corollary:kernel-gdpi-characterization} and~\ref{corollary:no-kernel-gdpi-characterization} demand that $\ell\circ h \in \Bb(\cZ)$ for every $h \in \cH$. That is, for the proper mathematical statement $\ell\circ h$ needs to be bounded and measurable. The cross-entropy loss $\ell_{\log}$ with the logistic regression hypothesis class $\cH$ is not necessarily bounded. 
There are two standard repairs, and none of our statements in the running example depends on the choice.
(a) Replace $\cH$ by $\cH^B \coloneqq \{h_\alpha \colon \|\alpha\|_\infty \le B\}$ for some
$B < \infty$. On the finitely supported distributions used below all losses are
then bounded by $\log(1+e^{3B})$, all infima are attained, and every strict
inequality we prove survives; the only value that changes is
$\BR_{\ell_{\log}\circ\cH}(\check\ka_b \phi_{(1,0,1,0)}) = \log(1+e^{-B})$ in
Lemma~\ref{lemma:bayes risks computed for corruptions a,b,c,d}, which is still $< \log 2$ for every $B>0$.
(b) Replace $\ell_{\log}$ by $\ell^{c}_{\log} := \min\{\ell_{\log}, c\}$ with
$c > \log 2$.
We work with $\ell_{\log}$ itself for readability and flag the one place where
attainment of the infimum matters (Remark~\ref{rem:wu-attainment}).
\end{remark}

\subsection{Proof for Example~\ref{sec:wu-gdpi}}
\label{app:wu-gdpi}
\begin{proof}
By Lemma~\ref{lemma:wu-affine} the class $\cH$ is $\cF$-attribute-invariant for
every $\cF \subseteq \Aff(\reals^2)$, so the hypotheses of Proposition~\ref{prop:sufficient conditions dpi attribute corruption}
are met and its proof gives the displayed containment. For $\ta \in T_\cF$ one
can also argue directly: writing $\ta = \sum_i \alpha_i \ta_{f_i}$ with
$f_i \in \cF$, Eq.~\eqref{eq:wu-action} gives
$\hat\ka_\ta(\ell\circ h) = \sum_i \alpha_i\, \ell\circ(h\circ f_i)$ with
$h \circ f_i \in \cH$, so the left hand side lies in
$\operatorname{co}(\ell\circ\cH)$, and Proposition~\ref{proposition:Bayes Risk Inequality with Simple Attribute Corruption} applies.
\end{proof}

\subsection{Proof for Example~\ref{sec:wu-comparing two learning problems}}
\label{app:wu-comparing two learning problems}
\begin{figure}
\centering
\begin{tikzpicture}[x=2.7cm, y=2.7cm, >=stealth]
  % ---- superprediction-set hatching (drawn first, as background) ----
  % spr(l_01): area up-right of the red staircase
  \begin{scope}
    \clip (0,1) -- (0.5,1) -- (0.5,0.5) -- (1,0.5) -- (1,0)
          -- (1.9,0) -- (1.9,1.9) -- (0,1.9) -- cycle;
    \fill[pattern={Lines[angle=45, distance=3.2pt, line width=0.35pt]},
          pattern color=OrangeRed!35!white] (0,0) rectangle (1.9,1.9);
  \end{scope}
  % spr(l_log): area up-right of the blue cross-entropy curve
  \begin{scope}
    \clip plot[domain=-1.7379:1.7379, samples=140, smooth]
              ({ln(1+exp(\x))}, {ln(1+exp(-\x))})
          -- (1.9,1.9) -- cycle;
    \fill[pattern={Lines[angle=135, distance=3.2pt, line width=0.35pt]},
          pattern color=MidnightBlue!35!white] (0,0) rectangle (1.9,1.9);
  \end{scope}

  % the two regions that convexification adds (drawn on top of the hatching)
  \fill[orange!18] (0,1) -- (0.5,1) -- (0.5,0.5) -- cycle;
  \fill[orange!18] (0.5,0.5) -- (1,0.5) -- (1,0) -- cycle;
  % axes
  \draw[wuaxis] (-0.05,0) -- (1.95,0) node[right, black, font=\small] {$\ell(\pi,0)$};
  \draw[wuaxis] (0,-0.05) -- (0,1.95) node[above, black, font=\small] {$\ell(\pi,1)$};
  \foreach \t/\l in {0.5/{\tfrac12}, 1/{1}}{
     \draw[black!45] (\t,0.02) -- (\t,-0.02) node[below, wutiny] {$\l$};
     \draw[black!45] (0.02,\t) -- (-0.02,\t) node[left, wutiny] {$\l$};}
  % staircase boundary of spr(l_01)
  \draw[OrangeRed, line width=1pt]
       (0,1.9) -- (0,1) -- (0.5,1) -- (0.5,0.5) -- (1,0.5) -- (1,0) -- (1.9,0);
  % boundary of the convex hull
  \draw[wuboundary] (0,1) -- (1,0);
  % cross-entropy curve
  \draw[MidnightBlue, line width=1pt, smooth, samples=120, domain=-1.7379:1.7379]
       plot ({ln(1+exp(\x))}, {ln(1+exp(-\x))});
       %    % @rab IS EQUIVALENT TO \draw[MidnightBlue, line width=1pt, smooth, samples=120, domain=0.01:0.99]
%    % plot ({-ln(\x)}, {-ln(1-\x))});
  % the three loss pairs attainable with the zero-one loss
  \fill[OrangeRed] (0,1) circle (1.5pt);
  \fill[OrangeRed] (0.5,0.5) circle (1.5pt);
  \fill[OrangeRed] (1,0) circle (1.5pt);
  % the witness point on the curve, and the mixture that lies below it
  \fill[MidnightBlue] (0.9741,0.4741) circle (1.6pt);
  \fill[ForestGreen] (0.6667,0.3333) circle (1.8pt);
  %\draw[ForestGreen, ->, line width=0.7pt] (0.6667,0.3333) -- (0.952,0.462);
  \node[wulab, ForestGreen, below left=1pt and -2pt] at (0.6667,0.3333) {$m$};
  %\node[wutiny] at (0.62,1.55) {$\spr(\ell_{01})$};
    \node[wutiny, fill=white, fill opacity=0.75, text opacity=1,
      inner sep=1.5pt, rounded corners=1pt] at (0.62,1.55) {$\spr(\ell_{01})$};
  % ---- magnified inset of the corner region [0.90,1.04] x [0.40,0.55] ----
  \fill[white] (1.30,0.85) rectangle (1.93,1.525); % mask hatching behind inset
  \draw[black!45, line width=0.3pt] (0.90,0.40) rectangle (1.04,0.55);
  \draw[black!45, line width=0.3pt] (1.04,0.55) -- (1.30,1.525);
  \draw[black!45, line width=0.3pt] (1.04,0.40) -- (1.30,0.85);
  \begin{scope}
    \clip (1.30,0.85) rectangle (1.93,1.525);
    \fill[orange!18] (1.30,0.85) rectangle (1.75,1.30);
    \draw[MidnightBlue, line width=1pt, smooth, samples=80, domain=0.30:0.75]
      plot ({1.30+4.5*(ln(1+exp(\x))-0.90)}, {0.85+4.5*(ln(1+exp(-\x))-0.40)});
    \draw[OrangeRed, line width=1pt] (1.30,1.30) -- (1.75,1.30) -- (1.75,0.85);
    \fill[MidnightBlue] (1.633,1.183) circle (1.8pt);
  \end{scope}
  \draw[black!45, line width=0.3pt] (1.30,0.85) rectangle (1.93,1.525);
  \node[wutiny, above, fill=white, fill opacity=0.75, text opacity=1,
      inner sep=1.5pt, rounded corners=1pt] at (1.615,1.53) {magnified $4.5\times$};
  \node[wulab, orange, below, align=center, fill=white, fill opacity=0.75, text opacity=1,
      inner sep=1.5pt, rounded corners=1pt] at (1.615,0.82)
       {$\ell_{\log}(\pi_t,\cdot)$ at $t=\tfrac12$};
  % legend, below the axes
  \begin{scope}[shift={(0.0,-0.34)}, wulab]
    \draw[OrangeRed, line width=1pt] (0,0) -- (0.14,0);
      \node[right=1pt, OrangeRed] at (0.14,0) {$\partial\spr(\ell_{01})$};
    \draw[wuboundary] (1.0,0) -- (1.14,0);
      \node[right=1pt, ForestGreen] at (1.14,0) {$\partial\cco\spr(\ell_{01})$};
    \draw[MidnightBlue, line width=1pt] (0,-0.17) -- (0.14,-0.17);
      \node[right=1pt, MidnightBlue] at (0.14,-0.17) {$\ell_{\log}(\pi,\cdot)$};
    \fill[orange!18] (1.0,-0.205) rectangle (1.14,-0.135);
      \node[right=1pt, orange] at (1.14,-0.17) {added by convexification};
    \fill[pattern={Lines[angle=45, distance=3.2pt, line width=0.35pt]},
          pattern color=OrangeRed!35!white] (0,-0.375) rectangle (0.14,-0.305);
      \node[right=1pt, OrangeRed] at (0.14,-0.34) {$\spr(\ell_{01})$};
    \fill[pattern={Lines[angle=135, distance=3.2pt, line width=0.35pt]},
          pattern color=MidnightBlue!35!white] (1.0,-0.375) rectangle (1.14,-0.305);
      \node[right=1pt, MidnightBlue] at (1.14,-0.34) {$\spr(\ell_{\log})$};
  \end{scope}
\end{tikzpicture}
\caption{\textbf{Caption: Why convexification matters
(Remark~\ref{rem:wu-convexity}), drawn for one fixed attribute $x$.} A
prediction is a point $\pi \in \Deltaca(\cY)$, and we plot the pair of loss
values $(\ell(\pi,0), \ell(\pi,1))$ it produces; a superprediction set is then
the region up and to the right of the plotted values. The tie-broken zero-one
loss $\ell_{01}$ produces only three such pairs---$(0,1)$ when
$\pi_1 < \tfrac12$, $(1,0)$ when $\pi_1 > \tfrac12$, and
$(\tfrac12,\tfrac12)$ at the tie---so $\spr(\ell_{01})$ is the region bounded by
the red staircase, which is \emph{not} convex; its closed convex hull is bounded
by the green segment $\ell(\pi,0) + \ell(\pi,1) = 1$ and adds the two orange
triangles. The cross-entropy loss traces the blue curve
$t \mapsto (\log(1+e^{t}), \log(1+e^{-t}))$, $t = w_\alpha(x)$, and it passes
\emph{through} an orange triangle: at the marked $t = \tfrac12$ the point lies in
$\cco\spr(\ell_{01})$ but not in $\spr(\ell_{01})$, i.e.\ no single model has
zero-one loss below the cross-entropy loss there. A \emph{mixture} of models
does: at this $x$ one third of the intercept-shifted classifiers predicts $0$ and two thirds predict $1$, which realizes
$m = \tfrac13(0,1) + \tfrac23(1,0)$, and $m$ is below the blue point in both
coordinates. This is the containment
$\cco\spr(\ell_{\log}\circ\cH) \subseteq \cco\spr(\ell_{01}\circ\cH)$ of
Corollary~\ref{corollary:no-kernel-gdpi-characterization}, and it holds only after convexification.
%\textbf{Alt text:} 
}
\label{fig:wu-convexity}
\end{figure}
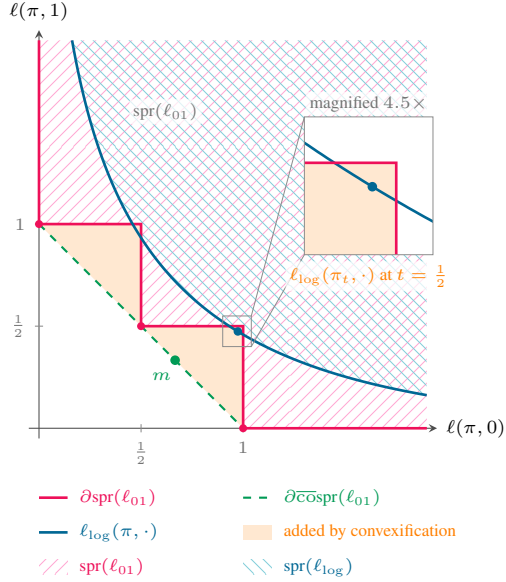

\begin{proof}
By Definition~\ref{def:Constrained Superprediction Set with Respect to model class} it suffices to show
$\ell_{\log}\circ\cH \subseteq \cco\spr(\ell_{01}\circ\cH)$. Fix
$h_\alpha \in \cH$ and let $h^{(s)} \coloneqq h_{\alpha + s e_0}$ be the same model
with intercept shifted by $s$, so that $w_{\alpha + s e_0} = w_\alpha + s$ and $h^{(s)} \in \cH$. Put $c \coloneqq \tfrac34$ and,
\begin{align*}
    m \coloneqq \tfrac13\bigl( \ell_{01}\circ h^{(-c)} + \ell_{01}\circ h^{(0)} + \ell_{01}\circ h^{(c)} \bigr) \in \operatorname{co}(\ell_{01}\circ\cH).
\end{align*}
We show,
\begin{align}
    \label{eq:wu-01-log-proof}
    m \le \ell_{\log}\circ h_\alpha,
\end{align}
pointwise on $\cX\times\cY$, which
gives $\ell_{\log}\circ h_\alpha \in \spr(\operatorname{co}(\ell_{01}\circ\cH))
\subseteq \cco\spr(\ell_{01}\circ\cH)$.

For simplifying notation we introduce $t \coloneqq w_\alpha(x)$. The right hand-side of Eq.~\eqref{eq:wu-01-log-proof} can be written as $\ell_{\log}\circ h_\alpha(x,1) = L(t)$, respectively
$\ell_{\log}\circ h_\alpha(x,0) = L(-t)$, where $L(t) \coloneqq \log(1+e^{-t})$.
The left hand-side of Eq.~\eqref{eq:wu-01-log-proof} can as well be simplified. To see that,
\begin{align*}
    \ell_{01}\circ h^{(s)}(x,1) &= 
    \begin{cases}
        1, & \sigma(t + s) < \tfrac12, \\
        \tfrac12, & \sigma(t + s) = \tfrac12, \\
        0, & \sigma(t + s) > \tfrac12, \\
    \end{cases}\\
    &= \begin{cases}
        1, & t + s < 0, \\
        \tfrac12, & t + s = 0, \\
        0, & t + s > 0. \\
    \end{cases}
\end{align*}
Hence,
\begin{align*}
    m(x,1) = \begin{cases}
        1, & t < -c, \\
        \tfrac23, & -c < t < 0, \\
        \tfrac13, & 0 < t < c, \\
        0, & t > c,
    \end{cases}
    \qquad\text{and}\qquad m(x,0) = 1 - m(x,1) ,
\end{align*}
with the values $\tfrac56, \tfrac12, \tfrac16$ at $t \in \{-c, 0, c\}$, which
lie between the neighbouring ones and are therefore covered by the estimates
below. As $L$ is decreasing it suffices to check $m(x,1) \le L(t)$ at the right
end of each interval:
\begin{align*}
    t \le -c:\ 1 \le L(-c) = \log(1+e^{3/4}) = 1.137\ldots, &\qquad
    t \in (-c,0]:\ \tfrac23 \le L(0) = \log 2 = 0.693\ldots, \\
    t \in (0,c]:\ \tfrac13 \le L(c) = \log(1+e^{-3/4}) = 0.386\ldots, &\qquad
    t > c:\ 0 \le L(t) .
\end{align*}
The case $y = 0$ follows by the symmetry $m(x,0) \le L(-t)$ under
$t \mapsto -t$. Hence $m \le \ell_{\log}\circ h_\alpha$.

The additional statement in the example of a strict inequality at $\phi_{(1,0,1,0)}$
follows from Eq.~\eqref{eq:wu-BRlog} and because
%\begin{align*}
    $\BR_{\ell_{01}\circ\cH}(\phi_{(1,0,1,0)}) = \tfrac14.$
%\end{align*}
By Lemma~\ref{lemma:wu-param}, classifying all four atoms correctly would require
$w_1 > 0$, $w_3 > 0$, $w_2 < 0$, $w_4 < 0$, hence $w_1 + w_3 > 0 > w_2 + w_4$,
contradicting $w \in V$; so at least one atom, of mass $\tfrac14$, is
misclassified. The value $\tfrac14$ is attained: $\alpha = (1,1,-\tfrac12)$
gives $w = (\tfrac32, -\tfrac12, \tfrac12, \tfrac52) \in V$, which errs on
$x^{(4)}$ only.
\end{proof}

\begin{remark}[Why convexity matters]
\label{rem:wu-convexity}
The set containment above genuinely needs the convex hull (for an illustration see Figure~\ref{fig:wu-convexity}). The
stronger conditions in the hierarchy of Figure~\ref{fig:implications for set containment condition} fail. Take $h_\alpha$
with $w_\alpha(x) = x_1$ and set $\delta := -\log(e^{1/2}-1) = 0.432\ldots$, so
that $\tfrac12 \le L(\pm t) < 1$ for $|t| < \delta$. Suppose
$\ell_{01}\circ h_\beta \le \ell_{\log}\circ h_\alpha$ for some
$h_\beta \in \cH$. At any $x$ with $|x_1| < \delta$ both values
$L(\pm x_1)$ are $< 1$, so $h_\beta$ may err neither at $(x,0)$ nor at
$(x,1)$; this forces $w_\beta(x) \ge 0$ and $w_\beta(x) \le 0$, \ie,
$w_\beta(x) = 0$. An affine function vanishing on the open strip
$\{|x_1| < \delta\}$ vanishes identically, so $\ell_{01}\circ h_\beta \equiv
\tfrac12$; but $\ell_{\log}\circ h_\alpha(x,1) = L(x_1) < \tfrac12$ for
$x_1 > \delta$, a contradiction. Hence
%\begin{align*}
    $\ell_{\log}\circ\cH \not\subseteq \spr(\ell_{01}\circ\cH),$
    %\qquad\text{%
    while%
    %}\qquad
    $\ell_{\log}\circ\cH \subseteq \cco\spr(\ell_{01}\circ\cH) .$
%\end{align*}
Pointwise domination of the losses fails for the same reason:
$\ell_{01} > \ell_{\log}$ whenever the model errs but only barely, \ie, when
$\pi_y \in (e^{-1}, \tfrac12)$. Only after averaging over the three shifted
classifiers is the discrete loss dominated. This is the gap between the top and
the bottom node of Figure~\ref{fig:implications for set containment condition}, in a concrete learning problem.
\end{remark}

\section{Duality Between Functions and Measures}
\label{app:duality function and measures}

\begin{definition}[Dual pair, Definition 5.90 in \citet{aliprantis2006infinite}]
    A dual pair is a pair of vector spaces $\cZ, \cZ'$ together with a bilinear functional $(z, z') \to \langle z, z' \rangle$ from $\cZ \tm \cZ'$ to $\reals$ %that separates the points of $\cZ$ and $\cZ'$. That is:
    such that:
    \begin{enumerate}
        \item $z \to \langle z, z' \rangle$ is linear for each $z' \in \cZ'$;
        \item $z' \to \langle z, z' \rangle$ is linear for each $z \in \cZ$;
        \item If $\langle z, z' \rangle = 0 \; \forall z' \in \cZ'$, then $z=0$;
        \item If $\langle z, z' \rangle = 0 \; \forall z \in \cZ$, then $z'=0$.
    \end{enumerate}
\end{definition}

Each component space involved in a dual pair can be interpreted as a set of linear functionals on the other.
Conditions 1 and 2 are the ones required for the definition of a \emph{bilinear functional}. The bilinear functional $(z, z') \mapsto \langle z, z' \rangle$ is also called the \emph{duality} of the dual pair.

An example of dual pair in a finite dimensional setting is a vector space $\cX$ and its algebraic dual $\cX^*$, and $\langle x, x^* \rangle = x^*(x)$. In the infinite dimensional setting, the algebraic dual and the topological dual (the one we call here simply as \emph{dual}) do not generally coincide.

The set of functions $f \colon \cZ \to \reals$ that are \emph{bounded}, \ie, $\exists C\in \reals_{\ge0}$ \st $|f(z)| \le C, \forall \, z \in \cZ$, and \emph{$\Sigma(\cZ)$-measurable} is denoted as $\Bb(\cZ)$. This set together with the supremum norm, 
$$\| f\|_\infty = \sup_{z \in \cZ} |f(z)|$$ 
for $f \in \Bb(\cZ)$ forms a Banach space \cite[page 804]{schechter1997handbook}.

The norm dual, \ie, the \emph{set of all continuous, linear functionals} on $\Bb(\cZ)$ coincides with $\ba(\cZ)$, the space of all finitely additive, signed measures $\mu\colon (\cZ, \Sigma(\cZ)) \to \reals$ with bounded total variation $\|\mu\|_\infty$ \cite[Theorem 14.4]{aliprantis2006infinite}. That is defined as
\begin{align}
\| \mu\|_\infty \coloneqq \sup_{\cA \in \Sigma(\cZ)} |\mu(\cA)|,
\end{align}
analogously to the supremum norm.
The space $\ba(\cZ)$ normed with $\| \mu\|_\infty$ forms a Banach space \cite[29.6.c \& 29.29.f]{schechter1997handbook}.

\begin{proposition}[Theorem 14.4 in \citet{aliprantis2006infinite}]
    Every continuous linear functional on $F: \Bb(\cZ) \to \reals$ can be written as
    $F: f \mapsto \int f d\mu$
    for a unique $\mu \in \ba(\cZ)$.
\end{proposition}

\begin{proposition}[Theorem 11.8 in \citet{aliprantis2006infinite} and page 804 of \citet{schechter1997handbook}]\label{prop:fin-add-integral}
    The integral $\int f d\mu$ is well-defined, furthermore, it defines a continuous bilinear functional of the sets $\Bb(\cZ) \tm \ba(\cZ) \to \reals$, as
    \begin{align}
    \langle f , \mu \rangle \coloneqq \int f d\mu\,.
    \end{align}
\end{proposition}

It follows that $\langle \Bb(\cZ), \ba(\cZ) \rangle$ forms a \emph{dual pair} by  \cite[28.4 HB22]{schechter1997handbook}, which applies here as the two spaces $\Bb(\cZ), \ba(\cZ)$ are Banach spaces and therefore fulfill the hypothesis by \cite[16.1 \& 26.1]{schechter1997handbook}.

\subsection{Relevant Topologies and Sets} 

When working with dual pairs, it is standard and helpful to consider a weak re-topologization of the sets $\Bb$ and $\ba$ \cite[Chapter 15]{aliprantis2006infinite}. We define $\sBbba$ as the topology on $\Bb(\cZ)$ which makes every functional $f \mapsto \langle f, \mu\rangle$ (for some $\mu \in \ba(\cZ)$) continuous. Analogously, $\sbaBb$ is the topology on $\ba(\cZ)$ which makes every functional $\mu \mapsto \langle f, \mu\rangle$ (for some $f \in\Bb(\cZ) $) continuous. The topologies are called weak as they are contained within the norm-topologies of the respective spaces \cite[28.13.b \& 28.22.a]{schechter1997handbook}. %Both topologies are locally convex \cite[28.13.a \& 28.15.a]{schechter1997handbook}.

Every closed convex subset of $\Bb$ in the norm topology is $\sBbba$-closed, and vice-versa \cite[28.14 HB24]{schechter1997handbook}.

We introduce the notation $\ca(\cZ)$ to denote the set of all countably additive, signed measures  with bounded total variation. Notice that we can give an alternative definition of the set of probability measures as,
\begin{align}
    \Prob(\cZ) \coloneqq \{\phi \in \ca(\cZ)\colon \phi \ge 0, |\phi(\cZ)| = 1 \},
\end{align}
and similarly for finitely additive probability measures,
\begin{align}
    \finProb(\cZ) \coloneqq \{\phi \in \ba(\cZ) \colon \phi \ge 0, |\phi(\cZ)| = 1 \}.
\end{align}
For $\phi \in \finProb(\cZ)$ we denote the expectation of a function $f \in \Bb(\cZ)$ as,
\begin{align}
    \mathbb{E}_{\phi}[f] \coloneqq \langle f , \phi \rangle.
\end{align}

The following property seems to be a known fact, as it is referenced, without proof, in \citet[Appendix D]{walley1991statistical} and \citet{duanmu2018finitely}. As we have not found any proof, we provide one here.
\begin{lemma}[Closure of Countably Additive Probabilities are Finitely Additive Probabilities]
\label{lemma:closure of countably additive probabilities are finitely additive probabilities}
    Let $\cZ$ be a Polish space. Then, with respect to the $\sBbba$-topology, $\overline{\Prob(\cZ)} = \finProb(\cZ)$.
\end{lemma}
\begin{proof}
 Let $\ext\cZ$ be the set of extreme points of a set. The set $\finProb(\cZ)$ is compact as it is a closed subset \cite[Appendix D4]{walley1991statistical} of the compact norm-ball. The norm is the supremum norm and the ball is compact by Banach-Alaoglu-Bourbaki-Theorem \cite[28.29.UF28]{schechter1997handbook}. By the Krein-Milman-Theorem \cite[7.68]{aliprantis2006infinite} 
 %the closed convex hull of the extreme points of a compact convex set in locally convex Hausdorff topological space is equal to the compact convex set. 
and \cite[16.1 \& 26.1]{schechter1997handbook} we know that $\finProb(\cZ) = \overline{\co} \ext \finProb(\cZ)$. The extreme points of the $\finProb(\cZ)$, \ie, $\ext \finProb(\cZ)$, is equal to the set of all Dirac-measures \cite[Theorem 15.9]{aliprantis2006infinite}. The convex closure of those extreme points $\co \ext \finProb(\cZ)$ is a subset of the countably additive probability measures $\Prob(\cZ)$, as all Dirac measures are countably additive, and the set $\Prob(\cZ)$ is convex. Thus, $\overline{\Prob(\cZ)} = \finProb(\cZ)$.
\end{proof}

\section{Remarks on Constrained Conditional Risk}
\label{app:remarks on constrained conditional risk}
\begin{definition}[Conditional Bayes Risk] \label{def:conditional bayes risk}
    Given a measurable loss $\ell \colon \Prob(\cY)\tm \cY \to \reals_{\ge 0}$ and a label distribution $\phi \in \Prob(\cY)$, the conditional Bayes risk is the quantity
    \begin{align}
        \cbr_{\ell}(\phi) \coloneqq \inf_{\psi \in \Prob(\cY)} \mathbb{E}_{Y \sim \phi}[\ell(\psi, Y)].
    \end{align}
\end{definition}
In the literature, this function is called ``conditional'' or ``prior'' risk depending on which probability it takes in input -- a conditional or a prior one. Here we abuse the conditional Bayes risk notation to indicate both functionals, as they are clearly distinguishable when the domain is specified.

\begin{remark}[Relation between unconstrained $\br$ and $\cbr$]
    \label{remark: br and cbr}
    Let $\cY$ be finite and $\cH = \cM(\cX, \cY)$. Consider a loss function $\ell\in\cL(\cY)$ and a Markov kernel $F\in \cM(\cX,\cY)$. We have,
    % Let us define the conditional Bayes risk,
    % \begin{align}
    %     \label{eq:conditional bayes risk}
    %     \cbr_{\ell}(\psi) \coloneqq \inf_{p \in \Prob(\cY)} \mathbb{E}_{Y \sim \psi}[\ell(p, Y)]
    % \end{align}
    % for $\psi \in \Prob(\cY)$.
    %Then, Proposition~\ref{prop:support function of superprediction set is bayes risk} recovers the unconditional Bayes risk, for $\phi \in \Prob(\XY)$,
    \begin{align}
        \br_{\ell}(\phi) &= \inf_{h \in \cM(\cX, \cY)} \mathbb{E}_{(X,Y)\sim\phi}[\ell \circ h]\\
        &= \inf_{h \in \cM(\cX, \cY)} \mathbb{E}_{X\sim\pi_{\cX}}[\mathbb{E}_{Y\sim\mt{F}(X)}[\ell \circ h]]\\
        &\overset{(\star)}{=} \mathbb{E}_{X\sim\pi_{\cX}}[\inf_{\psi \in \Prob(\cY)} \mathbb{E}_{\mt{F}(X)}[\ell(\psi, Y)]]\\
        &\overset{D\ref{def:conditional bayes risk}}{=} \mathbb{E}_{X\sim\pi_{\cX}}[\cbr(\mt{F}(X))]\,,
    \end{align}
    where $\pi_{\cX} \in \Prob(\cX)$ and \st the data-generating distribution fulfills the condition $\phi = \pi_{\cX} \tm F$. The equality $(\star)$ holds by \cite[Theorem 14.60]{rockafellar1998variational}. For constrained model classes, the theorem---and therefore this remark---may not hold; they should satisfy a specific property called \emph{decomposability} \cite[Definition 14.59]{rockafellar1998variational}.\footnote{We do not explore this here, as we do not use this result in the constrained case. The interested reader can refer to \cite{giner2009set,chancelier2022minimization,bui2023interchange} for examples of conditions \st this property holds.}
\end{remark}

\section{Remarks on Finitely Additive Probability Measures and Risk}
\label{app:Remarks on Finitely Additive Probability Measures and Risk}
For a complete treatment of finitely additive measures, we refer the reader to \cite{rao1983theory}.
In general, many nice results that hold for countably additive measures are not proved (at least, not in their exact form) for the finitely additive counterpart. Notable examples are the Radon-Nikodym theorem \cite[Chapter 6]{rao1983theory} and the disintegration theorem. We report here some properties useful in the context of learning problems.

\begin{definition}[Risk and Bayes Risk Operators for Finitely Additive Measures]
\label{def:br-Bb-finProb}
    Consider a function class $\cF \in \Bb(\cZ)_{\ge0}$, a function $f \in \cF$, and a probability distribution $\phi \in \finProb(\cZ)$. Then, the \emph{Bayes risk} and \emph{risk operators} are defined as
    \begin{align}
        \br_{\cF}(\phi) \coloneqq  \inf_{f \in \cF } \risk_{ \phi }(f)\,, \quad
        \risk_{ \phi }(f) \coloneqq \int f d\phi = \langle f, \phi \rangle \;. 
    \end{align}
\end{definition}

\begin{proposition}\label{prop:br-finProb-well-def}
    Definition~\ref{def:br-Bb-finProb} is well-defined and extends Definition~\ref{def:br-risk}.
\end{proposition}
\begin{proof}
    The well-posedness of risk and its infimum directly follows from the well-posedness of the integral \wrt a finitely additive signed measure, proved by Proposition~\ref{prop:fin-add-integral}. In particular, as $\ell \circ \cH \subseteq \Bb(\XY)_{\ge0}$ and $\Prob(\XY) \subset \finProb(\XY)$, the extension also follows.
\end{proof}

In this extended definition of risk, we allow the generating probability measure to be a finitely additive one. However, we keep the standard definition of loss, and therefore only admit a model class that is a set of Markov kernels, by definition countably additive.
In this way, we can make use of the dual pairing described in Appendix~\ref{app:duality function and measures}, while still being allowed to use a standard Bayesian framework for inference.

\begin{remark}
    Consider a Markov kernel $\ka \in \cM(\cZ_1, \cZ_2)$ and its associated operator as per Proposition~2.4. The operator $\amto\ka$ is defined via the dual pairing relating $\Bb$ and $\ba$. For this reason, we have 
    $$\amto{\ka} \colon \ba(\cZ_1) \to \ba(\cZ_2)$$
    and that $\amto\ka\ba(\cZ_1)  \subset \ba(\cZ_2)$. It follows that $\amto\ka\finProb(\cZ_1)  \subset \finProb(\cZ_2)$. This is a weaker statement than the one holding for countably additive probabilities, which enjoy the disintegration theorem and for which it can therefore be proved that $\amto\ka\Prob(\cZ_1) = \Prob(\cZ_2)$. This property is formally proved in \cite{iacovissi2026corruptions} when presenting properties of their taxonomy.
\end{remark}

\begin{lemma}[Lemma~\ref{lemma:kernel-is-adjoint-to-risk}]
Let $(\XY, \Sigma(\cX) \tm \Sigma(\cY)$ be standard Borel. Consider a  loss $\ell \in \cL(\cY)$, a model $h \in \cM (\cX,\cY)$, and a probability distribution $\phi\in \Prob(\XY)$. Let $\ka \in \cM (\XY, \XY)$. Then, 
%\begin{align}
   $ \risk_\phi [ \mto\ka (\ell \circ h )] =  \risk_{\amto\ka \phi}  [\ell \circ h ].$
%\end{align}
\end{lemma}
\begin{proof}
    We know that $\risk_{\phi}[f] \coloneqq \int f \, d\phi$ is well defined (Proposition~\ref{prop:br-finProb-well-def}). Hence, the lemma holds by applying the adjoint property of kernels, \ie,
    %\begin{align}
        $\langle f, \amto{\ka}\mu\rangle_{\cZ_2} = \langle \mto{\ka}f, \mu\rangle_{\cZ_1}  \ \forall f \in \Bb(\cZ_2), \mu \in \ba(\cZ_1) $ \citep[Theorem 19.9, point 2]{aliprantis2006infinite}.
    %\end{align}
\end{proof}

\section{From Unconstrained Superprediction Set to Constrained Superprediction Set}
\label{app:unconstrained to constrained superprediction set}

\begin{figure}
    \centering
    \includegraphics[width=0.9\linewidth]{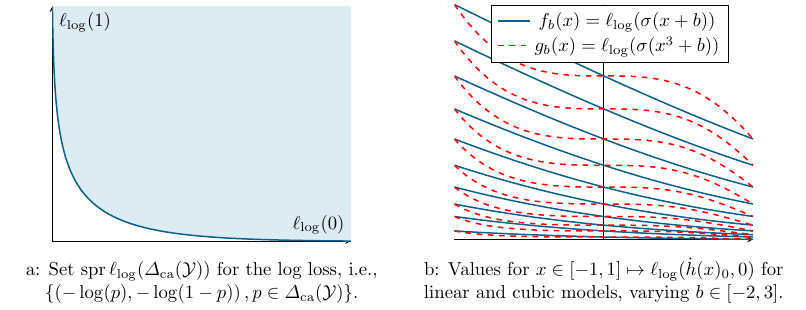}
    \caption{
    \textbf{Caption: Constrained superprediction set is not Cartesian product of unconstrained ones.} The left panel shows that, no matter what is the model class, the superprediction set of the loss (in this case for logarithmic loss $\ell_{\log}$) will look the same as long as the model class covers the whole simplex, \ie, $\Prob(\cY) = \bigcup_{x\in \cX}\{ \mt{h}(x) , h \in \cH \}$. However, in the right panel we see that the single components of the elements of the constrained superprediction set, \ie, $\big[\ell_{\log}(\mt{h}(x)_y, y)\big]_{x,y}$, take different values for different models. For linear model $\mt{h}_{\text{lin}}(x)_0 \coloneqq \sigma(x+b)$, with sigmoid function $\sigma$, the induced losses are $(\ell_{\log}\circ h_{\text{lin}})(x,y) = \big[-\log\big(\sigma(x+b)\big),-\log\big(1-\sigma(x+b)\big)\big]$; for cubic models, \ie, $\mt{h}_{\text{cub}}(x)_0 \coloneqq \sigma(x^3+b)$, the induced losses are $(\ell_{\log}\circ h_{\text{cub}})(x,y) = \big[-\log\big(\sigma(x^3+b)\big),-\log\big(1-\sigma(x^3+b)\big)\big]$. Hence, the corresponding constrained superprediction set will look different for each of these model classes.
    %\textbf{Alt text:} Two panels, labeled (a) and (b). Panel (a) depicts the positive quadrant of the Cartesian plane. The area above the log-loss curve is shaded in blue. Panel (b) shows overlaying loss-hypothesis class function values with x varying in a limited range. There is a family of blue functions which correspond to a linear predictor f and a family of red functions which correspond to a cubic predictor g.
    }
    \label{fig:constrained-spr-log}
\end{figure}

Classically, the (unconstrained) superprediction set is defined with respect to the loss function alone (\cf \cite{kalnishkan2002mixability, Kalnishkan2002absence}), \ie, 
\begin{align}
    \label{eq:unconstrained spr}
    \spr(\ell) = \spr(\ell \circ \Delta_{\ca}(\cY)) \coloneqq \bigcup_{\phi \in \Delta_{\ca}(\cY)} \big\{ u \in \Bb(\cY) \mid \ell(\phi) \le u \big\},
\end{align}
as illustrated in Figure~\ref{fig:constrained-spr-log} (a). One could imagine a different construction for a superprediction set for a loss function $\ell$ and a model class $\cH$ which goes as follows. First, note that we can define for every $x \in \cX$, $|\cX|<\infty$, $\cH_x \coloneqq \{ \mt{h}(x) \in \Prob(\cY) \mid h \in \cH_x\}$. We essentially restrict the unconstrained superprediction set \eqref{eq:unconstrained spr} to,
\begin{align}
    \spr(\ell \circ \cH_x) \coloneqq \bigcup_{\phi \in \cH_x} \big\{ u \in \Bb(\cY) \mid \ell(\phi) \le u \big\} \subseteq \spr(\ell)\,.
\end{align}
Naively, one might then take the Cartesian product over $x \in \cX$. However, in general,
\begin{align}
    \spr(\ell \circ \cH) \subset \bigtimes_{x \in \cX} \spr(\ell \circ \cH_x)\,,
\end{align}
where the subsethood can be strict $(\subsetneq)$. This is the case whenever the model class is not expressive enough to independently assign predictions to different $x \in \cX$. For instance, when $\mt{h}(x_1) = \phi \in \Prob(\cY)$ implies that $\mt{h}(x_2) = \psi (\cY)$ for every $h \in \cH$ and $\phi \neq \psi$, then the prediction assignment to $x_1$ interacts with the prediction assignment to $x_2$. Figure~\ref{fig:constrained-spr-log} illustrates the difference.

\section{Properties of Support Functions and Superprediction Sets}
\label{app:properties supp func and superpred}

We denote the \emph{Minkowski-sum} as $\oplus$, defined as: let $\cZ, \cZ'$ be subsets of a vector space, then $\cZ \oplus \cZ' \coloneqq \{z + z' \mid z \in \cZ, z' \in \cZ' \}$.
We use the notation $\alpha \cdot \cZ$ or $\alpha \cZ$ to indicate the scaled set $\{\alpha z \,\mid\, z \in \cZ\}$. When $\alpha = -1$, we simply write $-\cZ$.

\begin{lemma}[Properties of Concave Support Functions]
    \label{lemma: support function properties}
    Let $\cZ$ be a Polish space and $\cA, \cB \subseteq \Bb(\cZ)$ be non-empty.
    The following statements hold.
    \begin{enumerate}[label=(P\arabic*), ref=P\arabic*]
        \item \label{supp - conjugate} $\sigma_\cA = -\rho_{-\cA}$.
        \item \label{supp - upper semi-continuity} $\rho_\cA$ is \emph{upper semi-continuous} with respect to the $\sBbba$-topology.
        \item \label{supp - sublinearity} $\rho_\cA$ is \emph{sublinear}: that is, $\rho_{\cA}(\alpha \mu) = \rho_\cA(\mu)\ \forall \alpha >0$ \emph{(positive homogeneity)}, and $\rho_{\cA}(\mu + \nu) \ge \rho_\cA(\mu) + \rho_\cA(\nu)$ \emph{(superadditivity)} and all $\mu,\nu \in \ba(\cZ)$.
        \item \label{supp - invariance to closed convex hull} $\rho_{\cA} = \rho_{\overline{\co}(\cA)}$.
        \item \label{supp - minkowski sum} $\rho_{\cA \oplus \cB} = \rho_\cA + \rho_\cB$.
    \end{enumerate}
    Analogous properties hold for non-empty sets $\mc C,\mc D \subseteq \ba(\cZ)$.
\end{lemma}
\begin{proof}
    \begin{enumerate}
        \item[]
        \item Trivially follows by the linearity of the dual pairing and properties of $\sup$ and $\inf$ functions.
        
        \item The pointwise infimum over a family of $\sbaBb$-continuous functions is upper semicontinuous \cite[Lemma 2.41]{aliprantis2006infinite}. 
        
        \item Positive homogeneity is a direct consequence of the definition and the linearity of the bilinear mapping. Furthermore, the pointwise infimum over linear functions is concave ((\ref{supp - conjugate}) and \citet[Lemma 5.40]{aliprantis2006infinite}) and positive homogeneity with concavity implies superadditivity ((\ref{supp - conjugate}) and \citet[Definition 5.45]{aliprantis2006infinite}).
        
        \item Follows from (\ref{supp - upper semi-continuity}) and  (\ref{supp - sublinearity}) together with \citet[Theorem 7.51]{aliprantis2006infinite}.
        
        \item Follows from  (\ref{supp - conjugate}) and Lemma 7.54 (2) in \citet{aliprantis2006infinite}.
    \end{enumerate}
\end{proof}

\begin{lemma}
    \label{lemma:recession cone spr}
    Let $\cF \subseteq \Bb(\cZ)_{\ge 0}$. The associated superprediction set can be written as
    $$\spr(\cF) = \cF \oplus \Bb(\cZ)_{\ge 0}\,.$$
\end{lemma}
\begin{proof}
    We first show the set inclusion from left to right. Let $f \in \spr(\cF)$, \ie, there exists $a \in \cF$ such that $a \le f$. Define $g \coloneqq f - a \in \Bb(\cZ)_{\ge 0}$. Hence, clearly $f = a + g \in \cF \oplus \Bb(\cZ)_{\ge 0}$. For the reverse direction, note that any $f \in \cF \oplus \Bb(\cZ)_{\ge 0}$ can be written as $f = a + g$ for some $a \in \cF$ and $g \in \Bb(\cZ)_{\ge 0}$, hence $f \ge a$.
\end{proof}

\begin{proposition}[Support Function of Superprediction Set is Bayes Risk, Proposition~2.14]
    Let $\ell \in \cL(\cY)$ be loss function and $\cH \subseteq \cM (\cX, \cY)$ a model class. It holds,
    \begin{align}
        \rho_{\spr (\ell \circ \cH) }(\phi) = \br_{\ell \circ \cH}(\phi)\,, \enspace \phi \in \finProb(\XY) \,.
    \end{align}
\end{proposition}
\begin{proof}
    Let $\phi \in \finProb(\XY)$,
    \begin{align}
        \br_{\ell \circ \cH}(\phi) &= \inf_{h \in \cH}\mathbb{E}_{\phi}[\ell \circ h] = \inf_{h \in \cH}\langle \ell \circ h, \phi \rangle = \inf_{f \in \ell \circ \cH}\langle f, \phi \rangle \\
        &\overset{(\star)}{=} \inf_{f \in \ell \circ \cH}\langle f, \phi \rangle + \inf_{f \in \Bb(\XY)_{\ge 0}} \langle f, \phi \rangle = \inf_{f \in (\ell \circ \cH) \oplus \Bb(\XY)_{\ge 0}}\langle f, \phi \rangle\\
        &\overset{L.\ref{lemma:recession cone spr}}{=} \inf_{f \in \spr(\ell \circ \cH)}\langle f, \phi \rangle  = \rho_{\spr(\ell \circ \cH)}(\phi)\,.
    \end{align}
    Where the equality $(\star)$ is true as, for $f \in \Bb(\XY)_{\ge 0}$ and $\phi \in \finProb(\cY)$, we have $\langle f, \phi \rangle \ge 0$. Hence, $\inf_{f \in \Bb(\XY)_{\ge 0}} \langle f, \phi \rangle = \langle 0, \phi \rangle = 0$.
\end{proof}

\section{Proof of Proposition~\ref{prop:BRIkerns is convex}} \label{app:convexity of BRIkerns}

\begin{lemma}%[Convex Combination of Kernels on Superprediction Set]
\label{lemma:Convex Combination of Kernels on Superprediction Set}
    Let $\ka_1, \ka_2 \in \cM(\XY, \XY)$ and $\alpha \in (0,1)$. Let $\ell \in \cL(\cY)$ be a loss function and $\cH \subseteq \cM (\cX, \cY)$ a model class. For $\ka \coloneqq \alpha \ka_1 + (1-\alpha)\ka_2$,
    \begin{align}
        \spr(\mto{\ka} (\ell \circ \cH) ) \subseteq \alpha \spr(\mto{\ka}_1 (\ell \circ \cH) ) + (1-\alpha) \spr( \mto{\ka}_2 (\ell \circ \cH) ).
    \end{align}
\end{lemma}
\begin{proof}
    Observe that, by definition, $\spr(\mto{\ka} (\ell \circ \cH)) = \spr(\alpha \mto{\ka}_1 + (1-\alpha) \mto{\ka}_2 (\ell \circ \cH)).$
    Let $u \in \spr(\alpha \mto{\ka}_1 + (1-\alpha) \mto{\ka}_1 (\ell \circ \cH))$, then there exists $h_u \in \cH$ such that
    \begin{align}
        (\alpha \mto{\ka}_1 + (1-\alpha) \mto{\ka}_2 )(\ell \circ h_u) = \alpha \mto{\ka}_1 (\ell \circ h_u) + (1-\alpha) \mto{\ka}_2 (\ell \circ h_u) \le u,
    \end{align}
    where the equality holds by definition of superprediction set Definition~2.11.
    Let $x_u \coloneqq \alpha \mto{\ka}_1 (\ell \circ h_u) + (1-\alpha) \mto{\ka}_2 (\ell \circ h_u)$ be the point bounding $u$ from below.
    We can rewrite, for $\alpha \in (0,1]$,
    \begin{align}
        u &= \alpha \mto{\ka}_1 (\ell \circ h_u) + (u - x_u) + (1-\alpha) \mto{\ka}_2 (\ell \circ h_u) = \alpha \left( \mto{\ka}_1 (\ell \circ h_u) + \frac{1}{\alpha} (u - x_u)\right) + (1-\alpha) \mto{\ka}_2 (\ell \circ h_u),
    \end{align}
    where
    $\mto{\ka}_1 (\ell \circ h_u) + \frac{1}{\alpha} (u - x_u) $
    is a point of the set $\spr(\mto{\ka}_1 (\ell \circ \cH))$, because $u - x_u \ge 0$ and $\alpha \in (0,1)$. Since also $\mto{\ka}_1 (\ell \circ h_u) \in \spr(\mto{\ka}_2 (\ell \circ \cH))$, we get $u \in \alpha \, \spr(\mto{\ka}_1 (\ell \circ \cH)) + (1-\alpha) \, \spr(\mto{\ka}_1 (\ell \circ \cH))$.
\end{proof}

\begin{proposition}[$\BRIkerns$ is a Convex Set, Proposition~\ref{prop:BRIkerns is convex}]
    Let $\ell\in\cL(\cY)$ be a loss function and $\cH \subseteq \cM (\cX, \cY)$ a model class. The associated set $\BRIkerns(\ell \circ \cH)$ is convex.
\end{proposition}
\begin{proof}
    Let $\ka_1, \ka_2 \in \BRIkerns(\ell \circ \cH)$ and $\alpha \in [0,1]$. We show that $\ka \coloneqq \alpha \ka_1 + (1-\alpha)\ka_2 \in \BRIkerns$, as convexity then follows by a trivial inductive argument. By definition of $\BRIkerns(\ell \circ \cH)$, it holds that
    \begin{align}
        \label{eq:superpred containment in lemma sufficient conditions for two attribute corruption kernels}
        \mto{\ka}_i(\ell \circ \cH) \subseteq \cspr(\ell \circ \cH) , \quad i \in \{ 1,2\}.
    \end{align}
    Then, considering the corrupted superprediction set, we can write
    \begin{align}
        \spr(\mto{\ka} (\ell \circ \cH) ) 
        \overset{L.\ref{lemma:Convex Combination of Kernels on Superprediction Set}}{\subseteq}
        \alpha \, \spr(\mto{\ka}_1 (\ell \circ \cH) ) + (1-\alpha) \, \spr( \mto{\ka}_2 (\ell \circ \cH) )
        \overset{(\star)}{\subseteq} 
        \cspr(\ell \circ \cH)\,
    \end{align}
    where $(\star)$ holds by convexity and the set containment in Eq.~\eqref{eq:superpred containment in lemma sufficient conditions for two attribute corruption kernels}.
    Hence, it follows that
        $\spr(\mto{\ka} (\ell \circ \cH) ) \subseteq \cspr(\ell \circ \cH)$, which implies the thesis. %---see Figure~\ref{fig:implications for set containment condition}.
\end{proof}

\section{Proof of Theorem~\ref{theo:support-ineq-characterization-superpredictionset}}
\label{app:characterization proof}

\begin{lemma}[Duality between Closed, Convex Sets and Support Functions]
\label{lemma:Duality between Closed Convex Sets and Support Functions}
    Let $\cZ$ be a Polish space and $\cA \subseteq \Bb(\cZ)$ be closed convex and non-empty. Then,
    \begin{align}
        \cA = \{ f \in \Bb(\cZ) \colon \rho_\cA(\mu) \le \langle f,\mu \rangle, \ \mu \in \ba(\cZ)\}.
    \end{align}
\end{lemma}
\begin{proof}
    This lemma is a consequence of Property~\ref{supp - conjugate} in Lemma~\ref{lemma: support function properties} and Theorem 7.51 in \cite{aliprantis2006infinite}.
\end{proof}

%Hence, a first connection between orderings of support functions and sets can be immediately proved. 

\begin{lemma}[Equivalence between Set and Support Function Orderings]
    \label{lemma:support-ineq-characterization}
    Consider two closed convex and non-empty sets $\cF, \cG \subseteq \Bb(\XY)$. Then, 
    $$\cG \subseteq \cF \quad \Leftrightarrow \quad \rho_{\cF}(\mu) \le \rho_{\cG}(\mu) \quad \forall \mu \in \ba(\XY) \,, $$
assuming that $-\infty \le -\infty$.
\end{lemma}
\begin{proof}
    That set containment implies the inequality follows directly. The reverse direction is a consequence of Lemma~\ref{lemma:Duality between Closed Convex Sets and Support Functions}.
\end{proof}

Once the technical setup has been laid out, one can also recall another property of the support function: the invariance to convex closure of the set, such that $\rho_{\spr \, \cF} = \rho_{\cspr \, \cF}$ (Lemma~\ref{lemma: support function properties} (\ref{supp - invariance to closed convex hull})). 
In addition, we also aim to ultimately characterize the behavior of learning problems, which involve probability measures $\finProb$ instead of general, signed measures $\ba$. For this reason, it makes sense to prove the following key result in terms of convex closures of sets and the set of finitely additive probability measure. What it proves is that we can express the correspondence between support functions \emph{of superprediction sets} and closed convex sets in terms of a restricted class of measures, the finitely additive probability distributions.

\begin{lemma}
    \label{lemma:co-spr-definition-with-fin-probabilities}
    Let $\cF \subseteq \Bb(\XY)_{\ge 0}$ non-empty. Then, 
    \begin{align}
    %\label{eq:superprediction set as dominated by linear mappings w/o kernel}
        \cspr \, \cF = \{\, g \in \Bb(\XY) \,\mid\, \rho_{\spr \, \cF}(\phi) \le \langle g, \phi \rangle, \enspace \forall \phi \in \finProb(\XY) \,\} \,.
    \end{align}
\end{lemma}
\begin{proof}
    We start by considering the well-known relationship between sets and support function formulated in Lemma~\ref{lemma:Duality between Closed Convex Sets and Support Functions}, which for $\cspr \, \cF$ amounts to
    \begin{align} 
        \cspr \, \cF = \{ g \in \Bb(\XY) \,\mid\, \rho_{\cspr \, \cF}(\mu) \le \langle g,\mu \rangle, \ \forall \mu \in \ba(\XY)\}.
    \end{align}
    Now notice that every bounded signed measure $\mu \in \ba(\XY)$ can be written as $\alpha \mu' = \mu$ for some $\mu' \in \bb{B}_1 \coloneqq \{ \nu \in \ba(\XY) \colon \|\nu\|_\infty = 1\}$ and $\alpha \in \reals_{\ge 0}$, \eg if $\mu \neq 0$, by choosing $\alpha \coloneqq \| \mu\|_\infty$ and $\mu' \coloneqq \frac{\mu}{\alpha}$. Thus, we can rewrite the set as
    \begin{align}
        \cspr \, \cF &=\{\,g \in  \Bb(\XY) \,\mid\, \rho_{\cspr \, \cF}(\mu) \le \langle g, \mu \rangle, \enspace \forall \mu \in \ba(\XY)\,\}\\
        &= \{\,g \in  \Bb(\XY) \,\mid\, \rho_{\cspr \, \cF}(\alpha \mu) \le \langle g, \alpha \mu \rangle, \enspace \forall \mu \in \bb{B}_1, \alpha \in \reals_{\ge 0}\,\}\\
        &= \{\,g \in  \Bb(\XY) \,\mid\, \rho_{\cspr \, \cF}(\mu) \le \langle g, \mu \rangle, \enspace \forall \mu \in \bb{B}_1\,\},
    \end{align}
    where the last step uses the linearity of the bilinear mapping and Lemma~\ref{lemma: support function properties} (\ref{supp - invariance to closed convex hull}) of support functions. Let us introduce a partition of $\bb{B}_1$ in two sets of positive and signed finitely additive measures with norm 1, such that,
    \begin{gather}
        \cP_{+} \coloneqq \bb{B}_1 \cap \ba(\XY)_{\ge 0},\\ 
        \cP_{-} \coloneqq \bb{B}_1 \setminus \ba(\XY)_{\ge 0},\\
        \cP_{-} \cup \cP_{+} = \bb{B}_1, \cP_{-} \cap \cP_{+} = \emptyset.
    \end{gather}
    Hence it holds that
    \begin{align}\label{eq:spr-with-p-minus}
        \cspr \, \cF = \{\,g \in  \Bb(\XY) \,\mid\, \rho_{\cspr \, \cF}(\mu) \le \langle g, \mu \rangle, \enspace \forall \mu \in \cP_{-} \cup \cP_{+}\,\}.
    \end{align}
    We show now that for every $g \in \Bb(\XY)$, the condition $\rho_{\cspr \, \cF}(\mu) \le \langle g, \mu \rangle$ is satisfied for all $\mu \in \cP_{-}$, and thus we can neglect this set in Eq.~\eqref{eq:spr-with-p-minus}. To this end, let us rewrite the support function as 
    \begin{align}
        \rho_{\cspr \, \cF}(\mu) = \inf_{f \in \cF} \langle f, \mu \rangle + \inf_{f \in \Bb(\XY)_{\ge 0}} \langle f, \mu \rangle,
    \end{align}
    by Lemma~\ref{lemma:recession cone spr} and Lemma~\ref{lemma: support function properties} (\ref{supp - minkowski sum}).
    Observe that for every measure $\mu \in \cP_{-}$ there exists $\cA \in \Sigma(\XY)$ such that $-1 \le \mu(\cA) < 0$.\footnote{Notice that this is not the negative set of $\XY$ \wrt $\mu$, as we are considering the family of finitely additive measures. The Hahn decomposition theorem may not hold \cite[Chapter 2.6]{rao1983theory}.} 
    We can define a subset of functions in $\Bb(\XY)_{\ge0}$ given by the scaled indicator function of this set, \ie, $\{ \alpha \inset_\cA \,\mid\, \alpha \ge 0 \} \subset \Bb(\XY)_{\ge0}$. Hence, we have that
    $$\inf_{f \in \Bb(\XY)_{\ge0}} \langle f, \mu \rangle \le \inf_{\alpha \ge 0} \langle \alpha \inset_\cA, \mu \rangle = \alpha \mu(\cA) \enspace \Rightarrow \inf_{f \in \Bb(\XY)_{\ge 0}} \langle f, \mu \rangle = - \infty,$$
    as $\alpha$ gets arbitrarily large while $\mu$ is bounded from below.
    Since $\inf_{f \in \cF} \langle f, \mu \rangle$ is finite because $\cF \neq \emptyset$, it follows that $\rho_{\cspr \, \cF}(\mu) = -\infty.$
    Hence, we can indeed neglect $\cP_{-}$ in Eq.~\eqref{eq:spr-with-p-minus}, and write
    \begin{align}
        \cspr \, \cF &=\{\,g \in  \Bb(\XY) \,\mid\, \rho_{\cspr \, \cF}(\mu) \le \langle g, \mu \rangle, \enspace \forall \mu \in \bb{B}_1\,\} \\
        &=\{\,g \in  \Bb(\XY) \,\mid\, \rho_{\cspr \, \cF}(\phi) \le \langle g, \phi \rangle, \enspace \forall \phi \in \cP_{+} \,\}.
    \end{align}
    We conclude the proof observing that, by monotonicity of positive measures, $\cP_+ = \finProb(\XY)$.
\end{proof}

The reverse direction of Theorem~4.2, the one left unproved in the main body of the paper, is direct a consequence of Lemma~\ref{lemma:co-spr-definition-with-fin-probabilities}.

\section{Antipolar Sets and Strong GDPI} \label{app: antipolar inequalities}
The GDPI result we have studied in the previous pages is an inequality, giving us information on the existing ordering between clean and corrupted Bayes risk. However, it gives us no information about the \textit{tightness} of the upper bound provided by the corrupted Bayes risk over the clean one. To address this gap in classical settings, a quantitative version of the data processing equality is introduced, called strong data processing inequality \citep{ahlswehde1976spreading}.\footnote{The term ``strong data processing inequality'' was \emph{not} used in the paper \citep{ahlswehde1976spreading}, but occurred for the first time in 2013 in several publications, \eg, by Courtade, Raginsky, Duchi, Jordan and Wainwright.}

Given our generalized theory of data processing, it is natural to ask whether is is possible to introduce and prove properties for a \textit{strong generalized data processing inequality} (SGDPI) using our superprediction set machinery.
Hence, we now proceed to develop a framework based on antipolar sets and functions, which expands what was done until now for the GDPI. 
Leveraging these objects and their properties, we introduce an antipolar data processing inequality and establish some characterization results, both at the level of antipolar sets and at the level of the associated antipolar support functions.
The results are then used in two ways. First, we obtain a characterization of the Strong GDPI for Bayes Risk in terms of an upper bound expressed via an antipolar support function. Second, we show that the Strong GDPI for Bayes Risk coefficient itself admits an exact representation as the value of an antipolar support function evaluated at a certain function.

\subsection{Antipolar Duality}
Key objects for the following sections are antipolar sets and functions.
They are the concave counterparts of the more common polar and gauge functions, widely studied in, \eg, \citet{rockafellar1970convexanalysis,hiriart2004fundamentals,chancelier2024unified}. 
We prove here relevant results for our concave setting.

\begin{definition}[Antipolar Set]
    \label{def:antipolar-set}
    Consider the dual pairing $\langle \cdot, \cdot \rangle \colon \Bb(\cZ) \tm \ba(\cZ) \to \reals$ for some Polish space $\cZ$. Let $\cA \subseteq \Bb(\cZ)$ and $\cB \subseteq \ba(\cZ)$. The associated antipolar sets are defined as,
    \begin{align}
        \cA^\diamond &\coloneqq \{ \mu \in \ba(\cZ) \mid \forall f \in \cA, \langle f, \mu \rangle \ge 1 \}\\
        \cB^\diamond &\coloneqq \{ f \in \Bb(\cZ) \mid \forall \mu \in \cB, \langle f, \mu \rangle \ge 1 \}\,.
    \end{align}
\end{definition}

\begin{definition}[Co-star-shaped and Shady Set]
\label{def:co-star and shady set}
    A non-empty subset $\cA \subset \Bb(\cZ)$ is said to be \emph{shady} if and only if it is convex, $0 \not \in \cA$ and such that,
    \begin{align}
        [1, \infty ) \, \cA \subseteq \cA \,.
    \end{align}
    The above property alone identifies \emph{co-star-shaped} sets. We can give an equivalent definition for sets $\cB \subset \ba(\cZ)$.
\end{definition}

\begin{proposition}[Properties of Antipolar Sets] \label{prop: Properties of antipolar sets}
Given $\cA \subseteq \Bb(\cZ)$ \st $\cA^\diamond \neq \emptyset$:
    \begin{enumerate}[label=(P\arabic*), ref=P\arabic*]
        \item \label{prop: Properties of antipolar sets - closed shady sets} $\cA^\diamond = \text{level}_{\ge 1} ( \rho_\cA ) \coloneqq \{ \mu \in \ba(\cZ) \mid \rho_\cA(\mu) \ge 1 \} $;
        \item $\cA^\diamond = \overline{\co} \Big( [1, \infty)  \cA \Big)^\diamond$;
        \item The mapping $\cA \mapsto  \cA^\diamond $ takes closed shady sets into closed shady sets;
        \item \label{prop: Properties of antipolar sets - bipolar theorem} $\cA^{\diamond\diamond} = \overline{\co} \Big( [1, \infty)  \cA \Big) $;
        \item $\mc G \subseteq \cA \Rightarrow \overline{\co} \Big( [1, \infty)  \mc G \Big) \subseteq \overline{\co} \Big( [1, \infty)  \cA \Big) \Leftrightarrow \cA^\diamond \subseteq \mc{G}^\diamond$;
        \item \label{prop: Properties of antipolar sets - antipolar of scaled set is scaled antipolar} $(\lambda \cA)^\diamond = \lambda \cA^\diamond$ for $\lambda \neq 0$, $(\lambda \cA)^\diamond = \emptyset$ for $\lambda =0$;
        \item \label{prop: Properties of antipolar sets - antipolar of union is intersection of antipolars} Given a collection of sets $\mc C$ in $\Bb(\cZ)$. It holds that $\left( \bigcup_{\cA\in \mc C} \cA \right)^\diamond = \bigcap _{\cA\in \mc C} \cA^\diamond$.
    \end{enumerate}
    Analogous properties hold for $\cA \subseteq \ba(\cZ)$.
\end{proposition}
\begin{proof}
    \begin{enumerate}[label=(P\arabic*), ref=P\arabic*]
        \item[] 
        \item True by definition.
        \item \cite[Lemma 4.2]{penot2000harmonic}
        \item \cite[Lemma 4.2]{penot2000harmonic}
        \item Follows from  claim \ref{prop: Properties of antipolar sets - closed shady sets}.
        \item \label{prop: Properties of antipolar sets - general antipolar incapulation} The first $\Rightarrow$ follows immediately. For the second $\Rightarrow$: pick any $f' \in \cA^\diamond$.  
        By definition $\langle f, f'\rangle \ge 1$ for all $f \in \cA$, hence in particular for all $f \in \mc{G}$, from which follows $f' \in \mc{G}^\diamond$. The direction $\Leftarrow$ follows from (\ref{prop: Properties of antipolar sets - bipolar theorem}).
        \item By definition, we can write
        \begin{align}
            (\lambda \cA)^\diamond &\coloneqq \{ \mu \in \ba(\cZ) \mid \langle f, \mu \rangle \ge 1 \ \forall f\in \lambda \cA\} \\
            &= \{ \mu \in \ba(\cZ) \mid \langle \lambda f', \mu \rangle \ge 1 \ \forall f'\in  \cA\} \\
            &= \{  \mu \in  \ba(\cZ) \mid \langle  f', \lambda\mu \rangle \ge 1 \ \forall f'\in  \cA\} \\
            &= \{  \mu' \in \lambda \ba(\cZ) \mid \langle  f', \mu' \rangle \ge 1 \ \forall f'\in  \cA\}
        \end{align}
        It follows that, for $\lambda\neq0$, $(\lambda \cA)^\diamond=\lambda \cA^\diamond$. For $\lambda=0$, we have
        \begin{align}
            (\lambda \cA)^\diamond \coloneqq \{ \mu \in \ba(\cZ) \mid \langle 0, \mu \rangle \ge 1  \, \} = \emptyset .
        \end{align}
        \item 
        \begin{itemize}
            \item[``$\subseteq$''] Let $\mu \in \left( \bigcup_{\cA \in \mc C} \cA \right)^\diamond$.  
            By definition, for every $f \in \bigcup_{\cA \in \mc C} \cA$ we have $\langle f, \mu \rangle \ge 1$.  
            In particular, for any fixed $\cA \in \mc C$, since $\cA \subseteq \bigcup_{\cA \in \mc C} \cA$, it follows that  $\langle f, \mu \rangle \ge 1$ for all $f \in \cA$, \ie $\mu \in \cA^\diamond$.  
            As this holds for every $\cA \in \mc C$, we can conclude that  
            $\mu \in \bigcap_{\cA \in \mc C} \cA^\diamond$.
            \item[``$\supseteq$''] 
            Let $\mu \in \bigcap_{\cA \in \mc C} \cA^\diamond$.  
            Then for each $\cA \in \mc C$ we have $\langle f, \mu \rangle \ge 1$ for all $f \in \cA$.  
            If $f' \in \bigcup_{\cA \in \mc C} \cA$, then $f' \in \cA'$ for some $\cA' \in \mc C$, and hence 
            $\langle f', \mu \rangle \ge 1$.  
            Therefore $\mu \in \left( \bigcup_{\cA \in \mc C} \cA \right)^\diamond$.
        \end{itemize}
    \end{enumerate}
\end{proof}

\begin{definition}[Antipolar Support Function]
\label{def:antipolar-function}
    Let $\cA$ be a set in $\Bb(\XY)$, its associated \emph{antipolar support function} $\rho^\diamond \colon \Bb(\XY) \to \overline{\reals}_{\ge0}$ is defined as 
    \begin{align} 
        \rho_\cA^\diamond (f) \coloneqq \sup \{ \lambda \ge 0 \,\mid\, \langle f, \mu \rangle \ge \lambda \rho_\cA(\mu) \enspace \ \forall \mu \in \ba(\XY) \} \,.
    \end{align}
    %where we assume that $\f00 = 0$ and $\f{\infty}{\infty} = \infty$.\footnote{We do not analyze further the consequences of these choices, as it is outside the scope of this presentation. For a more careful analysis, please refer to \cite{chancelier2024unified}}
\end{definition} 

Notice that, since both the bilinear form $\langle f, \mu \rangle$ and the support function are 1-homogeneous functions, an equivalent definition of the antipolar support function only involves signed measures of norm one, \ie,
\begin{align}
    \rho_\cA^\diamond (f) \coloneqq \sup \{ \lambda \ge 0 \,\mid\, \langle f, \mu \rangle \ge \lambda \rho_\cA(\mu) \enspace ,\ \forall \mu \in \bb{B}_1 \cup \{ 0 \}\, \} \,.
\end{align}

\begin{definition}[Antigauge Function]
\label{def:lsc gauge function}
    Let $\cA$ be a set in $\Bb(\XY)$, its associated \emph{antigauge function} $\beta \colon \Bb(\XY) \to \overline{\reals}_{\ge0}$ is defined as 
    \begin{align}
        \beta_\cA(f) \coloneqq \sup \{ \lambda \ge0 \,\mid\, f \in \lambda \cA \}\,.
    \end{align}
\end{definition}

One can prove that this formulation of the antigauge function is upper semi-continuous when $0 \not\in \overline{\cA}$ \cite[Proposition 2.4 (c)]{penot2000harmonic}, unlike other available definitions in the literature. It also holds that, if $0 \not\in \overline{\cA}$, the antigauge function does not take the value $+\infty$. For a detailed discussion and proof, see from Example 2.1 on in \citet{penot2000harmonic}.

\begin{proposition}[Properties of Support Function and Antipolar] \label{prop: Properties of antipolar and support functions}
Given a closed shady set $\cA \subseteq \Bb(\cZ)$:
    \begin{enumerate}[label=(P\arabic*), ref=P\arabic*]
        \item \label{prop: Properties of antipolar function - support of bi-antpolar} $\rho_{\cA}(\mu) = \rho_{\cA^{\diamond\diamond}} (\mu) = \rho_{[1, +\infty)  \cA}(\mu) \ \forall \mu\in \ba(\cZ)$;
        \item \label{prop: Properties of antipolar function - support is gauge of antipolar} $\rho_{\cA^\diamond} (f) = \beta_\cA (f)$ and $\rho_\cA (f) = \beta_{\cA^\diamond} (f)\ \forall f\in \Bb(\cZ)$; 
        \item \label{prop: Properties of antipolar function - antipolar support is support antipolar} $\rho_{\cA}^\diamond (f) = \rho_{\cA^\diamond } (f)\ \forall f\in \Bb(\cZ)$;
        \item \label{prop: Properties of antipolar function - antipolar support invariant clco} For a general set $\cA \subseteq \Bb(\cZ)$, $\rho^\diamond_{\cA} (f) = \rho^\diamond_{\overline{\co} \cA} (f) \ \forall f\in \Bb(\cZ)$.
    \end{enumerate}
\end{proposition}
\begin{proof}
    \begin{enumerate}
        \item[]
        \item Notice that the support function in invariant to convex closure and that $\cA$ is closed shady, hence the bipolar theorem holds (Proposition~\ref{prop: Properties of antipolar sets} (\ref{prop: Properties of antipolar sets - bipolar theorem})) and proves the statement.
        \item We define the convex indicator function, which takes values
        \begin{align}\label{eq:indicator function}
            \iota_\cA (x) \coloneqq 
            \begin{cases}
                0 &\text{ if }  x \in \cA , \\
                + \infty &\text{ otherwise.}
            \end{cases}
        \end{align}
        We observe it has the following connection with the support function,
        \begin{align}
            -\rho_{\cA^\diamond}(f) \overset{ (L.\ref{lemma: support function properties})}{=} \sigma_{-\cA^\diamond}(f) = \sup_{\mu\in -\cA^\diamond}\langle f, \mu \rangle = \sup_{\mu\in \ba(\cZ)} \langle f, \mu \rangle - \iota_{-\cA^\diamond}(\mu)  \eqqcolon \iota_{-\cA^\diamond}^*(f)\,.
        \end{align}
        The right hand side is known as the Legendre-Fenchel conjugate of a function, 
        \begin{align}
            F^*(f) \coloneqq \sup_{\mu\in \ba(\cZ)} \langle f, \mu \rangle - F(\mu),
        \end{align}
        computed for $F=\iota_A$.
        Hence, we can apply \cite[Lemma 4.1]{penot2000harmonic} and get that 
        \begin{align}
            [(-\beta_{\cA})^*]^*(f) = \iota_{-\cA^\diamond}^*(f) = -\rho_{\cA^\diamond}(f)\,.
        \end{align}
        As a consequence of the bipolar theorem (Proposition~\ref{prop: Properties of antipolar sets} (\ref{prop: Properties of antipolar sets - bipolar theorem})), the closed and shady set $\cB \coloneqq \cA^\diamond$ is \st $B^\diamond = \cA$. Thus,
        \begin{align}
            [(-\beta_{\cB^\diamond})^*]^*(\mu) = \iota_{-\cB}^*(\mu) = -\rho_{\cB}(\mu)\,.
        \end{align}
        To conclude the argument, we notice that under the assumption of $\cA$ being closed and shady, $\beta_{\cA}$ is a concave, upper semi-continuous function, therefore the biconjugate theorem \cite[Theorem 2.3.3]{zalinescu2002convex} holds for $-\beta_\cA$ and $-\beta_{\cA^\diamond}$.
        \item Consider the definition of the antipolar of the support function:
        \begin{align}
            \rho^\diamond_\cA(f) &\coloneqq \sup \{ \lambda \ge 0 \,\mid\, \langle f, \mu \rangle \ge \lambda \rho_\cA(\mu) \enspace \forall \mu \in \ba(\cZ)\} \\
            &= \sup \Big\{ \lambda \ge 0 \,\mid\, \inf_\mu \Big( \langle f, \mu \rangle - \lambda \rho_\cA(\mu) \Big) \ge 0 \Big\}\\
            &= \sup \Big\{ \lambda \ge 0 \,\mid\, -\sup_\mu \Big( \langle -f, \mu \rangle + \lambda \rho_\cA(\mu) \Big) \ge 0 \Big\}\\
            &\overset{(\ref{prop: Properties of antipolar function - support is gauge of antipolar})}{=} \sup \Big\{ \lambda \ge 0 \,\mid\, -\sup_\mu \Big( \langle -f, \mu \rangle + \lambda \beta_{\cA^\diamond}(\mu) \Big) \ge 0 \Big\} \\
            &\overset{(\star)}{=} \sup \{ \lambda \ge 0 \,\mid\, -\iota_{-(\lambda \cA)^{\diamond\diamond}}(-f) \ge 0 \} \overset{\eqref{eq:indicator function}}{=} \sup \{  \lambda \ge 0 \,\mid\, f \in (\lambda \cA)^{\diamond\diamond} \} \\
            &\overset{(\ref{prop: Properties of antipolar sets - antipolar of scaled set is scaled antipolar})}{=} \sup \{  \lambda \ge 0 \,\mid\, f \in \lambda \cA^{\diamond\diamond} \} = \sup \{  \lambda \ge 0 \,\mid\, f \in \lambda \cA \} = \beta_{\cA}(f),
        \end{align}
        where $(\star)$ is Lemma 4.1 in \citet{penot2000harmonic}. Notice that we are allowed to use \ref{prop: Properties of antipolar sets - antipolar of scaled set is scaled antipolar} here, since the case $\lambda =0$ cannot occur in the supremum. That is because it would cause the condition to become ``$f\in (0 \cA)^{\diamond\diamond}=\emptyset$'', which cannot be met. Using \ref{prop: Properties of antipolar function - support is gauge of antipolar} again, we obtain the thesis.
        \item As a consequence of Lemma~\ref{lemma: support function properties}, we get
         \begin{align} 
        \rho_\cA^\diamond (f) &\coloneqq \sup \{ \lambda \ge 0 \,\mid\, \langle f, \mu \rangle \ge \lambda \rho_\cA(\mu) \enspace \ \forall \mu \in \ba(\XY) \} \\
        &= \sup \{ \lambda \ge 0 \,\mid\, \langle f, \mu \rangle \ge \lambda \rho_{\overline{\co}\cA}(\mu) \enspace \ \forall \mu \in \ba(\XY) \} = \rho^\diamond_{\overline{\co}\cA}(f)\,.
    \end{align}
    \end{enumerate}
\end{proof}

\subsection{Antipolar GDPI for Bayes Risk}

Inspired by the analysis of the antipolars of superprediction set, loss and conditional Bayes risk carried out by \citet{williamson2023geometry}, we aim to study the antipolar version of the GDPI for Bayes Risk. 
First, we notice that a characterization of the GDPI for Bayes Risk through subsethood holds also in the antipolar setting, for general shady sets. 

\begin{lemma}
    \label{lemma:antipolar-duality-support-sets}
    Given two shady sets $\cF, \mc G \subseteq \Bb(\cZ)$ and a dual pairing $\langle \cdot, \cdot \rangle \colon \Bb(\cZ) \tm \ba(\cZ) \to \reals$, the following statements are equivalent:
    \begin{enumerate}
        \item $\rho_{\mc G}(\mu) \leq \rho_{\cF} (\mu) \quad \forall \mu \in \ba(\cZ)$;
        \item $\rho_{\cF^\diamond}(f) \leq \rho_{\cG^\diamond} (f) \quad \forall f \in \Bb(\cZ)$;
        \item $\rho^\diamond_\cF(f) \leq \rho^\diamond_\cG (f) \quad \forall f \in \Bb(\cZ)$;
        \item $\cG \subseteq \cF $;
        \item $\cF ^\diamond \subseteq \cG^\diamond$.
    \end{enumerate}
\end{lemma}
\begin{proof}
     We know already that $1 \Leftrightarrow 4$ from the proof of Lemma~\ref{lemma:support-ineq-characterization} and that $4 \Leftrightarrow 5$ because of \ref{prop: Properties of antipolar sets - general antipolar incapulation} of Proposition~\ref{prop: Properties of antipolar sets}. 
     Because of the bijection between support functions and closed convex sets (Lemma~\ref{lemma:Duality between Closed Convex Sets and Support Functions}), we know that also $2\Leftrightarrow 5$ holds.  The statements in 2 and 3 are equivalent because of Proposition~\ref{prop: Properties of antipolar and support functions} (\ref{prop: Properties of antipolar function - antipolar support is support antipolar}). This concludes the proof.
\end{proof}

In order to ensure the above result can be applied to the set $\cspr (\ell \circ \cH)$, we need to prove the following lemmas. 
\begin{lemma}%[Recession Cone of Closed Convex Hull of Superprediction Set]
\label{lemma:recession cone of closed convex hull of superprediction set}
    Suppose $\cF \subseteq \Bb(\XY)_{\ge 0}$. Then,
    \begin{align}
        \cspr \, \cF \oplus \Bb(\XY)_{\ge 0} \subseteq \cspr \, \cF.
    \end{align}
\end{lemma}
\begin{proof}
    We show that $\cspr\left( \cspr \, \cF \oplus \Bb(\XY)_{\ge 0} \right) = \cspr \, \cF$.
    We use Lemma~\ref{lemma:co-spr-definition-with-fin-probabilities}, hence we have to show that,
    \begin{align}
        \rho_{\cspr \left( \cspr \, \cF \oplus \Bb(\XY)_{\ge 0} \right)}(\phi) = \rho_{\cspr \, \cF}(\phi), \ \forall \phi \in  \finProb(\XY)\,.
    \end{align}
    For this, note that, for all $\phi \in \finProb(\XY)$,
    \begin{align}
        \rho_{\cspr (\cspr \, \cF \oplus \Bb(\XY)_{\ge 0})}(\phi) &%\overset{(\ref{supp - invariance to closed convex hull})}{=} \rho_{\spr (\cspr \, \cF \oplus \Bb(\XY)_{\ge 0})}(\phi) 
        \overset{(L\ref{lemma:recession cone spr})}{=} \rho_{\cspr \, \cF \oplus \Bb(\XY)_{\ge 0} \oplus \Bb(\XY)_{\ge 0}}(\phi)\\
        &\overset{(\ref{supp - minkowski sum})}{=}\rho_{\cspr \, \cF}(\phi) + \rho_{\Bb(\XY)_{\ge 0}}(\phi) + \rho_{\Bb(\XY)_{\ge 0}}(\phi) = \rho_{\cspr \, \cF}(\phi)\,,
    \end{align}
    because $\rho_{\Bb(\XY)_{\ge 0}}(\phi) = 0$ as $0 \in \Bb(\XY)_{\ge 0}$.
\end{proof}

\begin{lemma}%[Closed Convex Hull of Superprediction Set is Shady]
\label{lemma:superprediction set is shady}
    Let $\cF \subseteq \Bb(\XY)_{\ge 0}$. If $0 \not\in \cspr \, \cF$, the set $\cspr \, \cF$ is shady. 
\end{lemma}
\begin{proof}
    We use the logic given in \cite[Lemma 18]{williamson2023geometry}. For every $b \in \Bb(\XY)_{\ge 0}$ and all $\alpha \in [0,\infty)$, $\alpha b \in \Bb(\XY)_{\ge 0}$, hence,
    \begin{align}
        \cspr \, \cF \oplus \{ \alpha b \} \subseteq \cspr \, \cF\,,
    \end{align}
    by Lemma~\ref{lemma:recession cone of closed convex hull of superprediction set}. Since $\cspr \, \cF \subseteq \Bb(\XY)_{\ge 0}$, we can assume that $b \in \cspr \, \cF$, hence,
    \begin{align}
        \cspr \, \cF \oplus \{ \alpha b\} \subseteq (1+\alpha) \,\cspr \, \cF \subseteq \cspr \, \cF\,,
    \end{align}
    which concludes the proof.
\end{proof}

\begin{proposition}[Characterization of Antipolar Generalized Data Processing Inequality]
    \label{prop:antipolar-dpi-characterization}
    Let $\ell\in\cL(\cY)$ be a loss function and $\cH \subseteq \cM(\cX, \cY)$ a model class. Let $\langle \cdot, \cdot \rangle $ be the canonical dual pairing relating $\Bb(\XY)$ and $ \ba(\XY)$. 
    Consider a Markov operator $\mto{\ka}$ associated to a kernel $\ka \in \cM( \XY , \XY)$. 
    Then, assuming that $0 \not\in \cspr(\ell \circ \cH)$, the condition 
    \begin{align}
        \big(\cspr(\ell \circ \cH)\big)^\diamond \subseteq \big(\cspr(\mto{\ka}(\ell \circ \cH))\big)^\diamond
    \end{align}    
    is equivalent to the GDPI for Bayes Risk and to its antipolar counterpart
    \begin{align}
        \rho^\diamond_{\cspr (\ell \circ \cH)}(f) \leq \rho^\diamond_{\cspr \mto{\ka}(\ell \circ \cH)} (f) \quad \forall f \in \Bb(\cZ)\,.
    \end{align}
    %as well as to all the statements in Lemma~\ref{lemma:antipolar-duality-support-sets}.
\end{proposition}
\begin{proof}
Consider the antipolar of the convex closure of the constrained superprediction set,
\begin{align}
    \big(\cspr (\ell \circ \cH) \big)^\diamond = \{ \phi \in \ba(\XY) \mid \forall f \in \cspr (\ell \circ \cH) \,, \ \langle f, \phi \rangle \ge 1 \} \,.
\end{align}
Using Lemma~\ref{lemma:antipolar-duality-support-sets},
\begin{align}
     \cspr(\mto{\ka}(\ell \circ \cH) ) \subseteq \cspr (\ell \circ \cH)  \quad \Leftrightarrow \quad \big( \cspr (\ell \circ \cH)\big)^\diamond \subseteq \big( \cspr(\mto{\ka}(\ell \circ \cH))\big)^\diamond \,.
\end{align}
Given that Lemma~\ref{lemma:superprediction set is shady} holds when $0 \not\in\ell \circ \cH$, we have that the corrupted superprediction set is shady. Thus, there is a bijection between the antipolar functions and antipolar superprediction sets. We can write
Applying Lemma~\ref{lemma:support-ineq-characterization} and Proposition~\ref{prop: Properties of antipolar and support functions} (\ref{prop: Properties of antipolar function - antipolar support is support antipolar}) to the RHS, and get
\begin{align}
    \rho^\diamond_{\cspr (\ell \circ \cH)}(f) \leq \rho^\diamond_{\cspr(\mto{\ka}(\ell \circ \cH))} (f) \quad \forall f \in \Bb(\cZ)\,.
\end{align}
Then, rest of the thesis is a direct consequence of Lemma~\ref{lemma:antipolar-duality-support-sets} and Theorem~4.2.
\end{proof}

\begin{remark}
    Notice that, as a direct consequence of how we prove Proposition~\ref{prop:antipolar-dpi-characterization}, the subsethood condition on the antipolar superprediction sets is equivalent to all the statements in Lemma~\ref{lemma:antipolar-duality-support-sets}
\end{remark}

\subsection{Strong GDPI for Bayes Risk} 

We now reconnect with a more learning theoretic perspective and relate all the antipolar objects to Risk-related ones.

\begin{definition}[Minimal Relative Risk]
\label{def:minimal relative risk}
    Given a set $\ell\circ\cH \subseteq \Bb(\XY)_{\ge 0}$ such that $\csprH$ is closed and shady, and a function $f \in \Bb(\XY)_{> 0}$, we define the quantity
    \begin{align}
        \reg_{\ell\circ\cH} (f) \coloneqq \inf_{\phi \in \finProb(\XY)} \f{\langle f, \phi \rangle}{\br_{\ell\circ\cH}(\phi)} \in (0,+\infty).
    \end{align}
    as the \emph{minimal relative risk} of $f$, \ie, a generalized prediction.
\end{definition}

The quantity above is well defined:  When $\csprH$ is a closed shady set, we know that $0\notin\csprH$. Hence, for every $\phi\in\finProb(\XY)$ and $f\in \Bb(\XY)_{>0}$ we have $\langle f, \phi \rangle > 0$ and $\br_{\ell\circ\cH}(\phi) > 0$. In addition, as $f$ is a bounded measurable function, and $\phi \in \finProb$, the numerator will not diverge.

\begin{proposition}
\label{prop: minimal relative risk greater antipolar support}
    Given a set $\ell\circ\cH \subseteq \Bb(\XY)_{\ge 0}$, it holds that
    \begin{align}
        \reg_{\ell\circ\cH} (f) \ge \rho^\diamond_{ \spr(\ell\circ\cH)} (f) \quad \forall f \in \Bb(\XY)_{>0}.
    \end{align}
\end{proposition}
\begin{proof}
    By definition of Minimal relative risk and support function, we can write
    \begin{align}
        \reg_{\ell\circ\cH} (f) &\coloneqq \inf_{\phi \in \finProb(\XY)} \f{\langle f, \phi \rangle}{\br_{\ell\circ\cH}(\phi)} \\
        &= \inf_{\phi \in \finProb(\XY)} \Big\langle f, \f\phi{\br_{\ell\circ\cH}(\phi)} \Big\rangle 
        \eqqcolon \rho_{R_\finProb}(f) \quad \forall f \in \Bb(\XY)_{>0},
    \end{align}
    where $R_\finProb \coloneqq \bigcup_{\phi \in \finProb} \left\{ \f\phi{\br_{\ell\circ\cH}(\phi)} \right\} \subset \ba(\XY)_{>0}$. Let us now consider the antipolar set of $R_\finProb$,
    \begin{align}
        R_\finProb^\diamond &= \left( \bigcup_{\phi \in \finProb} \left\{ \f\phi{\br_{\ell\circ\cH}(\phi)} \right\} \right) ^\diamond
        \overset{P.\ref{prop: Properties of antipolar sets}}{=} \bigcap_{\phi \in \finProb} \left\{ \f\phi{\br_{\ell\circ\cH}(\phi)} \right\}^\diamond \\
        &= \bigcap_{\phi \in \finProb} \left\{ f \in \Bb(\XY) \,\mid\, \Big\langle f, \f\phi{\br_{\ell\circ\cH}(\phi)} \Big\rangle \ge 1 \right\} \\
        &= \bigcap_{\phi \in \finProb} \left\{ f \in \Bb(\XY) \,\mid\, \langle f, \phi \rangle \ge \br_{\ell\circ\cH}(\phi) \right\}\\
        &= \bigcap_{\phi \in \finProb} \left\{ f \in \Bb(\XY) \,\mid\, \langle f, \phi \rangle \ge \rho_{ \spr(\ell\circ\cH)}(\phi) \right\} \supseteq \csprH.
    \end{align}
    Hence, by properties of the support function (Lemma~\ref{lemma:support-ineq-characterization}), 
    \begin{align}
        \rho_{ \csprH^\diamond} (f) \le \rho_{R_\finProb^{\diamond\diamond}} (f) = \rho_{\overline{\co}( [1,+\infty) R_\finProb) } (f) = \rho_{ [1,+\infty) R_\finProb} (f).
    \end{align}
    As $R_\finProb \subseteq [1,+\infty) R_\finProb $, we have the thesis.
\end{proof}

%{TODO IF TIME could be good to add here a remark or example of why the above is not an equality. what is missing? is the bound tight?}

\begin{definition}[Strong GDPI for Bayes Risk]
    \label{def:sgdpi}
    Let $\ell\in\cL(\cY)$ be a loss function and $\cH \subseteq \cM(\cX, \cY)$ a model class.
    Let $\ka \in \cM(\XY, \XY)$ be a Markov kernel.
    We call the inequality
    \begin{align}
     \br_{\ell \circ \cH} (\amto{\ka} \phi) \ge \alpha_{\ell \circ \cH}(\ka) \; \br_{\ell \circ \cH} (\phi) \quad \forall \phi \in \finProb(\XY) \,,
    \end{align}
    the Strong GDPI for Bayes Risk
    where $\alpha_{\ell \circ \cH}(\ka)$ is some suitable \emph{inflation coefficient}.
\end{definition}

This is the Bayes risk analogous  of the classical information theoretic counterpart, called the Strong Data Processing Inequality for the $f$-divergence,
\begin{align}
    D_f(\amto{\ka} \phi \,\|\, \amto{\ka} \psi ) &\le \eta_f(\ka) \, D_f(\phi \,\|\, \psi) \,,\\
    \eta_f(\ka) &\coloneqq \sup_{\phi, \psi \in \mc{P}}  \f{D_f(\amto{\ka} \phi \| \amto{\ka} \psi )}{D_f(\phi \| \psi )}\,,\\
    D_f(\phi \,\|\, \psi) &\coloneqq \int f \left( \f{d\phi}{d\psi} \right) d\psi.
\end{align}
where $\ka \in \cM(\cX,\cX)$, $f\colon (0, +\infty) \to \reals$ is convex and \st $f(0) \coloneqq \lim_{z\to 0+}f(z)$, $f(1)=0$, $\mc{P}$ is the set of probabilities for which $D_f$ takes positive and finite values and $\eta_f(\ka)$ is a called \emph{contraction coefficient}.%
\footnote{See \cite[Definition 7.1]{polyanskiy2025information} and \cite[Chapter 33.2]{polyanskiy2025information} for more details, and \cite{williamson2024information} for a way to relate Bayes risk and $f$-divergence, which additionally explains the change of direction in the inequality.} %{TODO IF TIME In appendix: Work out relationship btw this and BR and don in past work; Inf-risk bridge; is our coefficient coherent with/connected to that setting? it might be that $\alpha, \eta$ are lower and upper bounds for the non-linear relations btw BRs (a là non-linear SDPI)}.
To be compatible with our setting while analogous to the above, the coefficient $\alpha$ must take a different form than $\eta$. In the following, we turn \ref{def:minimal relative risk} around to give a suitable definition and show its connection to the antipolar support function.
\begin{definition}[Inflation coefficient]
    \label{def:alpha-coeff}
    Let $\ell\in\cL(\cY)$ be a loss function and $\cH \subseteq \cM(\cX, \cY)$ a model class inducing a shady $\csprH$. Consider a Markov kernel $\ka \in \cM(\XY, \XY)$. Then, the inflation coefficient of $\ka$ relative to $\ell\circ\cH$ is
    \begin{align}
        \alpha_{\ell\circ\cH}(\ka) \coloneqq \inf_{\phi \in \finProb(\XY)} \f{\br_{\ell\circ\cH}(\amto{\ka}\phi)}{\br_{\ell\circ\cH}(\phi)} \,.
    \end{align}
\end{definition}

\begin{lemma}[Infimum over Minimal Relative Risk is Inflation Coefficient]
    \label{lemma:minimal relative risk is sdpi coefficient}
    Let $\ell\in\cL(\cY)$ be a loss function and $\cH \subseteq \cM(\cX, \cY)$ a model class \st $0 \notin \cspr(\ell \circ \cH)$. Furthermore, suppose that $\ka \in \cM(\XY, \XY)$ \st $0 \notin \cspr(\mto{\ka}(\ell \circ \cH))$.
    Then,
    \begin{align}
        \alpha_{\ell\circ\cH}(\ka) = \inf_{f \in \cspr(\mto{\ka}(\ell \circ \cH))} \reg_{\ell \circ \cH}(f)\,.
    \end{align}
\end{lemma}
\begin{proof}
    By swapping infima we obtain,
    \begin{align}
        \inf_{f \in \cspr(\mto{\ka}(\ell \circ \cH))} \reg_{\ell \circ \cH} (f) &\overset{D.\ref{def:minimal relative risk}}{=} \inf_{f \in \cspr(\mto{\ka}(\ell \circ \cH))} \inf_{\phi \in \finProb(\XY)} \f{\bb E_{\phi}(f)}{\br_{\ell\circ\cH}(\phi)}\\
        &= \inf_{\phi \in \finProb(\XY)} \f{\inf_{f \in \cspr(\mto{\ka}(\ell \circ \cH))} \bb E_{\phi}(f)}{\br_{\ell\circ\cH}(\phi)}\\
        &= \inf_{\phi \in \finProb(\XY)} \f{\rho_{\mto{\ka}(\ell \circ \cH))} (\phi)}{\br_{\ell\circ\cH}(\phi)} \overset{L\ 2.10}{=} \alpha_{\ell\circ\cH}(\ka)\,.
    \end{align}
\end{proof}

\begin{proposition}
    \label{prop:sgdpi-characterization}
    Let $\ell\in\cL(\cY)$ be a loss function and $\cH \subseteq \cM(\cX, \cY)$ a model class inducing a generalized superprediction set $\spr(\ell\circ\cH)$ \st $0 \notin \cspr(\ell \circ \cH)$. Consider a Markov kernel $\ka \in \cM(\XY, \XY)$ \st $0 \not\in\cspr(\mto{\ka}(\ell\circ\cH))$.  %$\mto{\ka}$ is continuous \wrt $\sBbba$. 
    Let $\alpha_{\ell\circ\cH}(\ka)$ be its associated coefficient, defined as in Definition~\ref{def:alpha-coeff}.
    Then, the following statements are equivalent:
    \begin{enumerate}
        \item The GDPI for Bayes Risk holds;
        \item The Strong GDPI for Bayes Risk holds with inflation coefficient $\alpha_{\ell\circ\cH}(\ka) \ge 1$;
        \item It holds that $\reg_{\ell\circ\cH}(f) \ge 1\,, \enspace \forall f \in \cspr(\mto{\ka}(\ell\circ\cH))$.
    \end{enumerate}
\end{proposition}
\begin{proof}
    \begin{itemize}
        \item[] 
        \item[] $(1) \Rightarrow (2)$: This is an immediate consequence of the GDPI for Bayes risk.
         \item[] $(2) \Rightarrow (3)$: Lemma~\ref{lemma:minimal relative risk is sdpi coefficient}.
         \item[] $(3) \Rightarrow (1)$: The remaining implication is given by, 
         \begin{align}
             & & \reg_{\ell\circ\cH}(f) &\ge 1\,, \enspace \forall f \in \cspr(\mto{\ka}(\ell\circ\cH))\\
             &{\Leftrightarrow} &\inf_{\phi \in \finProb(\XY)} \f{\bb E_{\phi}(f)}{\br_{\ell\circ\cH}(\phi)} &\ge 1\,, \enspace \forall f \in \cspr(\mto{\ka}(\ell\circ\cH))\\
             &\Rightarrow &\f{\bb E_{\phi}(f)}{\br_{\ell\circ\cH}(\phi)} &\ge 1\,, \enspace \forall f \in \cspr(\mto{\ka}(\ell\circ\cH)), \forall \phi \in \finProb(\XY)\\
             &\Rightarrow &\bb E_{\phi}(f) &\ge \br_{\ell\circ\cH}(\phi) \,, \enspace \forall f \in \cspr(\mto{\ka}(\ell\circ\cH)), \forall \phi \in \finProb(\XY)\\
             &\Rightarrow &\inf_{f \in \cspr(\mto{\ka}(\ell\circ\cH))} \bb E_{\phi}(f) &\ge \br_{\ell\circ\cH}(\phi) \,, \enspace \forall \phi \in \finProb(\XY)\\
             &\overset{L\ 2.10}{\Rightarrow} &\br_{\ell\circ\cH}(\amto{\ka}\phi)  &\ge \br_{\ell\circ\cH}(\phi) \,, \enspace \forall \phi \in \finProb(\XY)\,.
         \end{align}
    \end{itemize}
\end{proof}

\begin{corollary}[Sufficient Condition for Strong GDPI for Bayes Risk]
    \label{cor: antipolar support function sufficient condition for sgdpi}
    Let $\ell\in\cL(\cY)$ be a loss function and $\cH \subseteq \cM(\cX, \cY)$ a model class inducing a generalized superprediction set $\csprH$ \st $0 \notin \cspr(\ell \circ \cH)$. Consider a Markov kernel $\ka \in \cM(\XY, \XY)$ \st $0 \not\in\cspr(\mto{\ka}(\ell\circ\cH))$. 
    Then,
    \begin{gather}
        \rho^\diamond_{\csprH}(f) \ge  1 \quad \forall f \in \Bb(\XY)_{>0} \\
        \hfill \Downarrow \hfill \\
        \alpha_{\ell\circ\cH}(\ka) \ge  1 \\
        \text{and}\\
        \br_{\ell \circ \cH} (\amto{\ka} \phi) \ge \alpha_{\ell \circ \cH}(\ka) \; \br_{\ell \circ \cH} (\phi) \quad \forall \phi \in \finProb(\XY) .
    \end{gather}
\end{corollary}
\begin{proof}
    The result is given by
         \begin{align}
             & & \rho^\diamond_{\csprH}(f) &\ge 1\,, \enspace \forall f \in \cspr(\mto{\ka}(\ell\circ\cH))\\
             &\overset{P\ref{prop: minimal relative risk greater antipolar support}}{\Rightarrow} &\inf_{\phi \in \finProb(\XY)} \f{\bb E_{\phi}(f)}{\br_{\ell\circ\cH}(\phi)} &\ge 1\,, \enspace \forall f \in \cspr(\mto{\ka}(\ell\circ\cH)).
         \end{align}
    Then, we apply Proposition~\ref{prop:sgdpi-characterization}.
\end{proof}

We conclude this section by remarking that, even if we have found new relationships between the SDPI coefficients and quantities related to the statistical learning problem, \ie, minimal relative risk and antipolar support function, we are still far from proposing a formula for computing the coefficient itself. 
This problem is know to be complex to tackle. 
Contraction coefficients for $f$-divergences (see formula and short discussion under Definition~\ref{def:sgdpi}) quantify how much an information measure shrinks under the action of a Markov kernel $\ka$. 
As we have already seen, the SDPI constant $\eta_f$ for a given divergence $D_f$ is defined as the smallest value such that
\begin{align}
    D_f(\amto\ka\phi \,\|\, \amto\ka\psi) \le \eta_f(\ka) \, D_f(\phi \,\|\, \psi).
\end{align}
These coefficients depend intricately on both the choice of $f$-divergence and the structure of the kernel $\ka$, and finding closed form expressions for them is rather complex. 
Classical cases include \cite{makur2020comparison}: the $\chi^2$-divergence, where $\eta_{\chi^2}$ is proven to be equal to the squared maximal correlation, and can be used to bound coefficients for other choices of $f$; the Dobrushin coefficient for the total variation ($f(x) = \f12 |x -1|$); the Kullback--Leibler divergence coefficient, which is typically harder to characterize but admits closed-form expressions for certain families of distributions or kernels \cite[see also][]{lee2024contraction}. 

In general, determining contraction coefficients remains a technically demanding problem, with progress often limited to specific divergences or channel classes. This seems to be the case also for our inflation coefficients. 
Nevertheless, we contributed a new class of objects, namely the constrained version of the strong data processing inequality for Bayes risk and its associated coefficient, and a set of preliminary results relating them both to geometrical and learning theory quantities. We hope these can provide a new avenue for future research at the intersection of information theory and machine learning theory.

\section{Constrained Information} \label{app:predictive information}

Statistical information \cite{degroot1962uncertainty,reid2011information} is defined as the difference of some notion of entropy and some notion of conditional entropy. While full Bayes risk can be used as a conditional entropy, we cannot use constrained conditional risk as entropy and be sure to enjoy the non-negativity property of their difference. %One can observe that many relevant measures of (mutual) information are defined or proved to be positive.
Many different alternative measures of information can be defined in a way that circumvents this issue, notably \cite{garcia2012divergences,williamson2024information,finzi2026entropy}. 
We have not explored this issue in the main body of the paper, as information is classically expressed in terms of the experiment $E$ and simple corruptions \cite[see][]{williamson2024information}, while we focused on a larger class of corruptions and on joint distributions. 
Therefore, we needed to look at Bayes risk.

While we defer the extension of our results to some notion of constrained information similar to \citet{williamson2024information}'s, we can define here an immediate one, following \citet{xu2020theory} in introducing the notion of predictive families of models, and define a possible notion of ``constrained statistical information'' with a generalized version of their predictive information. 
This will be then shown to be a generalization of \citet{degroot1962uncertainty}'s notion of information.

\begin{definition}[Predictive Family]
\label{def:predictive family}
    A model class $\cH \subseteq \cM(\cX, \cY)$ is a \emph{predictive family} if, for all probabilities that are in the image of $\cH$ through $\cX$, the model class contains its associated degenerate kernel $h_\pi$. In formulas, we ask that
    \begin{align}
        \exists \, h_\pi \in \cM(\cX, \cY)\ \mid \mt{h}_\pi (x) \equiv_x \pi \quad \forall \, \pi \in \bigcup_{ h \in \cH} \mt{h}(\cX) .
    \end{align}
    We denote the subset of these constant predictors as $\cH_c \subseteq \cH$.
\end{definition}

\begin{definition}[Predictive Information]
\label{def:predictive ell H information}
    Let $\ell\in\cL(\cY)$ be a loss function, $\cH\in \cM(\cX,\cY)$ a predictive family and $\phi \in \Prob(\XY)$. Then, we call
    % \begin{align}
    %     I_{\ell \circ \cH}(X \rightarrow Y) =  \br_{\ell \circ \cH_c}(\phi) - \br_{\ell \circ \cH}(\phi),
    % \end{align}
    \begin{align}
        I_{\ell \circ \cH}(\phi) \coloneqq \br_{\ell \circ \cH_c} (\phi) - \br_{\ell \circ \cH} (\phi)\,,
    \end{align}
    the \emph{predictive $\ell \circ \cH$-information} of $\phi$.
\end{definition}

Notice that our notion of information is a direct generalization of both \citet{degroot1962uncertainty}'s statistical information and \citet{xu2020theory}'s constrained entropy and information. 
\begin{remark}[Recovery of Statistical Information]
     Consider the model class $\cH = \cM(\cX, \cY)$, which is a predictive family as its associated set $\cH_c$ contains all degenerate Markov kernels in $\cM(\cX, \cY)$. Let $\phi = \pi_{\cY}\tm E \in \Prob(\XY)$ and $\ell\in\cL(\cY)$. Since it trivially holds that
     \begin{align}
        \br_{\ell \circ \cH_c}(\pi_{\cY}\tm E) = \cbr_{\ell}(\pi_{\cY}) 
        \quad \text{and} \quad
        \br_{\ell \circ \cH}(\pi_{\cY}\tm E) = \br_{\ell}(\pi_{\cY}\tm E),
    \end{align}
    we recover the definition of statistical information when the model class is not constrained.
\end{remark}

\begin{lemma}[Recovery of $\cH$-Information]
\label{lemma:Recovery of H-Information}
    Let $\cH \subset \cM(\cX, \cY)$ be a predictive family. Let $\phi \in \Prob(\XY)$. If $(p,y) \in \Prob(\cY) \tm \cY \mapsto \ell_{\log}(p, y) = - \log(p_y)$ is the logarithmic loss, and $H_{\cH}$ is the predictive (conditional) $\cH$-entropy  \cite[as per][Definition 2]{xu2020theory}, we have that
    \begin{align}
        \br_{\ell_{\log}\circ \cH_c}(\phi) = H_{\cH}(Y)
        \quad \text{and} \quad
        \br_{\ell_{\log}\circ \cH}(\phi) = H_{\cH}(Y|X).
    \end{align}
    In addition, we recover their definition of predictive $\cH$-information, as
    \begin{align}
        I_{\cH}(X, Y) = I_{\ell_{\log} \circ \cH}(\phi) .
    \end{align}
\end{lemma}
\begin{proof} Let $\cH (\cX) \coloneqq \bigcup_{ h \in \cH} \mt{h}(\cX) \subseteq \Prob(\cY)$. The statement is proved by noticing that
    \begin{align}
     H_{\cH}(Y)  &\coloneqq  \inf_{\pi' \in \cH(\cX) } \bb E_{\pi} [- \log(\pi'_{\rv Y})] \overset{(\star)}{=}  \inf_{h \in \cH_c }  \bb E_{\rv Y \sim \pi} \bb E_{\rv X \sim \mt E(Y)} [-\log([\mt{h}(\rv X)]_{\rv Y})] = \br_{\ell_{\log} \circ \cH_c} (\phi)\,,\\
     H_{\cH}(Y|X) &\coloneqq  \inf_{h \in \cH } \bb E_{\phi} [-\log([\mt{h}(\rv X)]_{\rv Y})] \eqqcolon \br_{\ell_{\log}\circ \cH}(\phi),
    \end{align} 
    for some experiment $E \in \cM (\cX, \cY)$ such that $\pi \tm E = \phi$, and for $[\mt{h}(\rv X)]_y)$ being the $y$-th entry of the vector $\mt{h}(\rv X)\in \Prob(\cY)$. The equality in $(\star)$ holds for this specific model class, since $\bb E_{\rv X \sim \mt{E}(\rv Y)} [-\log([\mt{h}(\rv X)]_{\rv Y})]$ is in fact independent from $\rv X$.
\end{proof}

Analogous to Proposition 2 in \cite{xu2020theory} we can show some basic properties of the constrained information.
\begin{proposition}[Properties of Predictive Information]
    \label{prop:properties of ellH information}
    Let $\cH$ and $\cH'$ be predictive families such that $\cH \subseteq \cH'$. For any distribution $\phi \in \finProb(\XY)$:
    \begin{enumerate}
        \item Monotonicity: $\br_{\ell \circ \cH}(\phi) \ge \br_{\ell \circ \cH'}(\phi)$.
        \item Non-Negativity: $I_{\ell \circ \cH}(\phi) \ge 0$.
        \item Respecting Independence: If $X$ is independent of $Y$ in $\phi$, then $I_{\ell \circ \cH}(\phi) = 0$.
    \end{enumerate}
\end{proposition}
\begin{proof}
    \begin{enumerate}
        \item[]
        \item Follows by definition of Bayes risk.
        \item Since $\cH_c \subseteq \cH$ by definition, the statement holds because of point 1.
        \item Let $\cH (\cX) \coloneqq \bigcup_{ h \in \cH} \mt{h}(\cX) \subseteq \Prob(\cY)$ and $\phi = \pi_{\cX} \tm \pi_{\cY}$:
        \begin{align}
            \br_{\ell \circ \cH}(\phi)
             &= \inf_{h \in \cH} \bb{E}_{X \sim \pi_{\cX}}[\bb{E}_{Y \sim \pi_{\cY}}[(\ell \circ h) (X,Y)]]
             \ge \bb{E}_{X \sim \pi_{\cX}}[ \inf_{h \in \cH} \bb{E}_{Y \sim \pi_{\cY}}[(\ell \circ h) (X,Y)]]\\
             &= \bb{E}_{X \sim \pi_{\cX}}[ \inf_{h \in \cH_c} \bb{E}_{Y \sim \pi_{\cY}}[(\ell \circ h) (X,Y)]] = \inf_{\psi \in \cH(\cX)} \bb{E}_{Y \sim \pi_{\cY}}[\ell(\psi, Y)] = \br_{\ell \circ \cH_c}(\phi),
        \end{align}
        which makes their difference zero.
    \end{enumerate}
\end{proof}

Finally, we can immediately transfer the estimation results in \cite{xu2020theory} to our information measure.

\begin{definition}[Empirical Predictive Information]
    Let $\phi \in \Prob(\XY)$. Let $\mc{D}$ be a finite, independently and identically distributed sample from $\phi$. Let $\cH\in \cM(\cX,\cY)$ be a predictive family and $\ell\in\cL(\cY)$ be a loss function. The \emph{empirical predictive information} \wrt $\ell \circ \cH$ is given by,
    \begin{align}
        \hat{I}_{\ell \circ \cH}(\mc{D}) \coloneqq \inf_{h \in \cH_c} \left[ \frac{1}{|\mc{D}|}\sum_{y_i \in \mc D} (\ell \circ h)(y_i) \right] - \left[ \inf_{h \in \cH} \frac{1}{|\mc{D}|}\sum_{(x_i,y_i) \in \mc D} (\ell \circ h)(x_i, y_i) \right].
    \end{align}
\end{definition}
\begin{proposition}[Estimation of Predictive Information]
    Let $\phi \in \Prob(\XY)$. Let $\mc{D}$ be a finite, independently and identically distributed sample from $\phi$.
    Let $\cH$ be a predictive family and $\ell\in\cL(\cY)$ be a loss function, bounded on the set $\cH(\cX) \times \cY \coloneqq \bigcup_{h \in \cH} \mt{h}(\cX) \times \cY$ by the constant $\cB \ge 0$. Then for any $\delta \in (0, 0.5)$, with probability at least $1-2\delta$,
    \begin{align}
        |I_{\ell \circ \cH}(\phi) - \hat{I}_{\ell \circ \cH}(\mc{D})| \le 4 \mathfrak{R}_{|\mc{D}|}(\ell \circ \cH) + 2B \sqrt{\frac{2 \log \frac{1}{\delta}}{|\mc{D}|}},
    \end{align}
    where $\mathfrak{R}_{|\mc{D}|}(\ell \circ \cH)$ denotes the Rademacher complexity of $\ell \circ \cH$ with sample number $|\mc{D}|$, defined as
    \begin{align} \label{eq:rademacher}
        \operatorname{Rad}_{|\mc{D}|}(\cF) = 
        \frac{1}{|\mc{D}|}
           \mathbb{E}_{R} \left[
           \sup_{f \in \cF}
           \left|\, \sum_{i=1}^{|\mc{D}|} (1-2 R_i) f(z_i)  \, \right|
        \right],
    \end{align}
    for a sample $\mc{D} = \{ z_i \}_i$, a function class $\cF = \{ z \mapsto f(z) \in \reals_{\ge0} \}$ and $R = (R_i)_i$ with $R_i$ i.i.d. and Bernoulli distributed with parameter $\f12$.
\end{proposition}
\begin{proof}
    The argument laid out in Theorem 1 in \cite{xu2020theory} holds for general, bounded $\ell$.
\end{proof}

For possible application of this information estimator, see follow-up works on \citet{xu2020theory}'s predictive information \cite{ethayarajh2022understanding,lu2023measuring, fel2024understanding}.

\end{document}